%% file: ms.tex
\documentclass{ieeeaccess}
\usepackage{cite}
\usepackage{amsmath,amssymb,amsfonts}
\usepackage{algorithm}
\usepackage{subcaption}
\usepackage{algpseudocode}
\usepackage{amsthm}
\usepackage{booktabs}
\usepackage{multirow}
\usepackage{graphicx}
\usepackage{pgf}
\usepackage{lmodern}
\usepackage{textcomp}
\usepackage{bm}
\usepackage{pgf}
\usepackage{lmodern}
 \usepackage{tikz}
 \usepackage{mathtools}
 \usepackage{soul}
 \allowdisplaybreaks
\usepackage{dsfont} 
\usepackage[shortlabels]{enumitem}

  \everymath=\expandafter{\the\everymath\displaystyle}

 \input{input}

  \NewSpotColorSpace{PANTONE}
  \AddSpotColor{PANTONE} {PANTONE3015C} {PANTONE\SpotSpace 3015\SpotSpace C} {1 0.3 0 0.2}
  \SetPageColorSpace{PANTONE}%
\makeatletter
\AtBeginDocument{\DeclareMathVersion{bold}
\SetSymbolFont{operators}{bold}{T1}{times}{b}{n}
\SetSymbolFont{NewLetters}{bold}{T1}{times}{b}{it}
\SetMathAlphabet{\mathrm}{bold}{T1}{times}{b}{n}
\SetMathAlphabet{\mathit}{bold}{T1}{times}{b}{it}
\SetMathAlphabet{\mathbf}{bold}{T1}{times}{b}{n}
\SetMathAlphabet{\mathtt}{bold}{OT1}{pcr}{b}{n}
\SetSymbolFont{symbols}{bold}{OMS}{cmsy}{b}{n}

\renewcommand\boldmath{\@nomath\boldmath\mathversion{bold}}}
\makeatother
\theoremstyle{definition}
\newtheorem{remark}{Remark}
\newtheorem{example}{Example}
\def\BibTeX{{\rm B\kern-.05em{\sc i\kern-.025em b}\kern-.08em
    T\kern-.1667em\lower.7ex\hbox{E}\kern-.125emX}}

\begin{document}
\history{Date of publication February 04 2025.}
\doi{10.1109/ACCESS.2025.3538599}

\title{Interacting Large Language Model Agents.
Bayesian Social Learning Based Interpretable Models. }
\author{\uppercase{Adit Jain}\authorrefmark{1}\IEEEmembership{Student Member, IEEE}, and 
\uppercase{Vikram Krishnamurthy}\authorrefmark{2}\IEEEmembership{Fellow, IEEE}}

\address[1]{School of Electrical and Computer Engineering, Cornell University, Ithaca, NY 14853 USA (e-mail: aj457@cornell.edu)}
\address[1]{School of Electrical and Computer Engineering, Cornell University, Ithaca, NY 14853 USA (e-mail: vikramk@cornell.edu)}

\tfootnote{
© 2025 IEEE. Personal use of this material is permitted. Permission from IEEE must be obtained for all other uses, in any current or future media, including reprinting/republishing this material for advertising or promotional purposes, creating new collective works, for resale or redistribution to servers or lists, or reuse of any copyrighted component of this work in other works. The final published version of this paper can be accessed at https://ieeexplore.ieee.org/document/10870230. This research was supported by Army Research Office grant W911NF-24-1-0083 and National Science Foundation grants CCF-2312198 and CCF-2112457.
}
\markboth
{Jain and Krishnamurthy: Interacting Large Language Model Agents for  Bayesian Inference}
{Jain and Krishnamurthy: Interacting Large Language Model Agents for  Bayesian Inference}

\corresp{Corresponding author: Adit Jain (e-mail: aj457@cornell.edu).}

\begin{abstract}
This paper discusses the theory and algorithms for interacting large language model agents (LLMAs) using methods from statistical signal processing and microeconomics. While both fields are mature, their application to decision-making involving interacting LLMAs remains unexplored. 
Motivated by Bayesian sentiment analysis on online platforms, we construct interpretable models and stochastic control algorithms that enable LLMAs to interact and perform Bayesian inference. Because interacting LLMAs learn from both prior decisions and external inputs, they can exhibit bias and herding behavior. Thus, developing interpretable models and stochastic control algorithms is essential to understand and mitigate these behaviors.

This paper has three main results. First, we show using Bayesian revealed preferences from microeconomics that an individual LLMA satisfies the necessary and sufficient conditions for rationally inattentive (bounded rationality) Bayesian utility maximization and, given an observation, the LLMA chooses an action that maximizes a regularized utility. Second, we utilize Bayesian social learning to construct interpretable models for  LLMAs that interact sequentially with each other and the environment while performing Bayesian inference. Our proposed models capture the herding behavior exhibited by interacting LLMAs. Third, we propose a stochastic control framework to delay herding and improve state estimation accuracy under two settings: (a) centrally controlled LLMAs and (b) autonomous LLMAs with incentives. Throughout the paper, we numerically demonstrate the effectiveness of our methods on real datasets for hate speech classification and product quality assessment, using open-source models like LLaMA and Mistral and closed-source models like ChatGPT.  The main takeaway of this paper, based on substantial empirical analysis and mathematical formalism, is that LLMAs act as rationally bounded Bayesian agents that exhibit social learning when interacting. Traditionally, such models are used in economics to study interacting human decision-makers. 
\end{abstract}


\titlepgskip=-21pt

\maketitle

\section{Introduction}
\PARstart{T}{his} paper discusses the theory and algorithms for interacting Large Language Model Agents (LLMAs) by leveraging techniques from Bayesian inference, stochastic control, and microeconomics. Specifically, we focus on developing interpretable models and stochastic control algorithms for \bayesianagents, enabling them to interact sequentially for Bayesian inference.


We construct interpretable models of \bayesianagents\ at two levels of abstraction, as outlined in Figure~\ref{fig:interpret}.   
 First, we model an individual \bayesianagent\ as a rationally inattentive Bayesian utility maximizer, capturing the agent’s decision-making process under limited attention. Second, we extend this approach to a sequence of \bayesianagents\ engaging in Bayesian social learning or \textit{groupthink}, where each agent is a Bayesian utility maximizer.  Our models are inspired by the self-attention mechanism in large language models (LLMs) and observed challenges, such as model collapse, that arise during LLM training.

\begin{figure*}[ht!]
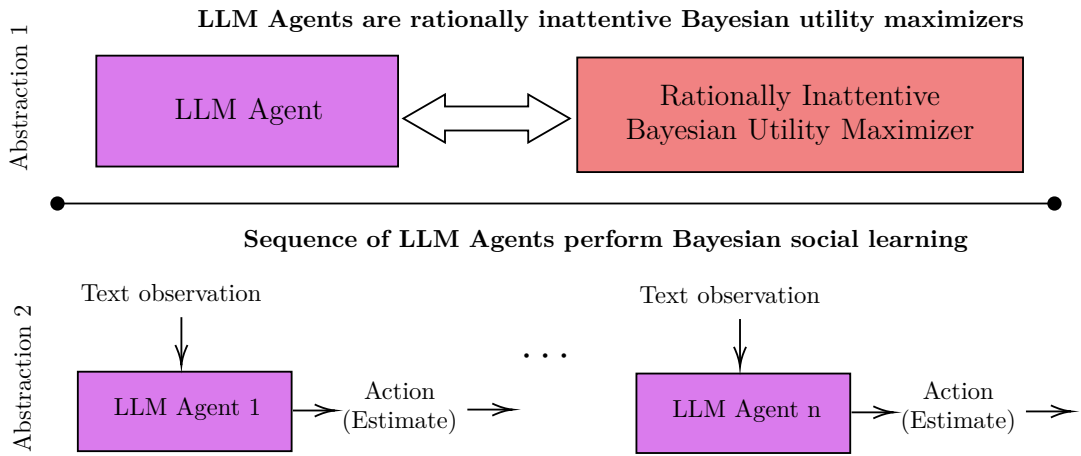

    \centering
    \include{plots/parts}
    \caption{Summary of the proposed contributions: We discuss the different blackbox models for LLMA and how LLMAs can be used as a sensing mechanism to perform Bayesian inference. Part 1 models the LLMAs as a rationally inattentive Bayesian utility maximizer and numerically establishes the behavior in applications of product quality identification and hate speech classification. Part 2 discusses how Bayesian social learning in a sequence of LLMAs can be used for sequential state estimation. However, in Part 3, we show that the agents can perform the same incorrect action due to herding. We then discuss a stochastic control approach to delay herding when LLMAs are centrally controlled and when they are autonomous but are incentivized.  }
    \label{fig:interpret}
\end{figure*}

Furthermore, motivated by the observed bias in the behavior of interacting \bayesianagents, we demonstrate that a sequence of \bayesianagents\ engaging in Bayesian social learning end up making identical decisions, or "herd". To address this phenomenon, we propose a stochastic control approach, formulating an optimal stopping problem to balance the trade-off between privacy and herding, to detect the failure state. Our approach is designed for two scenarios: (a) when the \bayesianagents\ are centrally controlled, and (b) when they operate autonomously.


The overarching goal of this paper is to demonstrate that concepts from controlled sensing and microeconomics, traditionally applied to human decision-making, can be used to both understand and synthesize the behavior of interacting \bayesianagents~\cite{kaplinanddeen,bhattcontrolled2020,chamley_rational_2004,quickestchangedetectionitit,krishnamurthy_partially_2016}. We support our theoretical findings with numerical experiments using advanced LLMs for Bayesian inference~\cite{gelman2013bayesian} on real-world data. This paper is written to engage a broad readership, highlighting applications of Bayesian agents in diverse fields, including financial news analysis, e-commerce review evaluation, and online toxicity detection. These examples underscore the flexibility of our methodologies for cross-disciplinary applications. The reproducible code for our experiments is publicly accessible at github.com/aditj/sociallearningllm.

\subsection{Motivation}
LLM agents (\bayesianagents) are being rapidly deployed for different applications~\cite{whitepaperagents}, and to quote Sam Altman, CEO of OpenAI (creators of ChatGPT, a popular LLM which has 200 million weekly active users):  ``in 2025, we may see the first AI agents join the workforce''~\cite{samaltman}. 

\bayesianagents\ use a large language model (LLM) to parse the input and have additional functionality to perform tasks. LLMs (such as ChatGPT and LLaMA) are neural networks with billions of parameters trained on trillions of tokens of textual data to parse long texts for summarizing, compiling key facts, and generating new text. The key technical improvement that leads to the efficient deployment of LLMs is the transformer architecture~\cite{10.5555/3295222.3295349}. The effectiveness of LLMs on textual texts has made their deployment and adoption widespread~\cite{min_recent_2023}. Many applications have been proposed in healthcare, online platform moderation, and finance, where these LLMs are used to parse the textual observations and suggest decisions based on their outputs~\cite{li_pre-trained_2022}. 
\begin{figure*}[ht!]
    \centering
    \include{plots/llmengineer}
    \vspace{-5mm}
    \caption{Engineering with \bayesianagentstext\ (\bayesianagents): We propose engineering with LLMs on three different levels: a) First, we construct \bayesianagents\ with an LLM attached to the Bayesian engine. The LLM acts as a sensor for the text input and outputs interpretable low-dimensional outputs, which are used by the Bayesian engine to produce a state estimate. b) Second, we formulate necessary and sufficient conditions for a \bayesianagents\ to be a rationally inattentive Bayesian utility maximizer (\ribum). We also present algorithms to reconstruct feasible utilities and rational inattention costs if the \bayesianagent\ is indeed a \ribum, attributing the \bayesianagent\ with an interpretable microeconomic model. c) Finally, we demonstrate how a sequence of \bayesianagents\ can efficiently perform sequential Bayesian social learning by controlling their outputs to delay herding optimally. Our Bayesian social learning models can be extended to study Bayesian social learning in a network of \bayesianagents. }
    \label{fig:engineer}
\end{figure*}
In many tasks, the outputs of the LLMs are often part of a more extensive pipeline; for example, the output of the LLMs, either in a specified format or as embeddings, is frequently used as inputs to other Bayesian entities, including classifiers~\cite{minaee_deep_2021}. The Bayesian framework also becomes essential in applications where the LLMs have to output decisions and need to provide confidence in the decision output. Thus, it is of interest to study a single Bayesian agent that uses the LLM to parse text observations, update its Bayesian belief, and take action. This paper studies such entities and refers to them as \textbf{L}arge \textbf{L}anguage \textbf{M}odels \textbf{A}gents (\bayesianagents).
Constructing interpretable models for LLMAs is crucial to understanding and controlling their interaction.  
\subsubsection{Interacting \bayesianagentstext}

It is predicted that by 2030, 90\% of web content will be generated by large language models (LLMs)~\cite{yahooOnlineContent}. In recent practical implementations, individual LLMs are part of a bigger system, referred to as LLMAs, and interact with the content generated by other LLMAs and the external environment~\cite{zhou2023agents}.  Furthermore, recent research has shown how generative models are trained on the data generated by other generative models can collapse~\cite{shumailov2024ai}. Therefore, naturally, LLMAs interact with each other either implicitly or explicitly.

Hence, controlling the dynamics of interacting LLMAs is essential to improve the accuracy and trustworthiness of decisions by LLMAs. To the best of our knowledge, only a few recent works systematically study the behaviors of LLMAs using tools from microeconomics and signal processing~\cite{jain2024identifying}. This study aims to bridge this gap by systematically reviewing LLMAs and the different mathematical frameworks by studying Bayesian social learning in a sequence of LLMAs to achieve Bayesian inference. 

\subsubsection{Interpretable Engineering of \bayesianagents} 
Many different third-party services have already started providing various kinds of \bayesianagents\ as a service, including Agentforce by Salesforce and IBM AI agents~\cite{agentforce}. The underlying intelligence engine of these third-party agents is an LLM or a vision language model (VLM). The \bayesianagents\ are used in personal applications for coding, shopping, and scraping data and in enterprise applications for getting insights on user activity and automating industrial workflows. 

Therefore, it becomes imperative to study interpretable models for these agents since many of the proposed applications for these agents involve sensitive information (like personal records, financial information, bio-medical data, and personal preferences). By interpretable, we refer to 
\textit{ models that facilitate a transparent understanding of complex models through clear and explainable representations of their decision-making processes.}
The workflows of the AI agents also include making decisions, and the interpretability and reliability of these agents become vital for them to be trustworthy. Therefore, mathematical models are needed to aid in engineering and deploying \bayesianagents. To this end, we propose a \bayesianagent\ composed of an LLM and a Bayesian engine, which by construction is interpretable. Further, we use Bayesian revealed preferences\footnote{The framework of Bayesian revealed preferences is also referred to as inverse optimization or inverse reinforcement learning.} to reconstruct a Bayesian utility function for both our constructed \bayesianagent\ and for off-the-shelf \bayesianagents.  

\subsubsection{Bayesian Inference from Multi-Modal Data Stream}
In various applications, like online e-commerce platforms, video streaming platforms, and social networks, there is a rich stream of multimodal data available using text, images, and videos. Different inference tasks involve fusing information from various data streams to get actionable insights. With the recent progress in deep learning, many of the traditional signal processing methods are being replaced with contemporary methods that use LLMs and VLMs. However, just using static models is not sufficient to model the dynamics of real-life settings, e.g. on online platforms, and underlying dynamics are better modeled in a Bayesian framework. Therefore, motivated by practical applications, we propose the construction of  \bayesianagents\ which can perform Bayesian inference sequentially on a data stream. This complements continual learning, which deals with continually learning new tasks without forgetting what was learned previously~\cite{continuallearning}.  

\subsection{Main Results}
\begin{figure*}
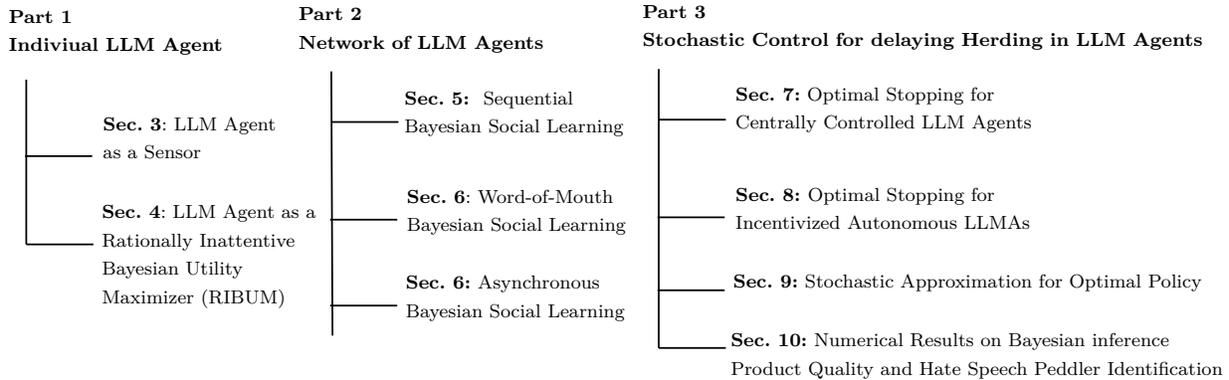

    \centering
    \include{plots/organization}
    \caption{Organization of the paper: The paper is divided into three parts. Part 1 deals with interpretable models for an individual LLM agent. Part 2 extends the models to a social learning setting where LLM agents interact with each other to perform Bayesian inference. Part 3 proposes stochastic control methods to delay herding in a sequence of LLM agents.}
    \label{fig:organization}
\end{figure*}
This paper builds on Bayesian revealed preferences from microeconomics (inverse reinforcement learning), sequential Bayesian estimation (from signal processing), and structured stochastic control to construct interpretable models and synthesize interaction of \bayesianagents. The impact of our results on more efficient, systematic, and interpretable engineering of \bayesianagents\ is summarized in Figure~\ref{fig:engineer}. The main contributions of this paper are:
\begin{enumerate}
    \item We propose constructing a \bayesianagent\ as a composition of a large language model (LLM) sensor\footnote{We consider LLMs as social sensors and not physical sensors since they do not sense physical quantities like temperature but virtually analyze text and multimodal data to provide observations.}, which acts as a low-dimensional map from the text space, and a Bayesian engine, which uses the measurement from the LLM to update the posterior and act optimally. We show how this model is useful for interpretable Bayesian inference with applications in sequential data on online platforms. 
    \item To obtain an interpretable utility function for a \bayesianagent, we provide necessary and sufficient conditions in Theorem~\ref{th:nscribum} for a \bayesianagent\ to be a rationally inattentive Bayesian utility maximizer (\ribum). For a \bayesianagent\ who is a \ribum, we propose Algorithm~\ref{alg:maxmargin} and Algorithm~\ref{alg:sparse} to reconstruct the max-margin and sparsest utility estimate, respectively. Our methods are applicable to both our \bayesianagent\ and off-the-shelf \bayesianagents.
    \item We study Bayesian social learning in a \bayesianagents, sequentially estimating a state given text observations and in Theorem~\ref{th:herding} show that such a sequence of \bayesianagents\ form an information cascade and herd in their actions. We show that this is true for both when no private observations are shared and when a finite number of private observations are shared. Further, we provide a detailed analysis of the effect of the quality of results from LLM of the \bayesianagent\ and the number of private observations. 
    \item To delay herding in a sequence of \bayesianagents, we formulate an optimal stopping problem for two regimes: 
    
    a) when the \bayesianagents\ are centrally controlled by an entity b) when the \bayesianagents\ are autonomous but are incentivized by an entity. We show in Theorem~\ref{th:threshold} and Theorem~\ref{th:fusioncenter} that under certain assumptions on the observation matrix and cost functions, the optimal policy for the partially observed Markov decision process of both the optimal stopping problems has a threshold structure. We then propose a policy gradient algorithm in Algorithm~\ref{alg:stochasticapprox}, which exploits the structural results to estimate the optimal policy parameters. The algorithm does not need access to the system parameters, is computationally efficient, and can track changes in the system.
    \item We finally present several numerical experiments to demonstrate the efficacy of our proposed methods. We show how our constructed \bayesianagent\ can be used for interpretable Bayesian inference for analyzing financial data. We show how the Bayesian revealed preferences framework can estimate the utility of an off-the-shelf LLM when used for hate-speech detection. Finally, we show numerical studies on two examples of sequential Bayesian inference: hate speech peddler identification and product quality analysis, to demonstrate herding of \bayesianagents, and the applicability of our structural results. 
\end{enumerate}

To summarize, this paper attempts to answer the following questions with respect to interacting LLM Agents,
\begin{enumerate}
\item {How can \bayesianagents\ be constructed so that they can be used for sequential Bayesian inference such that the observation and outputs are interpretable?}
\item {What is a principled approach to analyze whether a \bayesianagent\ is a Bayesian utility maximizer and also reconstructs its utility function given only blackbox access?}
\item {How does one systematically study Bayesian social learning in multiple interacting \bayesianagents\ to explain observed behaviors such as herding and model collapse?  }
    \item {How can herding in (centrally controlled or autonomous) \bayesianagents\ be optimally delayed so that the agents optimally switch between preserving privacy and improving estimation to achieve sequential detection?}
\end{enumerate}
\subsection{Organization}
This paper is organized into three parts, and the schematic of the organization is given in Figure~\ref{fig:organization}. Part I discusses the interpretable model for a single \bayesianagent\ and attempts to answer questions 1 and 2 above. 
Section~\ref{sec:relatedwork} discusses the related work in large language models, agents using LLMs, and current interpretable models for Bayesian inference. Section~\ref{sec:llmasensors} discusses the mathematical model used for modeling \bayesianagents\ in this paper and motivates the different components involved. Section~\ref{sec:ribum} gives the necessary and sufficient conditions for the \bayesianagents\ to be rationally inattentive Bayesian utility maximizers (\ribum). It further proposes algorithms to estimate the utility function for a \bayesianagent\ which is a \ribum.

Part II discusses interpretable models for interacting \bayesianagents\ and attempts to answer question 3. 
Section~\ref{sec:sociallearning} discusses the mathematical framework of Bayesian social learning in \bayesianagents\ and proves that a sequence of \bayesianagents\ form an information cascade in finite time. Section~\ref{sec:quickesttimeherding} discusses a stochastic control problem for the optimal stopping time problem to achieve quickest time herding with minimal loss to the privacy of \bayesianagents. Section~\ref{sec:wordofmouth}  discusses interpretable models to explain model collapse and data incest in \bayesianagents\ using word-of-mouth and asynchronous social learning. 

To decrease the bias when a sequence of \bayesianagents\ perform Bayesian inference, Part III deals with stochastic control for delaying herding in interacting \bayesianagents\ performing Bayesian sequential learning proves structural results, and proposes a policy gradient approach. 
Section~\ref{sec:incentivizedherding} considers the problem of a central controller optimally optimizing a sequence of autonomous~\bayesianagents\ to achieve the state estimation by optimally controlling herding. Section~\ref{sec:stochasticapproximation} proposes a policy gradient based approach to approximate the optimal policy, which has a threshold switching curve. Numerical results on real-life text classification tasks and related applications are discussed in Section~\ref{sec:numerical}. Section~\ref{sec:conclusion} concludes the paper with discussions on future works, open problems, and research opportunities. The appendix contains the proofs and details about the numerical experiments. 

We emphasize that the paper is built around a coherent framework with the unifying theme of building interpretable models for interacting LLMAs using Bayesian social learning and powerful generative models from behavioral economics.
For the ease of the reader, we have included a motivation and a discussion subsection in each section, which grounds the different aspects of \bayesianagents\ to a real-life application and different microeconomics and statistical signal processing tools presented in the section. We also provide different block diagrams and illustrative examples to further aid the reader. All the results reported in this paper are fully reproducible with code downloadable from github.com/aditj/sociallearningllm. 
\section{Brief Literature Review and Related Work}\label{sec:relatedwork}
In this section, we review related work on LLMs\footnote{In this paper, LLMs also refer to the various transformer architectures that process multi-modal data, including images, audio, and documents.}, \bayesianagents\ and social learning using LLMs. We first provide a brief background on LLMs and discuss the different applications and models for \bayesianagents. We provide motivation for the interpretability of the LLM agents. Finally, we review literature in sequential state estimation setup studied in this paper and provide motivation for using a sequence of contemporary \bayesianagents\ in a classical Bayesian inference setting. We also review applications of sequential Bayesian inference using \bayesianagents. Table~\ref{tab:relatedwork} summarizes some of the related work.

\begin{table}[]
\centering
\resizebox{\columnwidth}{!}{%
\begin{tabular}{@{}cll@{}}
\toprule
\textbf{Area}                             & \textbf{Topic}             & \textbf{Papers} \\ \midrule
\multirow{2}{*}{Large Language Models}    & Survey                     &  \cite{kaddour_challenges_2023,llmaccesssurvey}               \\
                                          & Applications               & \cite{llmsgraphs,llmsongames,chatgpteducation,inforetrieval}                \\\midrule
\multirow{3}{*}{LLM Agents (LLMAs)}               & Survey                     &  \cite{ijcai2024p890,li2024personalllmagentsinsights,li2024finmem,hsu2023helpinghelpersupportingpeer}               \\
                                          & Applications               &   \cite{zhou2023agents,llmautomation,suri_software_2023,webshop}              \\ 
                                          & Networks                &      \cite{sumers2024cognitive,zhuge2023mindstormsnaturallanguagebasedsocieties,liu2024a,wang2023jarvis1,chuang2024simulatingopiniondynamicsnetworks,han2024regulatormanufactureraiagentsmodeling,marchi2024heatdeathgenerativemodels}              \\\midrule
\multirow{3}{*}{Bayesian Social Learning} & General                    &  \cite{bhattcontrolled2020,quickestchangedetectionitit,chamley_rational_2004,duan_flocks_nodate}               \\
                                          & LLMs                       &   \cite{mohtashami_social_2024,llmsgraphs,NEURIPS2023_1190733f,chevalier2023adapting}              \\
                                          & Incentivization            &     \cite{kaddour_challenges_2023,agentforce}      \\\midrule
\multirow{4}{*}{Interpretability}         & Mechanistic                &         \cite{brinkmann-etal-2024-mechanistic,nanda2023progress}        \\
                                          & Glass Box Models           &         \cite{singh_augmenting_2023}        \\
                                          & blackbox Models           &         \cite{pattanayak,JMLR:v24:20-1202}        \\
                                          & Sociological Fairness      &         \cite{biasfairness,anthis2024impossibilityfairllms,accessxai,operationalizingresponsibleAI}        \\
                                \midrule
\multirow{3}{*}{\begin{tabular}[c]{@{}c@{}}Applications of LLMAs \\ in Sequential State Estimation\end{tabular}} & Product Quality Identification & \cite{Liu2024,sentimentanalysis} \\
                                          & Hate Speech Classification &    \cite{jain2024identifying,sachdeva_measuring_2022}             \\
                                          & Recommendation             &      \cite{recommendationllms,ni-etal-2019-justifying}           \\ \cmidrule(l){1-3} 
\end{tabular}%
}
\caption{Summary of Related Literature studying \bayesianagents\ and their interaction: There has been work in engineering, sciences, and economics, which motivates a careful study.}
\label{tab:relatedwork}
\end{table}
\subsection{Background on LLMs}
Large language models (LLMs) have become omnipresent in various industry applications, given the drastic improvement in compute availability and rapid development of open and closed-sourced models~\cite{llmaccesssurvey}. They are being rapidly deployed for various applications, including education, information retrieval, gaming, recommendation systems, and understanding graphs~\cite{chatgpteducation,inforetrieval,llmsgraphs,llmsongames}. 

The primary reason for the proliferation of LLMs is that they can take a high-dimensional multi-modal (text, images, audio etc.) data input and provide useful inferences about them. This is possible because they are deep learning networks (transformer architecture) with billions of parameters trained on massive amounts (in the order of petabytes) of data (CommonCrawl, etc.) using extremely fast GPUs that can parallelize computations efficiently. There are two different classes of LLMs: a) open source LLMs like LLaMA and Mistral~\cite{jiang2023mistral7b,touvron2023llamaopenefficientfoundation} and b) closed source LLMs like ChatGPT and Claude. Open source LLMs make available the underlying deep learning architecture that they use, some even share the data that the LLM is trained on; closed-source LLMs on the other hand only provide an inferencing interface, where the LLM can be asked different questions, the answer to which is provided by the LLM. The methods and framework of our paper only require blackbox access to these LLMs and are generally applicable to both classes. 

Although the adaption of LLMs is rapidly increasing, from a safety and reliability perspective, their deployment in sensitive applications like healthcare, finance, and defense still poses challenges~\cite{llmaccesssurvey}. Even in general-purpose applications, LLM-based chatbots provide spurious information, a phenomenon referred to as hallucination~\cite{llmaccesssurvey}. There have been different approaches to ensure that the outputs of the LLMs are reliable and interpretable, and many of the challenges specific to the new paradigm require revisiting traditional interpretability literature~\cite{singh2024rethinkinginterpretabilityeralarge}. As reviewed in~\cite{zhang2024llm}, one of the approaches is to improve the reasoning capabilities of LLMs so that the LLMs provide a descriptive reason for the output. Another school of thought is to mechanistically understand the transformer architecture and training procedure that is the backbone of an LLM~\cite{brinkmann-etal-2024-mechanistic,nanda2023progress} to better understand the working and eliminate the sources of bias, if possible.

\subsection{LLM Agents}
Standalone LLMs are powerful tools for many applications, but there is strong motivation to consider networked LLMs as part of complex systems and integrate them into existing workflows. These networked LLMs are often given \textit{functionality} to interact with the virtual environment and are referred to as LLM agents (\bayesianagents)~\cite{ijcai2024p890}. There are two distinct features that \bayesianagents\ have that make them different from LLMs: 
\begin{enumerate}
    \item \textit{Decisions}: In the workflows that the \bayesianagents\ are used in, they are provided functionality (agency) using different mechanisms, including function calls
    \item \textit{Communication}: The \bayesianagents\ are allowed to communicate with other \bayesianagents, to exchange information. Often, tasks are also broken into smaller sub-tasks and are performed parallelly and sequentially by different \bayesianagents\ leading to different topologies of \bayesianagents~\cite{li2024personalllmagentsinsights}. 
\end{enumerate}

\subsubsection{Applications of LLM Agents}
An important application \bayesianagents\ are used for is programming, primarily because of LLMs ability to generate code given a text prompt.   \bayesianagents\ are used to automate different parts of the software lifecycle, including development, deployment, testing, and fixing bugs~\cite{li2024personalllmagentsinsights}. Other applications propose using \bayesianagents\ in healthcare for counselling~\cite{hsu2023helpinghelpersupportingpeer}, financial trading~\cite{li2024finmem}, automating customer service~\cite{agentforce} and shopping assistants~\cite{webshop}. 
\subsubsection{Models for LLM Agents}
The different components used in the standard model of a \bayesianagent\ include a memory, retrieval mechanism, action sets, and an environment. 
 ~\cite{sumers2024cognitive} studies different cognitive architectures using these components for \bayesianagents. There has been a lot of work to improve the capabilities of these components in \bayesianagents, using a dynamic  context~\cite{chevalier2023adapting} and using self notes to perform continual learning~\cite{lanchantin2023learning}. However, we propose augmenting the LLM with a  Bayesian engine to perform sequential Bayesian inference on a stream of data. The Bayesian engine model proposed in this paper can also be used for more general tasks, as we discuss in the conclusion. 

\subsubsection{Networks of Agents}
 Since many of the LLM agents interact with other agents directly or through content generated by them, there is a need for more systematic and mathematically rich blackbox models for LLMs and LLM agents (LLMAs). Such models help understand their behavior and eventually control it to ensure reliability. Often, the collaboration of \bayesianagents\ are modeled as a directed graph~\cite{ijcai2024p890,chuang2024simulatingopiniondynamicsnetworks,han2024regulatormanufactureraiagentsmodeling}, which ~\cite{liu2024a} proposes dynamically adapting depending on the task. There are various different programmatic frameworks where the \bayesianagents\ can be abstractly programmed to perform different tasks~\cite{wang2024voyager,wang2023jarvis1}. Some of these frameworks even allow making these agents autonomous~\cite{zhou2023agents}. The methods in this paper deal with a line graph topology of \bayesianagents, which perform sequential Bayesian estimation. The setup studied in this paper can be extended to more general graph structures, and different issues such as data incest can be studied.


\subsection{Bayesian social learning in a sequence of LLM Agents}

\subsubsection{Multiple Bayesian agents sequentially estimating a state}
The interaction of multiple such Bayesian agents, each receiving a private observation, is motivated by privacy, improved detection, and finite context length. If the same private observation (even the low-dimensional representation) is used, the LLM can be fine-trained on this data, which might contain sensitive information~\cite{duan_flocks_nodate}. Also, different LLMs can be given a diverse set of contexts, which enables reducing the bias involved with their decisions~\cite{kaddour_challenges_2023}. Also, practically due to the finite context length, the observations can be considered private with respect to consecutive \bayesianagents\ evaluations.

\subsubsection{Framework of Bayesian social learning} \change{Recent work has looked at social learning in LLMs using a teacher-student framework}~\cite{mohtashami_social_2024}, but this work was in a static setting where the LLMs do not have a belief that they adaptively update. In general, sequential social learning in Bayesian agents has been studied extensively~\cite{chamley_rational_2004}, and our work formalizes the problem of Bayesian social learning in~\bayesianagents. The theoretical results presented in this paper have been studied before in the context of distributed Bayesian sensors in \cite{quickestchangedetectionitit} and \cite{bhattcontrolled2020}. Compared to~\cite{bhattcontrolled2020,quickestchangedetectionitit}, we view the \bayesianagents\ as interpretable Bayesian sensors and provide a more comprehensive outlook. We also consider multiple observations being shared and provide a concentration inequality (Theorem~\ref{th:submartingale}) for overspending the incentive with respect to a budget constraint. Recently \cite{infocascades}, looked at detecting information cascades using deep learning. Although in this paper we focus on delaying an information cascade to improve estimation accuracy, methods similar to~\cite{infocascades} can be integrated with our approach.

\subsubsection{Incentivization of the \bayesianagents\ by a central controller}
Recent research has studied modeling LLMs as autonomous agents and making LLM part of bigger autonomous agents, including robots, self-driving cars, and programming co-pilots~\cite{cui_drive_2024,suri_software_2023}. Such autonomous agents can be \textit{leased} from third-party services at a unit cost. The incentive can also be looked at from the following perspective: providing more context to the same LLM can lead to a more accurate output~\cite{kaddour_challenges_2023} but increases the cost of processing the query. Third-party \bayesianagents\ often offer a tiered pricing structure, where higher pricing provides access to more accurate LLMs. 

\subsection{Interpretability and Social Fairness}
There has been recent work in augmenting the LLM with an explainable artificial intelligence (xAI~\cite{accessxai}) system to provide more interpretable outputs as in~\cite{singh_augmenting_2023}.  This work is more aligned with the latter, in which we propose using the LLM as a sensor that provides interpretable low-dimensional outputs, used by a Bayesian engine to estimate the state. 

However, we use an interpretable Bayesian model to perform sequential Bayesian inference from text observations of an underlying state, whereas~\cite{singh_augmenting_2023} create an xAI model with decision trees and n-grams models using outputs from an LLM. Another such work is~\cite{sun2024crafting}, where the authors propose training a separate concept neural network that uses the output of an LLM to interpretably classify text embeddings. This approach can complement the work in this paper when the setting is dynamic. Our work uses tools from revealed preferences and social learning to analyze the behavior of individuals and interacting \bayesianagents\ from a microeconomic lens.

Further, our focus is also on designing \bayesianagents, which are safe, reliable, and fair, goals which are aligned with operationalizing responsible AI~\cite{operationalizingresponsibleAI}. This becomes challenging to do since \bayesianagents\ have been known to show biases which are inherent to human beings like conformity~\cite{baltaji-etal-2024-conformity} and bias towards different attributes~\cite{borah2024implicitbiasdetectionmitigation}. These effects will be more prominent when the \bayesianagents\ interact with each other in different scenarios like a Mindstorm \cite{zhuge2023mindstormsnaturallanguagebasedsocieties} or a language model cascades~\cite{gupta2024language}.
Finally, the need to study interpretable models for \bayesianagents\ is motivated by the rise of unified agents~\cite{palo2023towards}, which are a representation of a trend in artificial intelligence that the different models are converging to a single efficient model~\cite{pmlr-v235-huh24a}. This paper therefore tries to systematically understand Bayesian social learning in \bayesianagents, to help prevent undesirable phenomena like model collapse~\cite{shumailov2024ai,marchi2024heatdeathgenerativemodels}.

There has been substantial work in the fairness of the machine learning models~\cite{biasfairness}, and even evaluating large language models for different measures of social fairness~\cite{evaluationsurveyllms}. However,~\cite {anthis2024impossibilityfairllms} recently highlighted how it is extremely difficult to benchmark LLMs on existing fairness metrics because of the way LLMs are used. This becomes even more challenging for \bayesianagents, where the agents further have functionality and can also communicate with other \bayesianagents. Therefore, the focus of our study is to construct interpretable models for \bayesianagents, which can be used to understand the decisions of \bayesianagents. This understanding can help construct more suitable societal fairness metrics. Our work additionally has relatively mild assumptions on the utility/cost function of the \bayesianagents, and hence can be adapted for different sociological costs.

\subsection{Applications of Sequential State Estimation using Bayesian Social Learning in \bayesianagents}
We detail examples of real-life problems where textual observations of the state are available and sequential Bayesian learning in \bayesianagents\ is used to perform state estimation. 
\subsubsection{Hate Speech Peddler Identification On Social Networks }
Identifying hate speech\footnote{There is an active debate on the definition of hate speech and the tradeoff between free speech and hate speech~\cite{howard_free_2019}. Hence, to circumvent this discussion, we use hate speech as an exemplary case study of our methods, and the definition of hate speech is implicit from the source of the dataset in the experiments. Our techniques can be applied to different definitions of hate speech and other applications as is described later.} and toxic content has been studied in various contexts, e.g., in reducing unintended bias, detecting covert hate speech, and mitigating hate speech on online platforms~\cite{kennedy_contextualizing_2020}. \cite{sachdeva_measuring_2022} have looked at how to quantify the intensity of hate speech and created labeled datasets. In~\cite{jain2024controlling}, the authors looked at controlling federated learning for hate speech classification. In this paper, we look at the problem of Bayesian agents identifying hate speech peddlers by sequentially parsing comments from users using an LLM.

\subsubsection{Financial Networks } In financial networks, LLMs can be used as sensors to parse textual information, including news articles, opinions on social networks, and financial reports. This can be especially useful for making decisions based on the low-dimensional observations from the LLMs. This process can be automated using \bayesianagents, based on algorithmic rules~\cite{li2024finmem}. But, since the actions of \bayesianagents\ of a single entity affect the environment (market), such a sequence of agents can herd in their decisions, leading to a financial bubble. This has been studied classically in human traders~\cite{chamley_rational_2004} and makes the study of sequential Bayesian learning in \bayesianagents\ interesting. 

\subsubsection{Product Quality Identification}
One of the issues on e-commerce platforms is to identify poor quality products early on; however, just using numerical ratings can be uninformative, especially when the number of ratings is less. However, there is a lot of information contained in the descriptive reviews of the product, which one can efficiently extract using LLMs. Therefore, a sequence of \bayesianagents\ can efficiently analyze the reviews of a product to identify if the product is of good quality or not. This is an extended and sophisticated version of opinion mining, which has shown to be effective in sentiment analysis on online platforms~\cite{sentimentanalysis}. 

\subsubsection{Personalized Recommendation Systems}
Another primary application of LLMs is in recommendation systems~\cite{recommendationllms}, where the LLMs act as a natural language interface between the user and the item (e.g., movies, products) database. \bayesianagents\ can be further used to enhance the recommendation quality by analyzing the past activity of the user and of the user's social network. However, using \bayesianagents\ directly raises privacy concerns since the users information is sensitive. This paper proposes one way to deal with this, by ensuring that each \bayesianagents\ has a different private observation. 

\subsection{Comparision To Prior Work}

\change{
The primary focus of this paper is to employ ideas from statistical signal processing and microeconomic theory to model and analyze interacting LLMAs. However, given the domain-specific constraints that are inherent to LLMAs, we claim the following modeling and theoretical innovation over prior work in Bayesian social learning and interpretable utility reconstruction: a) There is a hierarchical observation structure inherent to the Bayesian sensor model of an LLMA considered here, whereas the Bayesian revealed preference framework only has a single observation likelihood}~\cite{kaplinanddeen} \change{b) There has been limited research on interpretable models for LLMAs. However, in contrast to the other interpretable machine learning literature, including glassbox models, we reconstruct a utility function and an information acquisition cost given only blackbox access to the LLMA and is more closely aligned to}~\cite{pattanayak}\change{. 
c) We study the effect of sharing more private observations and a better observation likelihood on herding in a classical Bayesian social learning setting}~\cite{chamley_rational_2004,acemoglu_opinion_2011} \change{d) The social learning framework considered here allows for LLMAs with different LLMs and is presented for the specific application of Bayesian inference in contrast to}~\cite{mohtashami_social_2024}.

\vspace{0.5cm}
{\color{accessblue}{\hrule}} \vspace{0.5cm}

{\hspace{-3mm}\Large\color{accessblue} \textbf{Part I: Analyzing a  Single LLM Agent }}
\vspace{0.5cm}

\hspace{-4mm} In Part I of this paper, we consider a single \bayesianagent\ in isolation, where we first construct a Bayesian sensor model for a \bayesianagent\ which comprises an LLM and a Bayesian engine. Then, we look at the \bayesianagent\ as a rationally inattentive Bayesian utility maximizer and propose methods to reconstruct utilities for both our constructed Bayesian \bayesianagent\ and more general \bayesianagent. The motivation for this modeling from this perspective from the self-attention mechanism inherent to LLMs. Part I of this paper comprises of Section~\ref{sec:llmasensors} and Section~\ref{sec:ribum}. 
\section{LLM Agent as a Social Sensor }\label{sec:llmasensors}

\begin{figure*}[t!]
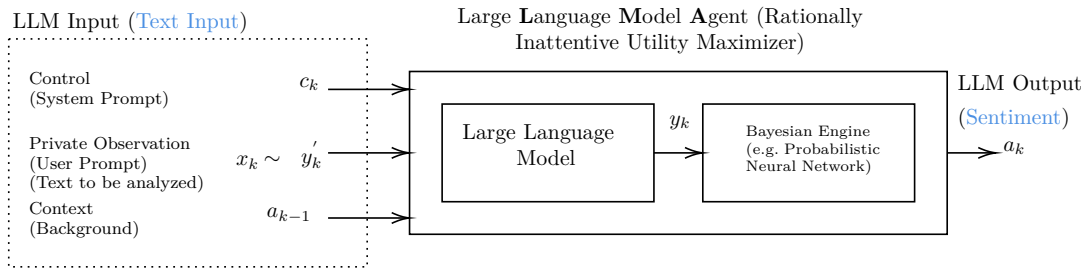

    \centering
    \include{plots/model}
    \caption{Brief schematic of a large language model agent as a sensing mechanism for Bayesian Inference: LLM Input is composed of the system instruction prompt, which is also the control; the user prompt, which is a private observation; and the in-context examples generated from the previous LLM agents are the past actions. Based on the input, the LLM outputs an intermediate textual output. The Bayesian engine uses a likelihood function and past actions to select an action maximizing the expected utility. If utility function is not explicitly given, Bayesian revealed preference is used to obtain a set-valued estimate using an input-output dataset. The paper discusses variations of this model with application in Bayesian sentiment analysis.   }
    \label{fig:systemmodel}
\end{figure*}
Motivated by interpretable Bayesian inference on online platforms using \bayesianagents\, this section discusses the Bayesian sensor model we consider for a single large language model agent. We propose the model of LLMs as a sensing mechanism as a map from a high dimensional space (e.g., text prompt) to a low dimensional space (e.g., structured outputs). The LLMs are equipped with a Bayesian engine and are referred to as an LLM agent (LLMA), which updates the prior regarding the state to be estimated using the text observations. This proposed model is depicted in Figure~\ref{fig:systemmodel}.

The different aspects of the mathematical model for the LLMAs are discussed, and the utility of a \bayesianagent\ is introduced, which can be reconstructed for a blackbox \bayesianagent\ using the Bayesian revealed preference framework discussed later. We discuss how the framework and the results of the paper can be extended to contemporary models like vision language models (VLMs).
In essence, this section, therefore, shows how \bayesianagent\ acts as a social sensor, which can be applied to sophisticated settings such as online platforms where physical sensors do not work to sense the underlying state from observation obtained by the interaction of humans.

\subsection{Motivation. LLM Agent for Interpretable Sentiment Analysis. }
Since in many contemporary applications, LLMs are used for inferring the underlying state given a text observation; we construct an LLM (which acts as a likelihood function) equipped with a Bayesian engine, both of which act as a \bayesianagent\ to perform Bayesian inference on textual data with applications in sentiment analysis on online platforms~\cite{sentimentanalysis}.  

We further motivate the construction of such a Bayesian sensor model for an LLM from the point of view of interpretability, reliability, and controllability. There has been a lot of work done to improve the interpretability of the output of a standard LLM~\cite{singh2024rethinkinginterpretabilityeralarge}. In the blackbox setting, the approach proposes asking the LLMs to provide a reason in addition to the output~\cite{zhang2024llm}.  This works well in practice for simple applications, however when sequential Bayesian inference needs to be performed on millions of text observations interpreting the reason itself becomes a tedious task. Therefore, we propose using the LLMs as a low-dimensional map from the high-dimensional text space by designing prompts that are useful in analyzing.  The LLMs can either be explicitly controlled using the system prompt or their outputs can be restricted to a certain state. 

A Bayesian engine then uses these low-dimensional variates to update the belief, which is an easier task to do than on the high-dimensional text data due to the curse of dimensionality. This helps in using the \bayesianagent\ in a reliable way since the \bayesianagents\ can provide confidence in the actions they would take given the observations. Such a model of \bayesianagent\ is also controllable with respect to the cost function associated with the Bayesian engine, as we illustrate below. 

\textbf{Notation: }
${\statevar}^\prime$ denotes the transpose of vector $\statevar$.
$\diag(\bf{\statevar}')$ denotes a diagonal matrix with $\bf{\statevar}$ as diagonal entries. Capital letters (e.g. $\observationmatrix$) denote matrices and $\observationmatrix_\observation$ denotes $\observation^{\text{th}}$ row of the matrix. 

\subsection{Bayesian Utility Maximization Model for LLMA}
We consider a \bayesianagent\ composed of a large language model (LLM) and a Bayesian Engine.

 \subsubsection{Abstracting a Large Language Model (LLM) as a sensor.}
First, we give a mathematical model for a general-purpose LLM.   For this paper, we consider blackbox access to the LLM and hence both open-source LLMs like LLaMA3~\cite{touvron2023llamaopenefficientfoundation} and closed-source LLMs like ChatGPT. One of the ways to model a blackbox LLM is as an input-output block. The input to an LLM is a single text prompt, which we decompose into three things: the \textit{system prompt}, which we refer to as \textit{control}, the \textit{context}, and the \textit{user prompt}, which we refer to as \textit{observation}. 

Assume that the dictionary of all words of the LLM (tokens\footnote{Although the input of the LLM is text, in most of the architectures, the text (or more generally multi-modal data) input is first decomposed into different \textit{tokens} and the tokens are processed, but since we consider the LLM as a blackbox, our interpretable models abstract these implementation details.}) is given by $\tokenspace$, and this dictionary also includes the blank word. The time index is given by $\timeindex = 1, 2, \dots$.
A blackbox LLM can be viewed as an input-output block.

 The control or the system prompt is an input to the LLM, often prepended before the in-context examples and the user prompt, which is used to give instructions to the LLMs on how exactly to respond. This is used to control the behavior of the LLM and ensure that the LLM behaves (outputs) as required. We assume a control of length $\lengthsystem$ and at time $\timeindex$ denote the control by $\systemprompt_\timeindex \in \tokenspace^{\lengthsystem}$. 

 Following the system prompt, the next input is a context to the LLM. This context could contain external information and examples that are time-dependent and may depend on previous interactions, which is dynamic and can not be put as a part of the system prompt. We consider a context of length $\lengthcontext$ and at time $\timeindex$ denote it by $\context_\timeindex \in \tokenspace^{\lengthcontext}$.

Finally, the user prompt, which we also refer to as the private observation, is the text sequence to which the LLM is supposed to give a response conditioned on the control $\systemprompt_\timeindex$ and the context $\context_\timeindex$. The private observation at time index $\timeindex$ is given by $\llmobservation_\timeindex \in \observationspace^{'}$. Where $\observationspace^{'}$ is the text observation space and for a maximum length of $\lengthuser$, $\observationspace^{'} = \tokenspace^{\lengthuser}$.

We consider an LLM, which is pre-trained on trillions of tokens of text to autoregressively generate the next-token\footnote{There are other techniques of generating the token, but they can be accommodated by a suitable augmentation in the mathematical formulation presented here.}. For developing a token of length 1, the output of the LLM is given by a conditional probability distribution,
\[\llmprob(\systemprompt_\timeindex,\context_\timeindex,\llmobservation_\timeindex) = \probabilitymeasure(\cdot|\systemprompt_\timeindex,\context_\timeindex,\llmobservation_\timeindex),\]
where $\llmobservation\in\observationspace'$ is the user prompt, $\systemprompt_\timeindex$ is the system prompt and $\context_\timeindex$ is the context. Therefore the function $\llmprob$,
\[\llmprob:\tokenspace^{\lengthsystem}\times\tokenspace^{\lengthcontext}\times\tokenspace^{\lengthuser}\to \probspace(\tokenspace),\] outputs a probability distribution over the dictionary $\tokenspace$. 
For generating an output of an LLM with length $>1$, we consider a function $\llmoutputfn$ which takes in the function $\llmprob$ and outputs tokens from the space $\tokenspace^{\lengthoutput}$, where $\lengthoutput$ is the maximum length of the output.  The output of the LLM, denoted by $\observation$ is obtained as $\observation=\llmoutputfn(\llmprob,\systemprompt_\timeindex,\context_\timeindex,\llmobservation_\timeindex)$. 
Therefore, we can represent a black box LLM with the following tuple,
\begin{align}\label{eq:llm}
    \llm = (\tokenspace,\lengthsystem,\lengthcontext,\lengthuser,\lengthoutput,\llmprob,\llmoutputfn).
\end{align}
Bayesian inference involves estimating an unknown state $\statevar \in \statespace$ where $\statespace$ is the state space using observations of the state. If we are performing Bayesian inference using text observations, LLM can be a powerful tool, as illustrated below in Example~\ref{ex:customer} and Example~\ref{ex:financial}.  The LLM can be directly used to infer an underlying state. In fact, it is already used to do so, as highlighted before in the motivation section. However, we now remark on the challenges with directly using an LLM using the formalism described above: 
\begin{enumerate}
\item Current Large Language Models (LLMs) in production lack the ability to explicitly express confidence in their generated outputs. This limitation significantly hinders their deployment in critical domains like finance and healthcare, where system safety and reliability are paramount.
\item Furthermore, the absence of explicit confidence scores presents a significant challenge for human-in-the-loop systems. These systems rely on the ability to seamlessly integrate human expertise when an LLM's confidence in its response falls below a certain threshold. Without this crucial feedback mechanism, it becomes difficult to effectively leverage human oversight. Finally, exclusive reliance on a single LLM for critical tasks is inherently risky. LLMs are vulnerable to adversarial prompting, where carefully crafted inputs can induce them to generate incorrect or misleading outputs. This vulnerability underscores the need for robust mechanisms to assess and mitigate the risks associated with LLM-based systems.
    \item  For a task where the LLMs are being used to infer the state given a sequential stream of text observations, the LLMs alone do not explicitly make use of the temporal nature of the observations and can not characterize how many of such observations are enough to be sufficiently confident in the estimate of the underlying state. 
\end{enumerate}

Therefore, we propose using the LLM as part of a mechanism (\bayesianagent) where, in addition to the LLM, there is a Bayesian engine that tackles the challenges enlisted above. This ensures that the \bayesianagent\ acts not just based on the Bayesian estimate of the underlying state, making the decisions of the \bayesianagent\ more interpretable. 

To this end, we propose using the LLM as a map from the high-dimensional space $\observationspace'$ to a low-dimensional space $\observationspace$ which takes value from whose cardinality $\observationdim \ll \tokenspace^{\lengthuser}$.  We construct a Bayesian engine composed of two probability distributions, the prior of the state, $\prior \in \probspace(\statespace)$  and the likelihood of a low-dimensional observation $\observationmatrixllm: \statespace \to \probspace(\observationspace) $ and is a conditional probability distribution.  We use $\observationmatrix$ and $\observationmatrixllm$ to denote the observation matrices in the rest of the paper.

From an implementation standpoint, there are different ways we can take to ensure that the output of the LLM is from a low-dimensional space. The low-dimensional output could be a structured output like JSON or a Python dictionary. One way is provide the LLMs with a few examples of the type of outputs we want, for instruction-tuned LLMs this technique has been shown to be very useful~\cite{min_recent_2023}. Next we can restrict the dictionary space of the output and reject text tokens not in the restricted dictionary, this is refered to as restrictive or constrained decoding and has been shown to be very effective~\cite{constrainedgeneration}. Additionally, a lot of LLMs, including ChatGPT, give explicit access to structured outputs (platform.openai.com/docs/guides/structured-outputs).

The following example illustrates how LLMs can be used to map the text space to a low-dimensional space. We discuss how this low-dimensional space can be constructed by an analyst or engineer so that the posterior computed by the Bayesian engine is interpretable. 

\begin{example}\label{ex:customer}
    Consider the text dataset of interactions between a customer-service agent and a customer (either a transcript of a call or a chat interface). A service quality engineer is interested in analyzing whether or not the problem of the customer was resolved. The engineer could use an LLM directly and design a prompt to assess if the issue was resolved. However, there is no interpretability, and if the engineer asks the LLM to give the reason for its answer, then it becomes more difficult for the engineer to analyze. The engineer could, however, design a system prompt so that the LLM answers a specific set of binary (yes-or-no) questions. For example, the engineer could ask the following set of questions: 
    \begin{enumerate}
        \item What part of the product was the user concerned about?
        \item Did the user understand the solution being provided?
        \item How satisfied did the last three messages of the customer seem? (attaches last three messages)
        \item Is the solution the best possible solution for the problem?
    \end{enumerate}
    Note that even if we assume the dictionary has  $|\tokenspace|=100$ words, and the conversation has length $\lengthuser = 500$, the dimension reduction is substantial from $500^{100}$ to $2^{4}$.
\end{example}
\subsubsection{Construction of Bayesian Engine}
Owing to dimensionality reduction because of an LLM sensor, we discuss next how a Bayesian engine uses the low-dimensional output to provide an interpretable model based on which an optimal action can be taken. 
\begin{figure*}[h!]
    \centering
    \includegraphics[width=0.8\linewidth]{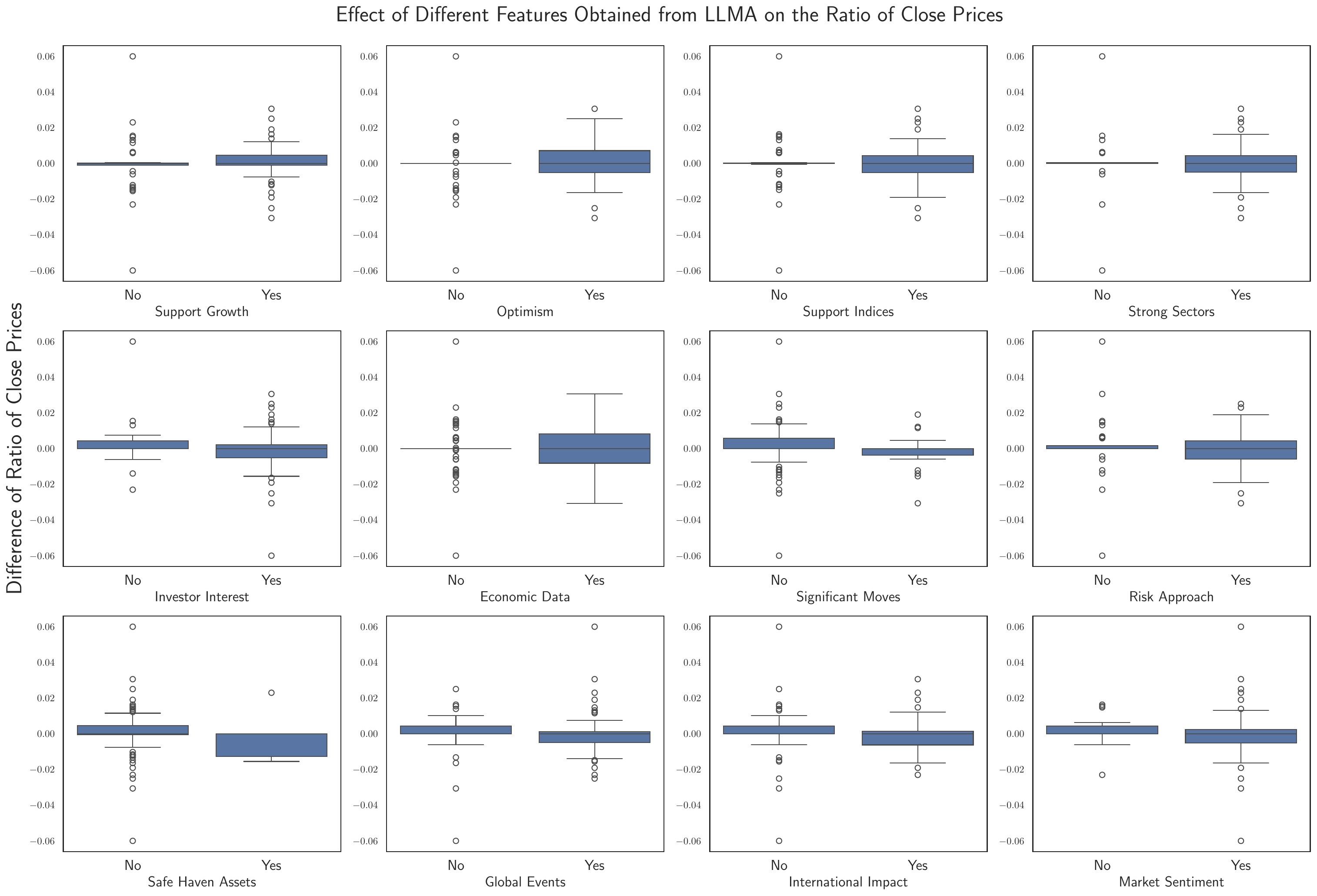}
    \caption{\bayesianagents\ can be used to detect and analyze the change in financial indicators (difference of close prices) by parsing financial news to extract $16$ interpretable features in Example~\ref{ex:financial}. We analyze the news articles from $03/2020$ to $08/2020$ corresponding to the \texttt{AAPL} stock. We query the LLM for $16$ binary with different features, including whether the news article indicates optimism about the market and whether there is investor interest in the stock. The interpretable features can be used to analyze the stock (and subsequently for Bayesian inference), as illustrated by the difference in the ratio of the stock prices across days.  }
    \label{fig:interpretablefigfinance}
\end{figure*}
Note that in general, \bayesianagent\ need not just be used for Bayesian inference but for more general tasks like coding, shopping assistants, research writing, etc., we construct the Bayesian engine to be more general purpose and detail on how a particular cost function leads to Bayesian inference. Also, note that each of the general tasks involves multiple Bayesian inference steps.

For a state $\statevar\in\statespace$, the \bayesianagent\ receives a text observation $\observation^{'}\in\observationspace'$ which the LLM of the \bayesianagent\ parses and provides a low-dimensional output $\observation\in\observationspace$. The \bayesianagent\ has a Bayesian engine which has a prior on the state space given by $\prior$ and an observation likelihood $\observationmatrixllm$. The \bayesianagent\ uses the low-dimensional output from the LLM to compute the posterior using Bayes' rule,
\begin{align*}
  \priorupdate(\prior,\observation) =   \probabilitymeasure_{\observationmatrixllm,\prior}(\statevar|\observation) = \frac{\observationmatrixllm_{\observation,\statevar}\prior(\statevar)}{\sum_{\statevar} \observationmatrixllm_{\observation,\statevar}\prior(\statevar)} .
\end{align*}
Let $\actionspace$ denote the finite action space of the \bayesianagent.
Let $\utility:\statespace\times\actionspace \to \R$ be the utility that the Bayesian agent receives from taking action $\action\in\actionspace$ when the underlying state $\statevar$. Then, the Bayesian agent performs the action, which maximizes the expected utility under the posterior distribution. 
\begin{align*}
    \action = \underset{{\action\in\actionspace}}{\argmax}\sum_{\statevar\in\statespace} \utility(\statevar,\action)  \probabilitymeasure_{\observationmatrixllm,\prior}(\statevar|\observation). 
\end{align*}
Therefore, our \bayesianagent\ can be described as the following tuple, 
\begin{align}\label{eq:bayesianagenttuple}
\text{LLMA}(\llm,\prior,\observationmatrixllm,\utility),
\end{align}
where $\llm$ is an LLM of the form~\eqref{eq:llm}, $\prior$ is the prior over the state, $\observationmatrixllm$ is the observation likelihood and $\utility$ is the utility function. 
We make the following remarks on our interpretable sensor model construction of a \bayesianagent. 
\begin{remark}
   We described the operation of Bayesian inference on a single observation when there is a stream of observations $\observation_1,\dots,\observation_\timeindex$ then the prior is updated after every step with the computed posterior $\prior_{\timeindex+1} =    \priorupdate(\prior_{\timeindex},\observation)$. 
\end{remark}
\begin{remark}
In Section~\ref{sec:sociallearning}, when we describe Bayesian social learning in a sequence of \bayesianagents, we will note that the equations of the social learning protocol are the same as the above equation. However, the prior of all the subsequent agents is updated based on the actions of agents subsequent to them (even when the rest of the agents do not observe a private observation at any given point), which leads to herding. 
\end{remark}
\begin{remark}
    For the case of Bayesian inference of the states, the action space is taken to be $\actionspace=\statespace$, and one of the possible utility functions is the indicator function $\utility(\action, \statevar) = \indicator(\action=\statevar)$.
\end{remark}
\begin{remark}
    The framework presented in this section can be extended to multi-modal models like the vision language models (VLMs~\cite{bordes2024introductionvisionlanguagemodeling}) for more general tasks by accordingly modifying dictionary $\tokenspace$ and output generation mechanism $\llmoutputfn$. 
\end{remark}
\begin{remark}
    The likelihood $\observationmatrixllm$ can be computed by using the LLM on a set of synthetic or public offline data where we simulate the state and use the text observations to obtain the low-dimensional observations from the LLM. 
\end{remark}
\subsubsection{Illustrative Example for Interpretable Feature Extraction using LLMAs}
In example~\ref{ex:customer}, we discussed how the low-dimensional representation can be constructed to reduce the observation space. The example can be extended to illustrate Bayesian inference, for example, by analyzing the performance of a particular customer service agent given its interactions with different customers. We next present an example for a different application, financial news analysis using \bayesianagents. 
\begin{example}\label{ex:financial}
    Consider a financial analyst who receives a stream of financial news and public opinion data from social media. The state in which the financial analyst wishes to estimate if the market is in an upturn or downturn. Similar to example~\ref{ex:customer}, the analyst can design questions using her domain knowledge, which extracts relevant information from the text. Since the stream of data is temporal in nature, the analyst can use our model of a \bayesianagent, to adaptively update the belief of the underlying state, and any point interpet the interminent outputs $\observation$ to identify trading opportunities. 
    We consider FNSPID, a financial news dataset~\cite{10.1145/3637528.3671629}, where we analyze news pertaining to the AAPL stock. We use LLaMA-3, and for each news article, ask $16$ binary questions. The questions are provided in the appendix, and we plot the distribution of the difference of the ratio of close prices (a performance metric used to gauge the performance of a stock across days) in Figure~\ref{fig:interpretablefigfinance}. It is clear how LLM can be used as a sensor to parse textual observations and extract interpretable features. 
    \change{The illustrative example using the financial news dataset focuses on the AAPL stock to demonstrate the applicability of the study to readers who may be less familiar with the subject or come from a different disciplinary background, such as finance. We note that the period that we consider was during COVID-19, and therefore was an exception to the standard stock market behaviour. However we again highlight that the purpose of our methods is to enable analysis by providing an interpretable toolbox and not to analyse any specific event. }
\end{example}

\subsection{Summary}
In order to construct an interpretable model of a \bayesianagent\ performing Bayesian inference, we use the LLM as a sensor attached to a Bayesian engine. 
This section gave a mathematical model for the \bayesianagent, which is modeled as Bayesian sensors. The \bayesianagent\ we consider in this paper is composed of an LLM and a Bayesian engine. 
We assumed that the entity interested in using the \bayesianagent\ has access to the observation matrix described in this section, $\observationmatrixllm$. However, this might not be the case where the \bayesianagent\ is used by a third party, and the Bayesian engine might not be explicit, but still, the entity might be interested in controlling and understanding the actions (decision) of the \bayesianagent. We, therefore, discuss in the next section under what conditions we can reconstruct the utilities of the \bayesianagent\ by probing the \bayesianagent\ given blackbox access. This extends the work done in explainable machine learning, where the deep neural network is modeled as a rationally inattentive Bayesian utility maximizer, and the post-training classifications are explained using the utilities obtained from the Bayesian revealed preferences framework~\cite{JMLR:v24:20-1202,pattanayak}.

\section{LLM Agent as a Rationally Inattentive Bayesian Utility Maximizer}\label{sec:ribum}
There are intriguing parallels between self-attention in LLMs and rational inattention in microeconomics. Self-attention is a special type of rational inattention mechanism. Therefore, we use Bayesian revealed preferences from microeconomics to estimate the utilities of a \bayesianagent, which can then be used to understand and control its behavior. 

Motivated by the inherent self-attention mechanism of an LLM, this section discusses how an interpretable model for a single \bayesianagent\ is to model to them as Rationally Inattentive Bayesian Utility Maximizers. 
First, we state the protocol that a Bayesian agent who is rationally inattentive follows. Then, we consider the problem of the viewpoint of an analyst who only observes the states of nature and the actions of a \bayesianagent\ and wishes to analyze if the \bayesianagent\ is a Rationally Inattentive Bayesian Utility Maximizer. For this, we state the necessary and sufficient conditions for \bayesianagents\ to act as Rationally Inattentive Bayesian Utility Maximizers.  Finally, we discuss algorithms that can be used to get a max-margin estimate and a sparse estimate of the utility function. A few illustrative examples are presented on real-life datasets to explain how the framework can be practically used to systematically obtain utilities of a \bayesianagent\ and also a \textit{standalone} LLM.  

\subsection{Motivation. Self-Attention Mechanism of LLM.}

The LLM of the \bayesianagent\ is driven by a transformer neural network. The key innovation of the transformer neural network is the self-attention mechanism. Self-attention allows a model to focus on different input parts when processing each token. In NLP tasks, for example, it helps a model understand which tokens in a sentence are important in relation to a given token. This relation is then to autoregressively generate texts~\cite{10.5555/3295222.3295349}.

In microeconomics, on the other hand, rational inattention is used to model the behavior of individuals managing cognitive resources by prioritizing certain information while ignoring less relevant details due to the inherent "cost" of processing. This is akin to how self-attention mechanisms in machine learning assign weights to different parts of an input sequence, prioritizing relevant segments to optimize understanding or prediction. Both processes are fundamentally about efficient allocation: rational inattention models decisions based on the economic trade-off of information processing costs, while self-attention models adaptively weigh parts of data to capture context, streamlining processing.

Motivated by the inherent self-attention mechanism central to the LLM transformer architecture~\cite{llmaccesssurvey}, the \bayesianagent\ can be modeled as a rationally inattentive Bayesian utility maximizer using the microeconomics model of Bayesian revealed preferences. Rational inattention is about constrained human decision-making due to limited attention, while self-attention is about LLMs assigning attention weights to different parts of input data to optimize understanding or predictions. 

Another motivation for studying \bayesianagents\ from Bayesian revealed preferences~\cite{kaplinanddeen} is to model the cost-accuracy tradeoff that is inherent in analyzing large amounts of data done by the LLM of an \bayesianagent. Namely, a \bayesianagents\ can output more accurate outputs by increasing the attention effort expended (by a better observation matrix using an LLM with a larger number of parameters or a larger context window). Such a rationally inattentive model for Bayesian agents was first proposed by Nobel laureate Christopher A. Sims. We present the theoretical framework of Bayesian revealed preferences and present experiments on real-life datasets using \bayesianagents. 

\subsection{Rationally Inattentive Bayesian Utility Maximizing Agent}\label{sec:ribumdesc}
In the last section, we discussed a Bayesian Sensor model of a single \bayesianagent. If the \bayesianagent\ is designed by the entity that is deploying them, then the utility function of the \bayesianagent\ can be set manually by the cost of the function $\utility$ of the Bayesian engine. However, if these \bayesianagents\ are used off the shelf with or without an explicit Bayesian mechanism, then the utility function is unknown to the entity using them. Although heuristics like confusion matrices or domain knowledge-based cost functions can be used, we need a more systematic approach to estimate the utility functions of a \bayesianagent. 

We now present the model of an agent who is a rationally inattentive Bayesian utility maximizer (\ribum)~\cite{pattanayak,JMLR:v24:20-1202}.

Consider a state $\statevar$ belonging to a finite state space $\statespace$ which is sampled from a prior denoted by $\prior_0 \in \probspace(\statespace)$. A \ribum\ operates in $\numenvironments$ environments indexed by $\environment$ and performs an action (denoted by $\action$) from the finite set $\actionspace$. The utility functions of the \ribum\ are given by $\utility_{\environment,\action} = [\utility_{\environment,\action}(1,\action),\dots,\utility_{\environment,\action}(\statedim,\action)]^{\prime}$ for each action $\action\in\actionspace$ and each environment $\environment \in \{1,\dots,\numenvironments\}$. The \ribum\ observes the states $\statevar$ through observations $\observation$ from a finite observation space $\observationspace$. For a given observation matrix (which is a stochastic matrix) $\observationmatrix= (\observationmatrix_{\statevar\observation}=\actionposterior(\observation|\statevar),\statevar\in\statespace,\observation\in\observationspace)$, the \ribum\ has an information acquisition cost $\infoacquisitioncost(\observationmatrix,\prior)$. 
Let $\observationmatrix_\observation = \diag(\observationmatrix_{1\observation},\dots,\observationmatrix_{\statedim\observation}),\observation\in\observationspace$ denote the different probabilities for observing a particular $\observation \in \observationspace$. Let $\stochastickernels$ denote the set of all $\statedim\times\observationdim$ stochastic kernels $\observationmatrix$. 

Then the \ribum\ follows the following protocol:
\begin{enumerate}
    \item (Step 1) The \ribum\ first optimizes for the observation matrix by maximizing the expected utility regularized by the information acquisition cost. This optimization is given by, \begin{align}\label{eq:observationoptimization}
    \begin{split}
        \observationmatrix(\environment) &\in \underset{\observationmatrix\in\stochastickernels}{\argmax}\ \expectedreward(\utility_\environment,\observationspace,\prior_0) - \infoacquisitioncost(\observationmatrix,\prior_0)\\
        \expectedreward(\utility_\environment,\observationspace,\prior_0) &\triangleq \expectation\{\max_{\action\in\actionspace} \expectation\{\utility_\environment(\statevar,\action)|\observation\}\} \\&= \sum_{\observation\in\observationspace}\max_{\action\in\actionspace}\utility_{\environment,\action}^{'} \observationmatrix_{\observation}\prior_0.
    \end{split}
\end{align}
    \item (Step 2) A state of nature $\statevar^0 \in \statespace$ is drawn from prior $\prior^0$ and is not known to the \ribum. 
    \item (Step 3) The \ribum\ draws an observation $\observation$ from optimized observation likelihood $\observationmatrix_{\statevar^0\observation}(\decisionspace)$.
    \item (Step 4) Given an observation $\observation$ the \ribum\ computes its posterior using Bayesian update step: \begin{align}\label{eq:ribumpriorupdate}
    \prior = \priorupdate(\prior_0,\observation,\environment) \triangleq \frac{\observationmatrix_\observation(\environment)\prior_0}{\ones^\prime \observationmatrix_\observation(\environment) \prior_0},
\end{align}
where $\ones$ is a row vector $[1,\dots,1]^\prime$.
    \item (Step 5) Finally, the \ribum\ performs the action maximizing the expected utility where the expectation is taken with respect to the posterior computed in Step 4, 
    \begin{align}\label{eq:actionutilitymaxm}
        \action \in \argmax_{\action^\prime} \expectation\left\{\utility_\environment(\statevar,\action^\prime)|\observation\right\} = \argmax_{\action' \in \actionspace} \utility^\prime_{\environment,\action}\observationmatrix_{\observation}(\environment)\prior_0.
    \end{align}
\end{enumerate}
Therefore, given the above protocol a \ribum\ agent can be parameterized by the following tuple~\cite{krishnamurthy_partially_2016},
\begin{align}\label{eq:ribum}
    (\environmentspace,\statespace,\observationspace,\actionspace,\prior_0,\infoacquisitioncost,\{\observationmatrix(\environment),\utility_\environment, \environment\in\environmentspace\}).
\end{align}

We now make several remarks on the above protocol.

\begin{remark}
   The information acquisition cost $\infoacquisitioncost(\observationmatrix,\prior_0)$ of Step~1 can be considered as the sensing cost that \ribum\ incurs in acquiring the information to make the decision on which action $\action$ to perform. The information acquisition cost can also be interpreted as, 
   \begin{align*}
\infoacquisitioncost(\observationmatrix,\prior_0) = \sum_{\observation} \entropyfunction(\priorupdate(\prior_0,\observation,\environment),\prior_0) \indicator^\prime \observationmatrix_\observation(\environment) \prior_0,
   \end{align*}
   where $\entropyfunction$ is an entropic regularizer (e.g., mutual information or Renyi Entropy)~\cite{JMLR:v24:20-1202}. Intuitively, a higher information cost is incurred for a more accurate attention strategy since we obtain a more accurate estimate of the state~\cite{pattanayak}. 
\end{remark}
\begin{remark}
   Using the above interpretation of the information acquisition cost, the optimization of~\eqref{eq:observationoptimization} can be seen as the \ribum\ agent optimally choosing the observation sampling strategy. This strategy is chosen to maximize the expected utility regularized by a rational inattention cost.  
\end{remark}
\begin{remark}
We now remark on the correspondence between the model of the \bayesianagent\ from~\eqref{eq:bayesianagenttuple} and the \ribum\ tuple of \eqref{eq:ribum}. First note that even a single LLM $\llm$ of the form~\ref{eq:ribum} can be considered as a RIBUM where the $\llm$ optimizes its self-attention matrix to optimize for picking the tokens that best predict the next token using the conditional probability distribution $\llmprob$. Further the \bayesianagent\ can be considered as a product of three observation matrices: 1) from the state to the text observations 2) from the $\llm$ (as described above) and 3) from the Bayesian engine, $\observationmatrixllm$. 1) is not in control of the \bayesianagent; however, it is known to the analyst who simulates the text observations for a given state $\statevar$. 2) comes from the pretraining of the LLM, which involves minimizing an entropic loss with respect to the self-attention mechanism. 3) comes from the  \bayesianagent\, which has a pre-trained likelihood function on a suitable dataset. 
\end{remark}
\subsection{Viewpoint of Analyst}
An analyst who observes the actions (\textit{behavior}) of the \bayesianagent\ under different states aims to ascertain if the \bayesianagent\ behaves as a \ribum. In particular, the analyst the following dataset, 
\begin{align}\label{eq:dataset}
    \dataset = \left\{ \prior_0, \actionposterior_\environment(\action|\statevar), \statevar \in \statespace, \action \in \actionspace, \environment \in \environmentspace\right\}.
\end{align}
where $\prior_0$ denotes the prior distribution of the state and $\actionposterior_\environment(\action|\statevar)$ denotes the conditional probability of performing the action $\action$ given the state $\statevar$ and environment $\environment$. The joint probability of the state-action pair ($\statevar$,$\action$) is given by
$\actionposterior_\environment(\action,\statevar) = \prior_0(\statevar)\actionposterior_\environment(\action|\statevar)$ and the probability of state $\statevar$ given action $\action$ by,
\[
\actionposterior_\environment(\statevar|\action) = \frac{\actionposterior_\environment(\action,\statevar) }{\sum_{\bar{\statevar}}\actionposterior_\environment(\action,\bar{\statevar}) }.
\]
\begin{remark}
   In practice, given the state-action pairs, the analyst empirically estimates the action posterior $\actionposterior_\environment(\action|\statevar)$, and by Kolmogorov’s law of large numbers, the empirical estimate converges to the true distribution w.p. 1 as the number of the state-action pairs goes to infinity.
\end{remark}
Given the dataset $\dataset$, the analyst aims to a) check if the \bayesianagent\ is a \ribum\ and b) if \bayesianagent\ is indeed a \ribum\ then obtain the reconstructed utility (reward) function $\hat{\utility}$ which rationalizes the behavior of the \bayesianagent. 

We next state the necessary and sufficient conditions for a) from Bayesian Revealed Preferences~\cite{kaplinanddeen} and then discuss two algorithms describing how b) can be performed. 
\subsubsection{Necessary and Sufficient Conditions for Rationally Inattentive Bayesian Utility Maximizing Behaviour}
As proved in the seminal work of~\cite{kaplinanddeen}, there are two inequalities, the No Improving Action Switches (NIAS) and the No Improving Action Cycles (NIAC), which are necessary and sufficient for an agent (in our case \bayesianagent) to be a \ribum. 

We now state the NIAS and NIAC conditions and provide intuition  for both of them, 

\textbf{No Improving Action Switches (NIAS)}
    \begin{align}\label{eq:nias}
        \sum_{\statevar} \actionposterior_\environment(\statevar|\action)(\reconstructedreward_\environment(\statevar,\bar{\action})-\reconstructedreward_\environment(\statevar,\action)) \leq 0 \ \forall \action,\bar{\action}\in \actionspace, \environment  \in \environmentspace.
    \end{align}
\begin{remark}
    The NIAS condition enforces that for any environment $\environment$, the agent chooses the optimal action with respect to the posterior probability mass function. 
\end{remark}
\textbf{No Improving Action Cycles (NIAC)}
 \begin{align}\label{eq:niac}
 \begin{split}
     \sum_{\action} &\max_{\bar{\action}}\sum_{\statevar}\actionposterior_\environmentalt(\statevar,\action)\reconstructedreward_\environment(\statevar,\bar{\action}) - \reconstructedinfocost_\environmentalt \\&- \left(\sum_{\action} \sum_{\statevar}\actionposterior_\environment(\statevar,\action)\reconstructedreward_\environment(\statevar,\action)-\reconstructedinfocost_\environment\right) \leq  0  \ \forall \environmentalt,\environment  \in \environmentspace.
 \end{split}
\end{align}
\begin{remark}
    The NIAC inequality operates on pairs of environments and ensures that the agent has an attention strategy that is optimal for all $\numenvironments$ environments. The above inequality is a pairwise version from~\cite{JMLR:v24:20-1202} of the original combinatorial inequality of~\cite{kaplinanddeen}. The combinatorial inequality gives a more intuitive explanation for the same, wherein the agent takes actions that are consistent across all possible subsets of the environments. Here, consistency is with respect to the action posterior. Intuitively,  NIAC ensures that every agent chooses the best attention strategy in a given environment.
\end{remark}

We now state the main results, which show that for a \bayesianagent\ to be a \ribum, it is sufficient to check if the dataset $\dataset$ obtained from the \ribum\ satisfies the NIAS and NIAC conditions. We summarize this feasibility check in Algorithm~\ref{alg:brp}, where the input is the dataset $\dataset$ of the form~\eqref{eq:dataset}, obtained from the \bayesianagent\ and ascertains if the \bayesianagent\ is a \ribum\ or not. 
\begin{theorem}\label{th:nscribum}
    \textbf{(Necessary and Sufficient Conditions for \bayesianagent\ to be a \ribum)} Let $\dataset$ be the dataset that the analyst has, as described in~\eqref{eq:dataset} for a \bayesianagent\ performing protocol (Step 1 to 5 of Sec.\ref{sec:ribum}) in $\numenvironments\geq 2$ environments. Then the \bayesianagent\ is a \ribum\ iff there exists a feasible solution $\{\reconstructedreward_\environment(\statevar,\action),\reconstructedinfocost_\environment(\statevar,\action)|\statevar\in\statespace,  \action\in\actionspace\}_{\environment=1}^{\numenvironments}$ to the NIAS inequality of~\eqref{eq:nias} and the NIAC inequality of~\eqref{eq:niac}.
\end{theorem} 
The result was first derived in~\cite{kaplinanddeen} and has been used extensively to verify if different engineering systems, including RADARs and Deep Learners, are \ribum\ or not~\cite{pattanayak,JMLR:v24:20-1202}. 
\begin{remark}
    The above theorem gives an if and only if condition for a \bayesianagent\ to be a \ribum. If the inequalities have a feasible solution, then there exists a reconstructable set of utilities and information acquisition costs that rationalize $\dataset$. The necessity implies that for a \ribum, the true utilities satisfy the NIAC and NIAS conditions; hence, Theorem~\ref{th:nscribum} yields consistent estimates of the utilities.
\end{remark}
\begin{algorithm}
    \begin{algorithmic}
        \State \textbf{Input: } Dataset $\dataset$ of the form~\eqref{eq:dataset} from \bayesianagent.
        \State \textbf{Ascertain: }If $\exists\ \reconstructedreward_\environment$ and $\reconstructedinfocost_\environment \forall \environment\in\environmentspace$ satisfying the NIAS inequality from~\eqref{eq:nias} and NIAC inequality from~\eqref{eq:niac}
    \end{algorithmic}
    \caption{Bayesian Revealed Preferences Feasibility (BRP)}
    \label{alg:brp}
\end{algorithm}
\begin{remark}
   The feasibility check is summarized in Algorithm~\ref{alg:brp} which is derived from Theorem~\ref{th:nscribum}. Any utility that satisfies the NIAS and NIAC inequalities with respect to the dataset $\dataset$ is a feasible utility. Hence, the Bayesian revealed preference returns a set-valued estimate (rather than point estimates) of the true utility. This set-valued estimate is given by $\reconstructedreward$ and the reconstructed information cost $\reconstructedic$ can be derived as~\cite{krishnamurthy_partially_2016}, 
    \begin{align}\label{eq:reconstructedic}
        \begin{split}
&\reconstructedic(\dataset) = \max_{\environment\in\environmentspace} \biggl( \reconstructedinfocost_\environment +\\
\ &\sum_{\action}\max_{\bar{\action}\in\actionspace}\sum_{\statevar}\actionposterior_\environment(\statevar,\action)\reconstructedreward_\environment(\statevar,\bar{\action}) - \sum_{\action}\sum_{\statevar}\actionposterior_\environment(\statevar,\action)\reconstructedreward_\environment(\statevar,\action) \biggr).
        \end{split}
    \end{align}
   
\end{remark}
\begin{figure*}
    \centering\includegraphics[width=0.6\linewidth]{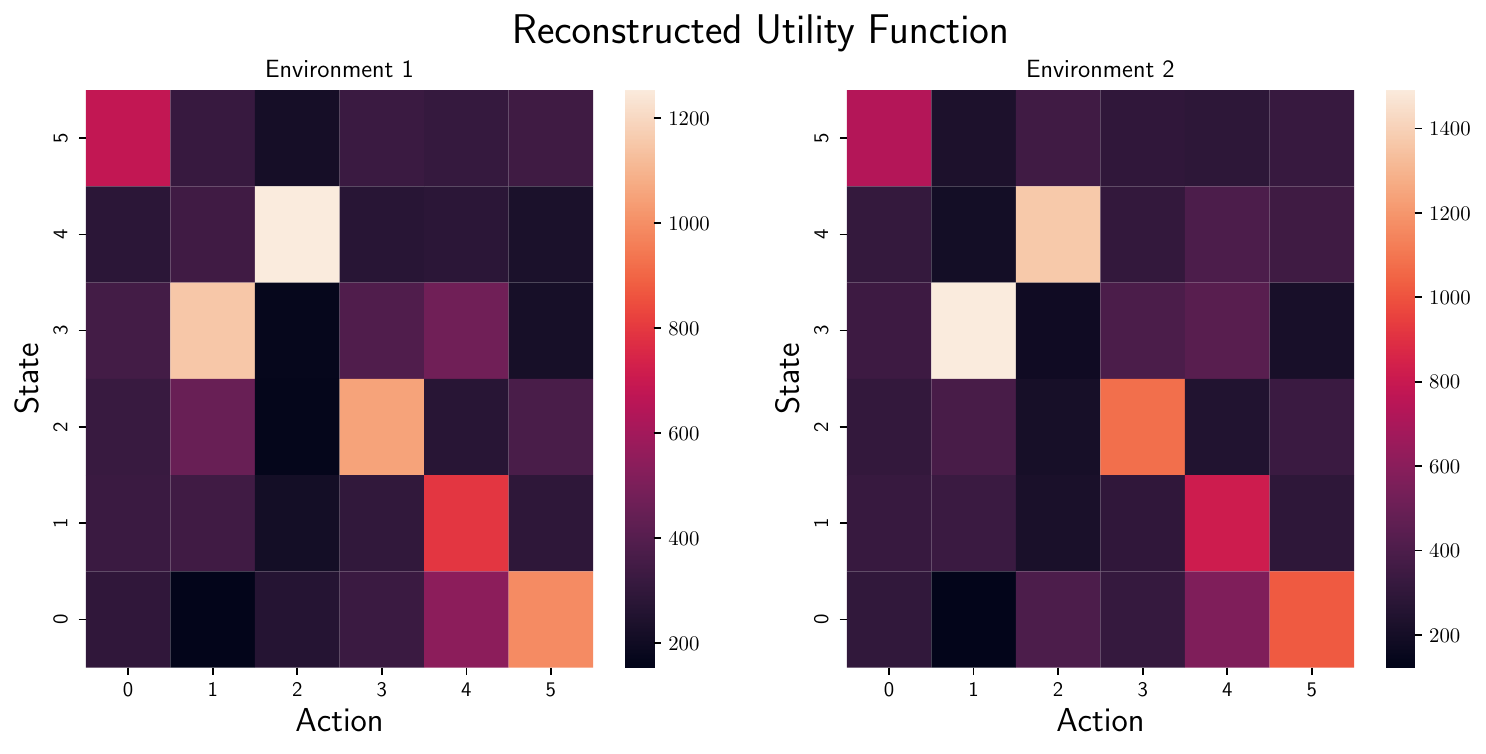}
    \caption{Reconstructed Max-Margin Utility of a \bayesianagent\ for Illustrative Example~\ref{ex:ribum}. The utilities near the diagonal are comparatively higher than off-diagonal entries, showing that \bayesianagent\ correctly classifies the states. Along the diagonal, utility is highest for the state $5$ which is the most toxic, hence the \bayesianagent\ gives the highest utility to classifying the most toxic state.    }
    \label{fig:reconstructedutilityexample}
\end{figure*}

Although Theorem~\ref{th:nscribum} is an extremely powerful result, for practical applications, we often need a single utility estimate rather than a set-valued estimate. We next discuss two different methods to obtain a point estimate for the utility: the max-margin method and the sparsest utility estimation.

\subsection{Estimating Set-Valued and Point Estimates for the Utilities}
Estimating the utility function of the \bayesianagents\ is an inverse optimization or inverse reinforcement learning problem, which is, in general, ill-posed. Hence, instead of reconstructing point-valued utility estimates, we describe how the BRP test can be used to reconstruct set-valued utility estimates. 
Each of the utilities in the set-valued estimate is a feasible utility. We provide two algorithms that return a point-valued estimate from the set, which satisfy certain other structural properties.

Firstly, we describe the max-margin approach to reconstruct the utility function. Since trivial utilities can satisfy the NIAS and NIAC conditions, we maximize the margin with which each condition is satisfied, denoted by decision variables $\epsilon_1$ and $\epsilon_2$, respectively. Specifically, we consider the following convex program summarized in Algorithm~\ref{alg:maxmargin},
\begin{align}
\begin{split}\label{eq:maxmargin}
&\underset{\{\reconstructedreward_\environment(\statevar,\action),\reconstructedinfocost_\environment(\statevar,\action)\}_{\environment=1}^{\numenvironments},\epsilon_1,\epsilon_2}{\argmin} (\epsilon_1 + \epsilon_2)\\&
       \nias(\cdot) \leq -\epsilon_1, \niac(\cdot) \leq - \epsilon_2, \\& \ \epsilon_1,\epsilon_2 > 0. 
\end{split} 
\end{align}
The max-margin formulation is especially useful when the analyst wishes to estimate the utilities of \bayesianagent\, which pass the NIAC and NIAS tests maximally. There is no guarantee that this will be close to the true utility because of the bias in the observed data. Still, this estimate is a useful reconstruction that has shown to work well in practice~\cite{pattanayak,JMLR:v24:20-1202}. 
\begin{algorithm}
    \begin{algorithmic}
        \State \textbf{Input: } Dataset $\dataset$ of the form~\eqref{eq:dataset} from \bayesianagent.
        \State \textbf{Output: } Reconstructed utilities $\reconstructedreward_\environment$ and information acquisition cost $\reconstructedic$ of the \bayesianagent\ in  $\environment\in\environmentspace$ environments; Margins for NIAS  $\epsilon_1$ and NIAC $\epsilon_2$. 
        \If{BRP($\dataset$) is \texttt{True}}
        \State \textbf{solve: } Optimization~\eqref{eq:maxmargin} for $\reconstructedreward_\environment$, $\reconstructedinfocost_\environment \forall \environment\in\environmentspace$, $\epsilon_1$, $\epsilon_2$
        \State Obtain $\reconstructedic$ from equation~\eqref{eq:reconstructedic}
        \Else{}
        \State \textbf{return: }\textit{Feasibility Error}: \bayesianagent\ is not a \ribum. 
        \EndIf
    \end{algorithmic}
    \caption{Max-Margin Utility Reconstruction for \bayesianagent}
    \label{alg:maxmargin}
\end{algorithm}

Next, we describe the utility reconstruction method, which minimizes the $\ell_1$-norm of the reconstructed utility so as to obtain a sparse representation. Here, we manually set tolerances \change{ $\epsilon_1\in\R^+$ and $\epsilon_2\in\R^+$ for the margins with which each condition is satisfied. Ideally, the margin should be as high as possible such that the linear inequalities still have a feasible solution. Therefore, some trial and error is required to select a suitable $\epsilon_1$ and $\epsilon_2$. However, this is possible since the Bayesian revealed preference step is offline.} We minimize the following convex program, and we summarize the procedure in Algorithm~\ref{alg:sparse},
\begin{align}\label{eq:l1norm}
\begin{split}
&\underset{\{\reconstructedreward_\environment(\statevar,\action),\reconstructedinfocost_\environment(\statevar,\action)\}_{\environment=1}^{\numenvironments}}{\arg\min} \sum_{\environment\in\environmentspace,\action \in \actionspace,\statevar\in\statespace} |\reconstructedreward_\environment(\statevar,\action)|\\&
       \nias(\cdot) \leq -\epsilon_1, \niac(\cdot) \leq - \epsilon_2 
\end{split}.
\end{align}
The sparsest utility reconstruction is especially useful when the analyst is interested in understanding the key state-action pairs that the \bayesianagent\ finds especially useful. This sparse utility can be informative in focussingdg the design of the environment and the system prompt that the \bayesianagent\ uses. 

\begin{algorithm}
    \begin{algorithmic}
        \State \textbf{Input: } Dataset $\dataset$ of the form~\eqref{eq:dataset} from \bayesianagent,  Margins for NIAS  $\epsilon_1$ and NIAC $\epsilon_2$. 
        \State \textbf{Output: } Reconstructed utilities $\reconstructedreward_\environment$ and information acquisition costs $\reconstructedic$  of the \bayesianagent\ in  $\environment\in\environmentspace$ environments;
        \If{BRP($\dataset$) is \texttt{True}}
        \State \textbf{solve: } Optimization~\eqref{eq:l1norm} for $\reconstructedreward_\environment$, $\reconstructedinfocost_\environment \forall \environment\in\environmentspace$
        \State Obtain $\reconstructedic$ from equation~\eqref{eq:reconstructedic}
        \Else{}
        \State \textbf{return: } \textit{Feasibility Error: } \bayesianagent\ is not a \ribum. 
        \EndIf
    \end{algorithmic}
    \caption{Sparsest Utility Reconstruction for \bayesianagent }
    \label{alg:sparse}
\end{algorithm}
\begin{remark}
    Note that the above Bayesian revealed preferences framework can be used even if the \bayesianagent\ does not have an explicit Bayesian engine. This is because, as remarked above, even a standalone LLM has an observation matrix from its pertaining, and the Algorithm~\ref{alg:brp} only requires access to the state-action pairs from interacting with the \bayesianagent. 
 \end{remark}
 \begin{remark}
Algorithm~\ref{alg:brp} is a feasibility test with $\numenvironments(|\actionspace||\statespace|+1)$  free variables and $\numenvironments^2 + \numenvironments(|\actionspace|^2 - |\actionspace|-1)$ linear inequalities. The number of free variables and inequalities in the feasibility test of Algorithm~\ref{alg:brp} scale linearly and quadratically, respectively, with the number of environments, $\numenvironments$.
     
 \end{remark}
 \begin{remark}
 \change{
During pre-training, LLMs are trained on massive datasets like the CommonCrawl, which are of the order of hundreds of terabytes, and therefore, they can produce outputs for different contexts. Therefore there is no sample size restriction on the dataset for applying Bayesian revealed preferences to LLMAs. It is important to emphasize the Bayesian revealed preferences provide a necessary and sufficient condition for Bayesian utility maximization, thereby providing a rigorous data analysis tool for real-world data. The primary requirement of both algorithms is that the dataset must be collected in at least two environments, with the action posterior properly defined. There are no additional prerequisites for reconstructing the set-valued utilities from the action posteriors. However, in practice, when the action posterior is estimated empirically, errors are introduced due to finite sample approximations. Characterizing the exact convergence of empirical action posterior probabilities to their true values would be an intriguing direction for future research.}
 \end{remark}
The feasibility condition of Theorem~\ref{th:nscribum} provides a set-valued estimate for the utilities since any utility that satisfies the NIAS and NIAC conditions is a valid utility function. However, for many applications, a single utility function is desired. 
Therefore, we now describe two algorithms that use the Bayesian revealed preference framework along with Theorem~\ref{th:nscribum} to obtain the reconstructed utility of the \bayesianagent\ along with information acquisition cost. 
\subsection{Numerical Experiment with LLaMA LLM-based Agent}

We illustrate the Bayesian revealed preferences for \bayesianagents\ in the toy example for hate-speech classification. 
\begin{example}\label{ex:ribum}
    We consider the example of analyzing a \bayesianagent\ which classifies a text into six levels of hate speech. Therefore for this task the action space $\actionspace$ is same as the state space $\statespace =\{0,1,\dots,5\}$. The levels indicate the intensity of hate speech. The details of the different levels and the exact construction of the \bayesianagent\ are given in Section~\ref{sec:numerical}. We then obtain $200$ pairs of state and actions from the \bayesianagent, which forms our dataset $\dataset$. We run the Algorithm~\ref{alg:maxmargin} and provide the reconstructed max-margin estimates of the utility in Figure~\ref{fig:reconstructedutilityexample}. It can be seen that the utilities quantify the observed behavior of the \bayesianagent. 
\end{example}

\subsection{Summary}
The LLM of a \bayesianagent\ uses a entropic regularization method to provide its output, this motivates looking at the \bayesianagent\ from the lens of rationally inattentive Bayesian utility maximization, which is a form of entropic regularized utility maximization. 
This section discussed the necessary and sufficient conditions for a \bayesianagent\ to be a rationally inattentive Bayesian utility maximizer (\ribum). We proposed two algorithms to reconstruct a point-valued estimate of the utility of the \bayesianagent\ if it is a (\ribum). We illustrate how the reconstructed utility is useful in analyzing the behavior of a blackbox \bayesianagent.

\vspace{0.5cm}
{\color{accessblue}{\hrule}} \vspace{0.5cm}

{\hspace{-3mm}\Large\color{accessblue} \textbf{Part II: Interacting LLM Agents}}
\vspace{0.5cm}

\hspace{-3.5mm}We now move on to Part II of the paper. Having studied a single \bayesianagent\ in isolation in Part I, we next study social learning in an interacting network of \bayesianagents. There is a lot of work done in studying different topologies of \bayesianagents~\cite{sumers2024cognitive}. However, we restrict ourselves to the three topologies described in Figure~\ref{fig:sociallearningtopologies}.  We motivate studying the different topologies to understand, analyze, and explain some of the observed phenomena in LLMs. We study Bayesian social learning in a sequence of \bayesianagents\ to analyze information cascade to an incorrect action. Motivated by model collapse observed while training LLMs from their generated dataset, we briefly study a different protocol for social learning in \bayesianagents: word-of-mouth social learning. Finally, we illustrate how data-incest can arise if \bayesianagents\ perform Bayesian inference in an asynchronous fashion.  Part II comprises Section~\ref{sec:sociallearning} and Section~\ref{sec:wordofmouth}.
\begin{figure*}
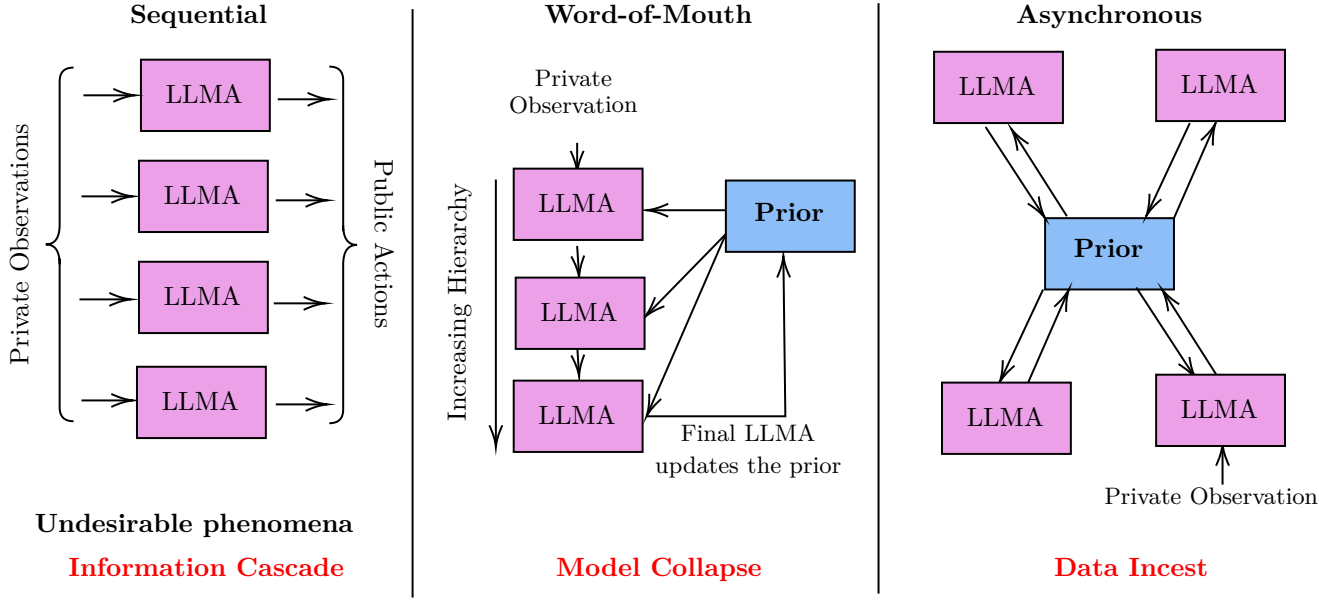

    \centering
    \include{plots/part2}
    \caption{The three topologies considered for a network of interacting \bayesianagents, to interpretably understand and mitigate undesirable phenomena observed in LLMs. We look at sequential Bayesian learning, word-of-mouth learning, and asynchronous Bayesian inference and motivate these topologies using information cascade, model collapse, and potential data incest.  }
    \label{fig:sociallearningtopologies}
\end{figure*}
\section{Bayesian Social Learning in A Sequence of LLM Agents}\label{sec:sociallearning}
LLMs are already trained on synthetically generated data by other models~\cite{shumailov2024ai} and also often use the output of other LLMs to output based on the current context~\cite{llmautomation}. Motivated by studying interacting LLMs, each of which has computational and privacy constraints, this section introduces a second layer of abstraction, wherein we study Bayesian social learning in a sequence of \bayesianagentstext. We first motivate the setting where a sequence of \bayesianagents\ sequentially estimate a state from their private observations and take a public action, which is used to update the public belief. We discuss the Bayesian social learning protocol in a sequence of~\bayesianagents, which aim to detect an underlying state by sequentially analyzing text observations of the text.  The optimal update equation for the public belief is derived. We consider two scenarios, one where no private observations are shared and one where private observations are shared to the next $\numagentsprev$ \bayesianagents. Finally, we show that under both scenarios, an information cascade takes place, and the agents take the same action irrespective of their private observation. To show this, we use the martingale convergence theorem~\cite{krishnamurthypoor}. We illustrate the effect of the number of private observations revealed and the resolution of the probe on the convergence in herding. We also present the mathematical model for incentivized autonomous LLM agents used later in Section~\ref{sec:incentivizedherding}, which is motivated by different entities employing such agents to perform Bayesian state estimation using textual data.

\subsection{Motivation. Interacting LLMs, Finite Context Length and Privacy in LLM agents. }  Even if a single \bayesianagent\ is used in an application, it can be treated as a sequence of different \bayesianagents\ since the context of the previous \bayesianagent\ evaluation might not be available due to privacy of the content and finite-context length~\cite{ijcai2024p890}. We consider two scenarios, one where no private observation is shared between the \bayesianagents\ and the second where each \bayesianagents\ can observe previous $\numagentsprev$ agents. This is motivated by practical constraints from the perspective of privacy, context length, and cost incurred inherent in using \bayesianagents. 

We motivate studying \bayesianagents\ using a Bayesian social learning perspective with the following constraints: 
\subsubsection{Privacy} 
Since the text observations often contain sensitive information, the text observations can be used to train the LLM of the \bayesianagent~\cite{panda2024teach}; hence, to prevent this often, systems involving LLMs often treat the private observation in a one-shot setting where the private observation is not stored. Even the low-dimensional representation of the text observation might contain information that can be used to identify attributes of the person the text observation comes from, and in a social network application, this can lead to unfair decisions by \bayesianagents~\cite{biasfairness}.  Therefore to preserve privacy of users, the \bayesianagents\ we consider either do not share the private observation or only share a limited sequence of private observations. 
\subsubsection{Limited Context Length and Storage Constraints}
Another constraint is the limited context length (length of text that the LLM can process at a time) that is inherent to the LLM used in the \bayesianagent. Note that there are methods which allow for infinite context window or a very large ($1$ million) context window, however these take a lot of time to work which is often not feasible in a real-time Bayesian inference setting. Also, the quality of the responses decreases with increasing context~\cite{kaddour_challenges_2023}.  
There are storage constraints that won't allow storage of an arbitrary amount of private observations, especially if there are methods (like those presented in this paper) that do not require storage of the private observations. 
\subsubsection{Computational Resources and Cost}  Computational resources (involving GPUs) are often limited, especially if the same LLM deployment is used for different applications. Also, the attention mechanism is such that the computational complexity grows quadratically (linear for state space LLMs) in the input length. Therefore, it is often required to limit the size of the context being provided. More importantly, the LLM service providers often bill on a per-token basis. Hence the costs scale linearly with increasing the size of the private observations, however the value of information of including a previous observation is concave. 
\subsubsection{Single \bayesianagent\ can be modeled as a sequence of \bayesianagents}
Finally, we remark that a single \bayesianagent\ can be modeled as a sequence of \bayesianagents, especially given the above three constraints. This is because when a single \bayesianagent\ is used to sequentially perform Bayesian inference, the constraints from above enforce that no more than $\numagentsprev$ observations can be considered at any given time. We consider updating the prior based on the action of the $\numagentsprev+1$-th previous \bayesianagent. This might seem counterintuitive assumption, however in practical applications often the observations are processed in batches. 
\vspace{-3mm}
\subsection{Social learning protocol when no private observation is shared}
A sequence of \bayesianagentstext\ (\bayesianagent) wishes to estimate an underlying state $\statevar \in \statespace$, where $\statespace$ is finite dimension discrete space. 
At time $\timeindex$ agent $\timeindex$ receives a private observation $\llmobservation_\timeindex \in \observationspace^\prime$ from the state $\statevar$, where $\observationspace^\prime$ is a high-dimensional discrete space (text). Agent $\timeindex$ uses a large language model (LLM) as a sensor to obtain a feature vector $\observation \in \observationspace$ where $\observationspace$ is a low-dimensional discrete space. The text observation is sampled according to the probability distribution $\probabilitymeasure(\llmobservation|\statevar)$. For a given text observation $\llmobservation$, the feature vector is sampled according to the probability $\probabilitymeasure(\observation|\llmobservation)$, $\probabilitymeasure$ represents the suitably defined probability measure. Therefore for a given state, the feature vector is sampled with probability $\observationmatrix_{\observation,\statevar}  = \sum_{\llmobservation \in \observationspace^{\prime}} \probabilitymeasure(\llmobservation|\statevar)\probabilitymeasure(\observation|\llmobservation)$, where $\observationmatrix \in \R^{\vert \observationspace \vert \times \vert \statespace \vert}$ denotes the observation matrix. Let $\actionspace$ be a discrete action space and $\cost:\statespace\times\actionspace\to\R^{+}$ be the cost function\footnote{To be consistent with standard Bayesian social learning, we consider cost minimization in Part 2, however, utility is simply the negative of cost for most nonpathological cost functions.}.

In classical Bayesian social learning, the agent $\timeindex$ computes a posterior belief on the state according to Bayes' rule, 
  \begin{align}\label{eq:bayesrule}
\probabilitymeasure(\statevar|\observation_\timeindex) = \frac{\observationmatrix_{\observation_\timeindex,\statevar}\prior_\timeindex(\statevar)}{\sum_{\statevar^\prime \in \statespace} \observationmatrix_{\observation_\timeindex,\statevar^\prime}\prior_\timeindex(\statevar^\prime)},
    \end{align}
    where $\prior_\timeindex(\statevar)$ denotes the public belief belief over the state space $\statespace$. 
    The agent $\timeindex$ takes an action $\action \in \actionspace$ minimizing the expected cost with respect to the posterior\footnote{A tie-breaking rule such as uniform sampling can be used if two actions have the same cost.},  
    \begin{align}\label{eq:action}
        \action_\timeindex = \underset{\action \in \actionspace}{\argmin}{\sum_{\statevar}\cost(\statevar,
        \action)\probabilitymeasure(\statevar|\observation_\timeindex)}.
    \end{align}
We make the following assumption which is standard in classical Bayesian social learning~\cite{chamley_rational_2004},
\begin{enumerate}[start = 1,label={\bfseries (B\arabic*)}]
    \item The observations $\llmobservation_\timeindexalt$ and  $\observation_\timeindexalt$ are private, i.e., are only available to \bayesianagent\ $\timeindexalt$, but the actions are public, i.e., visible to all subsequent \bayesianagent\ ($\timeindex=\timeindexalt+1,\timeindexalt+2,\dots$). 
\end{enumerate}
We relax this assumption in the next subsection by allowing the private observations to be shared with the next $\numagentsprev$ agents, which allows us to model more realistic Bayesian social learning in \bayesianagents. Based on the action $\action_\timeindex$ of the agent $\timeindex$, the public belief belief $\prior_\timeindex$ is updated using the following filtering equation, which follows from the filtering equation of a hidden Markov model derived in Appendix~\ref{app:sociallearningfilter},
    \begin{align}\label{eq:priorupdate}
        \prior_{\timeindex + 1} = \priorupdate(\prior_{\timeindex},\action_\timeindex),
    \end{align}
where $\priorupdate$ is given by the following equation, 
\begin{align}
 \priorupdate(\prior,\action) = \frac{\actionprob(\prior,\action) \prior}{\indicator_{\statedim}\actionprob(\prior,\action)  \prior} ,
\end{align}
where $\actionprob(\prior,\action) = \diag([\probabilitymeasure(\action|\statevar=1,\prior),\dots,\probabilitymeasure(\action|\statevar=\statedim,\prior)])$ is the probability of actions for different states given the prior.  For $\action\in\actionspace$, $\probabilitymeasure(\action|\statevar=\stateidx,\prior)$ is given by, 
\begin{align}\label{eq:probaction}
    \begin{split}
\probabilitymeasure(\action|\statevar=\stateidx,\prior) = \sum_{\observation\in \observationspace} \probabilitymeasure(\action|\observation,\prior)\probabilitymeasure(\observation|\statevar=\stateidx,\prior),\\ \probabilitymeasure(\action|\observation,\prior) 
 = \begin{cases}
     1, & \text{if } \cost^\prime_\action \observationmatrix_\observation  \prior \leq \cost^\prime_{\actionalt} \observationmatrix_\observation  \prior, \actionalt \in \actionspace \\
     0, &\text{otherwise}
 \end{cases},\\
    \end{split}
\end{align}
where $\observationmatrix_\observation = \mathsf{diag}([\probabilitymeasure(\observation|\statevar=1),\dots, \probabilitymeasure(\observation|\statevar=\statedim)])$ and $\cost_\action = [\cost(1,\action),\dots,\cost(\statedim,\action)]^\prime$. 
\begin{algorithm}[h!]
    \begin{algorithmic}[1]
    \State  Agents aim to estimate state $\statevar$
    \For{$\timeindex \in 1,2,\dots$}
    \State Agent $\timeindex$ observes $\llmobservation_\timeindex\sim \probabilitymeasure(\llmobservation_\timeindex|\statevar)$ 
    \State Agent $\timeindex$ uses LLM to obtain $\observation_\timeindex\sim \probabilitymeasure(\observation_\timeindex|\llmobservation_\timeindex)$ 
    \State Agent $\timeindex$ computes posterior using~\eqref{eq:bayesrule} or \eqref{eq:bayesruletwo} depending on availability of previous observations. 
    \State Agent $\timeindex$ takes optimal action according to~\eqref{eq:action}
    \State Agents $\timeindex+\numagentsprev+1,\dots$  update public belief using~\eqref{eq:priorupdate}
    \EndFor{}
    \end{algorithmic}
    \caption{Social Learning Protocol for \bayesianagents}
    \label{alg:sociallearning}
\end{algorithm}
\subsection{Social learning protocol when the last $\numagentsprev$ private observations are shared}
We now consider a modification of the Bayesian social learning described in the previous subsection. We let the agents use the observations from the last $\numagentsprev$ \bayesianagents\ to update their posterior. We weaken assumption (B1) to the following, 
\begin{enumerate}[start = 2,label={\bfseries (B\arabic*)}]
    \item The observations $\llmobservation_\timeindexalt$ and  $\observation_\timeindexalt$ are visible to agent $\timeindexalt$ and the next $\numagentsprev$ agents, i.e., to \bayesianagents\ $\timeindex=\timeindexalt,\timeindexalt+1,\dots,\timeindexalt+\numagentsprev$. 
\end{enumerate}
To ensure the privacy of the text $\llmobservation_\timeindex$, only the feature outputs by the LLMs $\observation_\timeindex$ can also be shared, with the implicit assumption that the likelihood for the different LLMs used by the previous agents is approximately the same. 

Agent $\timeindex$ creates a vectors using the $\numagentsprev+1$ observations $\mathbf{\observation_\timeindex} = [\observation_{\timeindex-\numagentsprev},\dots,\observation_\timeindex]'$. Since each of the observations of $\mathbf{\observation_\timeindex}$ of the state $\statevar$  are sampled independently. The augmented observation space is $\observationspace^\numagentsprev$ The joint likelihood can be computed as $\probabilitymeasure(\mathbf{\observation_\timeindex}|\statevar) = \prod_{\timeindexalt=0}^{\numagentsprev} \observationmatrix_{\observation_{\timeindex-\timeindexalt},\statevar}$. Then the Bayesian update of~\eqref{eq:bayesrule} can be augmented as follows, 
  \begin{align}\label{eq:bayesruletwo}
\probabilitymeasure(\statevar|\mathbf{\observation_\timeindex}) = \frac{\probabilitymeasure(\mathbf{\observation_\timeindex}|\statevar)\prior_\timeindex(\statevar)}{\sum_{\statevar^\prime \in \statespace} \probabilitymeasure(\mathbf{\observation_\timeindex}|\statevar)\prior_\timeindex(\statevar^\prime)}.
    \end{align}

The \bayesianagent\ $\timeindex$ takes the action $\action_\timeindex$ corresponding to the action which maximizes the expected cost using~\eqref{eq:action}. We make the following assumption related to the agents discarding actions of the previous agents if the observation is available,
\begin{enumerate}[start = 3,label={\bfseries (B\arabic*)}]
    \item In lieu of observations $\observation_{\timeindexalt-\numagentsprev},\dots,\observation_{\timeindex-1}$, \bayesianagent\  $\timeindexalt$ disregards observed actions $\action_{\timeindexalt-\numagentsprev},\dots,\action_{\timeindexalt-1}$ of the previous $\numagentsprev$ agents. 
\end{enumerate}
Hence as a consequence of (B3), the action of agent $\timeindex$ is used by agents $\timeindex+\numagentsprev,\timeindex+\numagentsprev+1,\dots$ to update their public belief $\prior_\timeindexalt,\timeindexalt=\timeindex+\numagentsprev,\timeindex+\numagentsprev+1,\dots$ using~\eqref{eq:priorupdate}.

\subsection{Emergence of Herds and Information Cascades}
This section proves that the \bayesianagents\ described in the previous section form an information cascade and herd in their actions when the public belief gets strong. 

We first define an information cascade occurring at time $\timeindex$ for the \bayesianagents\ in the following definition.
\begin{definition}\label{def:infocascade}(\textbf{Information Cascade}): 
    An information cascade is said to occur at time $\herdtime$ if the public belief of all agents after time $\herdtime$ are identical. That is, $\prior_{\timeindex}(\statevar) = \prior_{\herdtime}(\statevar)$ for all states $\forall \ \statevar\in\statespace$ for all time $\ \timeindex\geq\herdtime$. 
\end{definition}
Information cascade implies that the public belief freezes after time $\herdtime$, and since the public belief freezes, the optimal action taken using~\eqref{eq:action} under any the posterior of~\eqref{eq:bayesrule}. Since the information cascade implies the optimal action remains the same, the following definition naturally describes herding at time $\herdtime$ for \bayesianagents\ where the actions remain the same.
\begin{definition}\label{def:herding}(\textbf{Herding}):
    Herding in the \bayesianagents\ agents takes place at time $\herdtime$ if the action of all agents after $\herdtime$ are identical, i.e. $\action_{\timeindex} = \action_{\herdtime}$ for all time $\timeindex\geq\herdtime$. 
\end{definition}
\begin{figure}
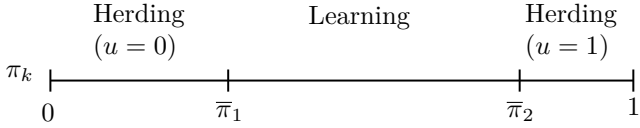

    \centering
    \include{plots/herdingregion}
    \caption{Herding and Learning Regions for state space with $2$ states. \change{The scalar $\prior_\timeindex \in [0,1]$ denotes the prior of state $1$.} If the prior is in the region  $\prior \in [\bar{\prior}_1,\bar{\prior}_2]$, learning happens; otherwise, the \bayesianagents\ form an information cascade and herd in their actions. }
    \label{fig:herdingregionstwo}
\end{figure}
It is straightforward to show that an information cascade (Def.~\ref{def:infocascade}) occurring at time $\timeindex$ implies that herding also takes place at time $\timeindex$ (Def.~\ref{def:herding}). We now state the main result on herding in \bayesianagents, which shows that the protocol of Algorithm~\ref{alg:sociallearning} leads to the agents herding in finite time~\cite{krishnamurthy_partially_2016}. 
\begin{theorem}\label{th:herding}(\textbf{Herding in Bayesian social learning of \bayesianagents})
    The social learning protocol of the \bayesianagents\ described in Algorithm~\ref{alg:sociallearning}, under either assumption (B1) or assumptions (B2,B3) leads to an information cascade (Def.~\ref{def:infocascade}) and agents herd (Def.~\ref{def:herding}) in finite time $\herdtime<\infty$ with probability 1.
\end{theorem}
\begin{proof}
    Proof in Appendix.
\end{proof}
Theorem~\ref{th:herding} shows that herding happens in finite time, and therefore, the agents take the same action regardless of their private observation. Discarding the private observation, which provides valuable information about the current state, makes their state estimation incorrect and inefficient.    
\begin{remark}
    From a purely statistical perspective, Theorem~\ref{th:herding} can be seen as the following: when the priors are updated without seeing the observation but rather using the correlated actions, the posterior becomes inconsistent and need not necessarily converge to the true value asymptotically. 
\end{remark}
\subsection{Effect of the number of private observations revealed and resolution of the probe on  herding convergence}\label{sec:herdingcvg}
We next discuss the effect of the resolution of the LLM probe and the number of private observations in changing the threshold at which the convergence takes place. For the purpose of this section, assume that the state space is such that $|\statespace| = 2$ and the action space is such that $\vert\actionspace\vert = 2$ and consider the case of the Bayesian agents performing inference.

We first mathematically describe the different regions with respect to the public belief. Then, we derive the relation between the threshold of the public belief and the observation matrix for the different observations. Such a derivation can be used to see the effect of more accurate observations either by considering a higher resolution probe or by considering more number of observations. 

We can derive the following for the different regions with respect to the public belief $\prior$ where herding happens and where it does not, i.e., where learning happens. 
\begin{align*}
    \probabilityregion(\prior) = \begin{cases}
        \text{Region 1  (Herding} \ \action = 1),&\\ \ \ \ \ \ \cap_{\observation \in \observationspace} \left\{\sum_{\statevar\in\statespace}(\cost(\statevar,1) - \cost(\statevar,2))\probabilitymeasure(\statevar|\observation) \leq 0 \right\} \\
        \text{Region 2 (Learning} \ \action = \observation), &\\ \ \ \ \ \
        \cup_{\observation \in \observationspace} \left\{\sum_{\statevar\in\statespace}(\cost(\statevar,1) - \cost(\statevar,2))\probabilitymeasure(\statevar|\observation) > 0 \right\}\\ \ \ \ \bigcap \cup_{\observation \in \observationspace} \left\{\sum_{\statevar\in\statespace}(\cost(\statevar,2) - \cost(\statevar,1))\probabilitymeasure(\statevar|\observation) > 0 \right\} \\
        \text{Region 3 (Herding} \ \action = 2), &\\ \ \ \ \ \ \cap_{\observation \in \observationspace} \left\{\sum_{\statevar\in\statespace}(\cost(\statevar,2) - \cost(\statevar,1))\probabilitymeasure(\statevar|\observation) \leq 0 \right\}  \\
    \end{cases}
    \end{align*}
\begin{remark}
    The above equation can be derived by equations~\eqref{eq:priorupdate} and~\eqref{eq:action} by setting the action to be constant for the herding regions. The regions are also illustrated in Figure~\ref{fig:herdingregionstwo} for $2$ states and are numerically shown for $3$ states of a real-world dataset in Figure~\ref{fig:herdingregions}.
\end{remark}
Note that learning (region 2) only happens when the action taken by the \bayesianagent\ corresponds to the observation. 

Let $\prior = [p,1-p]^\prime$ and $\cost(\statevar,\action) = |\action - \statevar|$. We first derive the expression of the region of herding for a specific observation $\observation$, 
\begin{align*}
    &- \frac{p\observationmatrix_{\observation,1}}{p\observationmatrix_{\observation,1}+(1-p)\observationmatrix_{\observation,2}} + \frac{(1-p)\observationmatrix_{\observation,2}}{p\observationmatrix_{\observation,1}+(1-p)\observationmatrix_{\observation,2}} \leq 0 \\ &\implies
     \frac{(1-p)\observationmatrix_{\observation,2}-p\observationmatrix_{\observation,1}}{p\observationmatrix_{\observation,1}+(1-p)\observationmatrix_{\observation,2}} \leq 0\implies 
     p \geq \frac{\observationmatrix_{\observation,2}}{\observationmatrix_{\observation,2} + \observationmatrix_{\observation,1}},
\end{align*}
and then take the intersection to obtain, 
\begin{align*}
    p \geq \max_{\observation\in\observationspace}\frac{\observationmatrix_{\observation,2}}{\observationmatrix_{\observation,2} + \observationmatrix_{\observation,1}}.
\end{align*}
We can prove a similar argument for state $2$ and obtain the following result,
\begin{align*}
     p \leq \min_{\observation\in\observationspace}\frac{\observationmatrix_{\observation,1}}{\observationmatrix_{\observation,2} + \observationmatrix_{\observation,1}}.
\end{align*}

This discussion shows us that even for a simplistic setup, improving the probe accuracy of the LLM helps reduce the herding threshold. However, more accurate LLMs are often larger and have a higher unit cost, clearly highlighting the tradeoff between herding and the cost incurred.  A similar result can be derived when the number of shared observations $\numagentsprev$ increases, as the observation space grows with $\numagentsprev$. 


\begin{example}\label{eg:hsp}
We show empirically how interacting LLMs can be used to identify a hate speech peddler (HSP) from a large corpus of data (tweets, blogs, pictures, essays, opinions). Since LLMs charge per token and have latency constraints, we consider multiple LLMs that collaborate to process
the large corpus of information. We now experimentally illustrate how information cascades emerge when LLM agents inter-
act to identify an HSP. We used the Mixtral-8x7B-v0.1 LLM. The state $\statevar \in {1 = \text{(not HSP)}, 2 = \text{(HSP)}}$ is the ground truth. The observations $\{\llmobservation_\timeindex\}$ are the content generated by users. When LLM agent $\timeindex$ receives $\llmobservation_\timeindex$, it parses the content to generate a low dimensional observation $\observation_\timeindex\in\{1,2\}$,
designed to detect hate speech. This observation $\observation_\timeindex$ is processed by the Bayesian engine, which computes the posterior of the state. Based on the posterior, the LLM agent selects action $\action \in \{1 = \text{not HSP}, 2 = \text{HSP}\}$ by minimizing the Type-1 error cost $\probabilitymeasure( \statevar = \text{HSP})|\action_1, . . . , \action_{\timeindex-1}, \observation_\timeindex)$.
This action $\action_\timeindex$ is broadcast to subsequent LLM agents that parse the remaining content, but the private observation $\observation_\timeindex$ is kept confidential to preserve privacy. We empirically computed $B_{11} = B_{22} = 0.8$ from training data. Figure~\ref{fig:herdingexample} displays two sample paths of actions generated by the LLM agents for different initial priors. Both sample paths emerge into information cascades.
\end{example}
\begin{figure}
    \centering
\resizebox{\columnwidth}{!}{
    \input{plots/plot.pgf}}
    
    \caption{Emergence of information cascade in LLM Agents for Bayesian Inference of Hate Speech Peddler in Example~\ref{eg:hsp}. Even though the underlying state is $2$ (not an HSP),  the cascade is in the wrong direction when $\prior_0(\statevar=1) = 0.34$.  }
    \label{fig:herdingexample}
\end{figure}
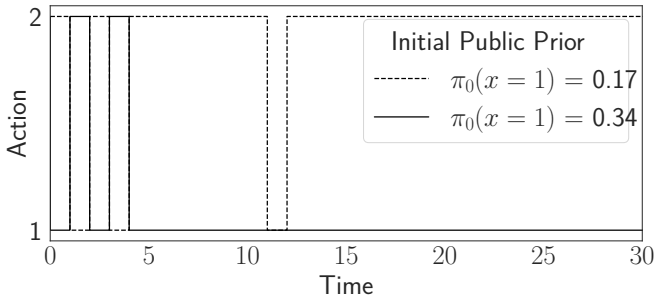
\subsection{Summary}
Given the privacy, computational, and cost constraints, multiple \bayesianagents\ need to interact with each other to perform sequential Bayesian inference on online platforms. 
This section studied Bayesian social learning in a sequence of \bayesianagents\ to analyze multiple \bayesianagents\ performing sequential state estimation on online platforms.  We discuss the relaxation of the standard Bayesian social learning protocol, in which the agents are allowed to share their private observations. We illustrate the effect on the herding threshold of the public belief when \bayesianagents\ are allowed to share private observations and use more accurate LLMs as sensors.

\section{Asymmetric Information Structures of Large Language Model Agents}\label{sec:wordofmouth}
Motivated by the observed model collapse while training LLMs in a sequential manner and data incest in asynchronous Bayesian sensors, this section studies word-of-mouth and asynchronous social learning. 
Word-of-mouth social learning is a hierarchical social learning paradigm characterized by asymmetric information flow, where lower-level agents process and communicate observations, and top-down influence, where top-level agents dictate public belief.
We provide protocols for both kinds of social learning and a corollary that shows that information cascades happen in word-of-mouth learning as well. Techniques to perform stochastic control to prevent model collapse and data incest are not in the scope of this paper and can be considered in future work. 
\subsection{Motivation. Model Collapse in LLMs. }
Several recent studies~\cite{shumailov2024ai} have shown that when LLMs are repeatedly trained on data generated by other LLMs, a phenomenon known as "model collapse" can occur.
 In model collapse, the output probability distribution of the model collapses to a degenerate distribution as the model is trained iteratively on data generated by the previously trained model~\cite{kazdan2024collapsethriveperilspromises}. Similar results have been shown on model distillation, where a smaller LLM is trained using data generated from a larger LLM~\cite{cvprmodeldistillation}.
We use the Bayesian social learning model to show how an information cascade in LLMAs is similar to the model collapse observed while training LLMs. 

Model collapse observed while training LLMs can be seen as a special case of sequential Bayesian social learning. In the case of training LLMs using data generated from the previous LLM, it can be considered as estimating the underlying state $\statevar$, which is the true probability distribution of the data. However at each time $\timeindex$, instead of receiving a private observation from the state $\statevar$ \bayesianagent\ $\timeindex$ receives observation $\mathbf{\observation}_\timeindex$ from the previous  \bayesianagent\ $\timeindex-1$. The \bayesianagent\ $\timeindex$ then minimizes a cost function, which is an entropic regularizer (maximum likelihood, KL-divergence, or cross-entropy loss) to obtain an estimate of the state $\statevar$ (represented by $\hat{\statevar}$) from the observations $\mathbf{\observation}_\timeindex$. The \bayesianagent, then uses the estimate $\hat{\statevar}$ to sample observations $\mathbf{\observation}_{\timeindex+1}$ which it provides to agent at $\timeindex+1$. 

If the underlying true probability distribution was a Gaussian, then it can be shown that such a protocol leads to a slowdown of learning~\cite{krishnamurthy2024slowconvergenceinteractingkalman}. However, in the case of discrete distributions in which the LLMs learn such a protocol, it leads to collapsing on one of the support points~\cite{shumailov2024ai}; therefore, model collapse in the training of LLMs can be studied using the framework of Bayesian social learning. 

Asynchronous social learning in \bayesianagents\ is motivated by real-time settings like online platforms where there is a stream of data of the order of a hundred thousand every second  \cite{liu2024a,10.1145/2623330.2623637,national2013frontiers}. Since LLM functionality within an LLMA often requires several milliseconds to a few seconds, especially if these are third-party services, sequential Bayesian learning is often not possible. This is true particularly when the \bayesianagents\ are used for tasks that are more sophisticated than just Bayesian inference\cite{li2024personalllmagentsinsights}. 
\subsection{Word of Mouth Bayesian Social Learning in LLM Agents}
We now describe the word-of-mouth social learning protocols in $\numagentsprev$ \bayesianagents. The protocol is summarized as a pseudo-code in Algorithm~\ref{alg:wordofmouth}. The protocol can be considered to run on two timescales. On the slower time scale, the first \bayesianagent\ receives a new text observation $\llmobservation_{\timeindex\numagentsprev} \sim \probabilitymeasure(\llmobservation|\statevar)$ of the state $\statevar$, \change{where $\probabilitymeasure(\llmobservation|\statevar)$ is the observation likelihood of text $\llmobservation$ given the state $\statevar$. The LLMA does not explicitly have knowledge of $\probabilitymeasure(\llmobservation|\statevar)$ but only receives a text observation $\llmobservation$ from it.} On the \change{faster} time scale, the \bayesianagents\ communicate with each other in a sequential fashion by generating text observations corresponding to the low-dimensional features from the received text observation. That is, agent $\timeindexalt$ first receives a text observation $\llmobservation_{\timeindexalt}$ and uses an LLM to obtain a low dimensional observation $\observation_\timeindexalt$. The agent $\timeindexalt$ then takes an action to update the public belief using \eqref{eq:action}. Then, the \bayesianagent\ uses the LLM to generate a \textit{synthetic} text observation $\llmobservation_{\timeindexalt+1}$ used by the next agent. 
\begin{algorithm}
\begin{algorithmic}[1]
    \For{$\timeindex \in 0,1,2,\dots$} 
    \State \bayesianagent-($\timeindex\numagentsprev$) receives a observation $\llmobservation_{\timeindex\numagentsprev}\sim \probabilitymeasure(\llmobservation|\statevar)$.
    \For{$\timeindexalt \in \timeindex\numagentsprev,\timeindex\numagentsprev+1,\timeindex\numagentsprev+2,\dots $\change{$(\timeindex+1)\numagentsprev-1$}}
        \State \bayesianagent-$\timeindexalt$ obtains low-dimensional features $\observation_{\timeindexalt}\sim \probabilitymeasure(\observation|\llmobservation_{\timeindexalt})$ using LLM
        \State \bayesianagent\ takes optimal action $\action_\timeindexalt$ using \eqref{eq:action}.
        \State Public prior is updated, $\prior_{\timeindex+1} = \priorupdate(\prior_\timeindex,\action_{\timeindexalt})$.
        \State \bayesianagent-$\timeindexalt$ generates synthetic text observation $\llmobservation_{\timeindexalt +1}\sim \probabilitymeasure(\llmobservation|\observation_{\timeindexalt},\action_\timeindexalt)$ using an LLM
    \EndFor{}
    \EndFor{}
\end{algorithmic}
\caption{Word Of Mouth Protocol for \bayesianagents}
\label{alg:wordofmouth}
\end{algorithm}

Hence, the main difference is that each agent does not receive a new private observation. Note that there are different versions of the word-of-mouth, for example, one where the prior update is done at the end of the inner loop once. All of them are interesting to study. However, we focus on the one presented because we can derive the following result as a corollary of Th.~\ref{th:herding} showing herding of \bayesianagents\ in Algorithm~\ref{alg:wordofmouth}. 
\begin{corollary}\label{corr:herdingwom}
 The word-of-mouth social learning protocol of the \bayesianagents\ described in Algorithm~\ref{alg:wordofmouth} leads to an information cascade (Def.~\ref{def:infocascade}) and therefore herd (Def.~\ref{def:herding}) in their actions with probability 1.    
\end{corollary}
The above result can be proved using Theorem~\ref{th:herding}, where each agent has a different observation likelihood based on the previous agent, each of which is a concatenation of two observation likelihoods: the LLM as a low-dimensional sensor map and the LLM sampling a text observation from the text observation. Corollary~\ref{corr:herdingwom} shows that even in the modified protocol of Algorithm~\ref{alg:wordofmouth}, where synthetic data is used to aid the decision-making of \bayesianagents, cascades are inevitable. 
\begin{algorithm}
\caption{Naive Asynchronous Data Fusion in \bayesianagents}
\begin{algorithmic}[1]
\State Initialize prior \( \prior_0 \)
\While {\bayesianagent \ $\timeindex$  receives new observation \( \llmobservation_\timeindex \) and a broadcasted prior $\prior_{\timeindex-1}$}
    \State \bayesianagent \ $\timeindex$ uses LLM to obtain $\observation_{\timeindex}\sim \probabilitymeasure(\observation|\llmobservation_{\timeindex})$
    \State Broadcast the posterior $\prior_{\timeindex} = \priorupdate(\prior_{\timeindex-1},\observation_{\timeindex})$ 
    \State $\timeindex = \timeindex+1$
\EndWhile
\State \textbf{Return} Estimate using $\prior_{\timeindex}$ and Eq.~\eqref{eq:action} 
\end{algorithmic}
\label{alg:async}
\end{algorithm}

\textbf{Asynchronous Social Learning in LLM Agents: }
We finally consider the asynchronous social learning setting in \bayesianagents. Here, the main difference between the previous two topologies is that the agents do not necessarily act in a predefined sequential manner and do not necessarily coordinate. 

The protocol is summarized as a pseudo-code in Algorithm~\ref{alg:async}. The public belief $\prior_\timeindex$ is updated asynchronously when a \bayesianagent\ receives a new private observation $\llmobservation_\timeindex$, which it parses using the LLM to obtain a low-dimensional observation $\observation_\timeindex$. This observation is used to compute the posterior at time $\prior_\timeindex$ using a previous prior $\prior_{\timeindex-1}$. The posterior is then broadcasted. 

It is immediate to see how this protocol can lead to data incest. For example, consider a case where the \bayesianagent\ $\timeindex$ uses the prior $\prior_{\timeindex-1}$ and updates the prior. The prior $\prior_{\timeindex-1}$ was, in turn, updated based on an \bayesianagent\ $\timeindexalt$, where $\timeindexalt\leq \timeindex-1$. However, \bayesianagent\ $\timeindexalt$  at time $\timeindex+1$ sees a new prior and uses the previous one to compute its estimate without accounting for the fact that its previous observation was already used to compute this updated prior. This leads to double counting the observation and, therefore, to data incest~\cite{krishnamurthy_partially_2016}.

\section{Summary}
We discuss how the phenomenon of model collapse observed in LLMs is a form of Bayesian social learning. We present models of asymmetric information structures of \bayesianagents\ including word-of-mouth social learning and asynchronous social learning protocols. We state a corollary showing how information cascade occurs with probability 1 even in word-of-mouth social learning and motivate the careful design of asynchronous data fusion in \bayesianagents.

\vspace{0.5cm}
{\color{accessblue}{\hrule}} \vspace{0.5cm}

{\hspace{-4mm}\Large\color{accessblue} \textbf{Part III: Stochastic Control for Bayesian Social Learning in LLM Agents}}
\vspace{0.5cm}

In Part II, we studied interpretable Bayesian social learning in \bayesianagents\ and proved that information cascade is inevitable; we would like to at least delay herding. This is especially important when the \bayesianagents\ are prone to cascading to the wrong prior, which is critical in different practical applications, including hate-speech peddler identification. 

Part III, therefore, looks at stochastic control for \bayesianagents. 
For both regimes of LLMAs, when they are collaborative and autonomous, the paper formulates optimal stopping time problems to control herding by balancing the tradeoff between privacy and estimation. Structural assumptions on the optimal policy of the stopping time problems are proved by making structural assumptions on the cost and observation probabilities. The proposed solutions are extensions to our work in quickest change detection and quickest time herding~\cite{quickestchangedetectionitit,quickestchangedetection,jain2024identifying}. A policy gradient algorithm is proposed to estimate the optimal policy without the knowledge of the transition probabilities. 

\section{Optimal Stopping Time Control in Centrally Controlled LLMAs}\label{sec:quickesttimeherding}

When \bayesianagents\ are deployed in real-life settings, they often exhibit bias in their actions, especially when there are multiple such agents~\cite{borah2024implicitbiasdetectionmitigation}. The previous section showed that such bias could be explained by the herding behavior of \bayesianagents. This section formulates an optimal stopping time problem for quickest time herding in a sequence of~\bayesianagents, which ensures that the herding is optimally delayed by letting the~\bayesianagents\ share private observations. The delay in herding helps improve the state estimate. We discuss a stochastic control approach to solve the optimal stopping time problem and state assumptions that ensure that the optimal policy for the stopping time problem has a switching threshold curve with respect to the public belief. This structural result is exploited in Section~\ref{sec:stochasticapproximation} to efficiently approximate the optimal policy.  Finally, we discuss extensions of the problem framework to the optimal switching between models of different sizes in sequential state estimation tasks is discussed. The schematic of the system model considered is illustrated in Figure~\ref{fig:centralcontrolschematic}. 

\subsection{Motivation. Controlling for bias in decisions of interacting LLMAs.} 

Since LLMAs exhibit herding behavior, there is a need to control them to ensure that the estimation is more accurate. In order to do this, we propose optimally switching between sharing the private observation and herding. Our setup is further motivated by the fact that the publicly available LLMs are available in different sizes, and with increasing size, the accuracy of the LLM improves, but so does the unit cost of using the LLM. As we explained in Section~\ref{sec:herdingcvg}, the herding can be delayed by using a more accurate LLM. Therefore, the optimal switching can be considered an optimally stopping problem for using a larger LLM and switching to a smaller one so that the cost is minimized and the herding is optimal.  However, solving for the optimal policy of the stopping time problem is a computationally intensive task. Since solving the optimal policy for the optimal switching can be computationally challenging, we look at structural assumptions on the system parameters such that the optimal policy has a threshold structure that can be more efficiently searched for. Note that the POMDP considered in this and the next section is non-standard and, therefore, allows for structural results that can be exploited for efficient policy estimation techniques.

\begin{figure}[h]
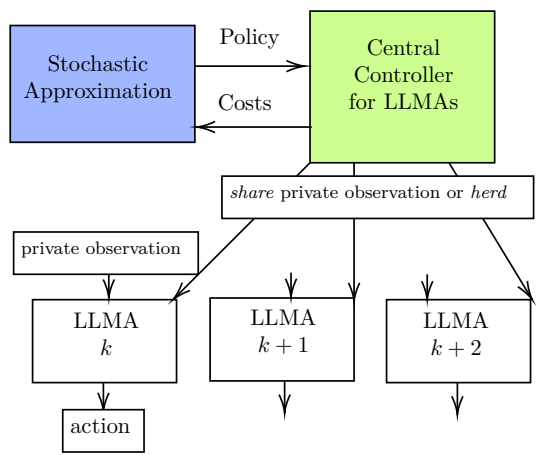

    \centering
    \include{plots/centrallcontrol}
    \caption{Schematic of a stochastic control approach for optimally delaying herding in a sequence of centrally controlled \bayesianagentstext. Such a setting is motivated by a central entity hosting the \bayesianagents.}
    \label{fig:centralcontrolschematic}
    \vspace{-6mm}
\end{figure}

\subsection{Social Welfare Objective For Optimal Stopping}
This subsection formulates an optimal stopping time problem to delay herding by making the agents optimally switch between two modes, acting benevolently by sharing their private observations \textit{or} herding by performing the action from~\eqref{eq:action}. Let $\decision_\timeindex \in \{0 = \texttt{share private observation}, 1 = \texttt{herd}\}$ denote the decision at $\timeindex$ for the chosen mode. Let $\policy: \probspace(\statespace) \to \{0 = \texttt{share private observation}, 1 = \texttt{herd}\}$ denotes the stationary policy which maps the public belief to the decision rule. $\policy$ is a sufficient statistic for optimally delaying herding since the information cascade depends on only the public belief~\cite{krishnamurthy_partially_2016}.For this section, we assume that $\statespace = \actionspace = \observationspace$, that is, in the sequential detection task, the observation $\observation$ is the noisy observation of the underlying state $\statevar$ and $\action$ is the maximum a-priori estimate of the state. 

Further, we formulate the stopping time problem such that one of the states is of special interest (denoted by $\indicatorstate_1$), and a decision to herd is taken once this state is estimated. To formulate the objective for the optimal stopping time problem, we consider the following natural filtration, $\filtration_\timeindex$ of the actions and decisions till time $\timeindex$, $\filtration_\timeindex = \sigma(\{\action_1,\dots,\action_{\timeindex-1},\decision_1,\dots,\decision_\timeindex\})$. 

The optimal stopping time problem is to decide when to stop sharing private observations ($\decision_\timeindex=0$) and announce state $\indicatorstate_1$ ($\decision_\timeindex=1$). Let $\stoppingtime$ denote the stopping time with respect to the filtrations $\filtration_\timeindex, \timeindex=1,\dots$. 

We now state the social welfare objective that each \bayesianagent\ optimizes to solve the optimal stopping time problem and achieve quickest time herding, 
\begin{align}
\label{eq:socialwelfarecost}
    \avgcost(\policy) = \expectation_{\policy}&\left\{\sum_{\timeindex = 1}^{\stoppingtime-1} \discountfactor^{\timeindex-1} \expectation\left\{\cost(\statevar,\action_\timeindex)|\filtration_{\timeindex-1}\right\} \right.\\\nonumber&\left.+ \sum_{\timeindex=1}^{\stoppingtime-1} \discountfactor^{\timeindex-1} \delaycost  \expectation_{\policy} \left\{\indicator(\statevar=\indicatorstate_1)|\filtration_{\timeindex-1}
    \right\} \right.\\\nonumber&\left. +
    \ \discountfactor^{\stoppingtime-1} \errorcost\expectation_\policy\left\{ \indicator(\statevar\neq \indicatorstate_1 )\right\}  \right.\\&\nonumber\left.+\ \frac{\discountfactor^{\stoppingtime-1}}{1-\discountfactor} \min_{\action \in \actionspace} \expectation\left\{ \cost(\statevar,\action)|\filtration_{\stoppingtime-1}\right\}
    \right\}.
\end{align}
\change{Here $\discountfactor\in(0,1)$ is the economic discount factor and can be set to a lower value if the central controller only wants the first few actions to determine the policy.}

\textit{Justification for social welfare cost: }In~\eqref{eq:socialwelfarecost}, the first part corresponds to the discounted cost incurred if the first $\stoppingtime$ agents perform sensing, and the last term corresponds to the cost incurred for the agents after the stopping time $\stoppingtime$ (which herd and take the same opportunistic action). The second term is the delay cost with delay parameter $\delaycost$ in announcing if the underlying state is $\indicatorstate_1$ and the third term is the error cost for misclassification with $\errorcost$ as the parameter. The social welfare cost incentivizes revealing private observations to delay herding so that enough private observations are available for estimation of the state. The reason we consider estimation with respect to a single state is that many practical applications focus on the identification of a critical state ( hateful user, bad product, etc.). The same social welfare cost function of~\eqref{eq:socialwelfarecost} is optimized to obtain a policy with respect to the belief which decides when to stop herding.  
\begin{remark}

\change{In a centrally controlled setting, the central controller is the human operator who manages the platform the LLMAs are deployed on. The action space of the agents consists of either performing the optimal action or revealing their private observation. The decision to `stop' corresponds to taking the optimal action, as this would result in herding. Consequently, the optimal stopping formulation inherently leads to agents sharing their private observations until the decision to stop is made.}
\end{remark}

We consider the decision rule such that depending on the prior, the learners either share the private observation $\observation_\timeindex$ or they output the action which minimizes the expected cost of~\eqref{eq:action} under the public belief, i.e., the action $\action_\timeindex$ is of the form,
\begin{align}\label{eq:constraineddecisionrule}
\action_\timeindex(\prior_{\timeindex-1},\observation_\timeindex,\policy) = \begin{cases}
        \observation_\timeindex, & \text{if } \policy(\prior_{\timeindex-1}) = 2\\
        \argmin_\action \cost_\action \prior_{\timeindex-1}& \text{if } \policy(\prior_{\timeindex-1}) = 1 
    \end{cases}.
\end{align}
The motivation for this decision rule is two-fold: a) the herding can be delayed by ensuring enough private observations are shared, and b) the agents can reduce cost by outputting the opportunistic action once the public belief is strong enough. 

Next, we show that under certain assumptions, the optimal policy $\policy^{*}$ minimizing the expected cost~\eqref{eq:socialwelfarecost},
\begin{align}\label{eq:optpolicy}
    \policy^{*} = \inf_{\policy} \avgcost(\policy),
\end{align}
indeed has a nice threshold structure that optimally delays herding and improves the detection of state $\indicatorstate_1$.
\subsection{Structural Results on Optimal Policy}
We make the following assumptions on the cost function $\cost$ and the observation matrix. 
\begin{enumerate}[start=1,label={(\bfseries S\arabic*):}]
    \item $\cost(\indicatorstate_\stateidx,\action) - \cost(\indicatorstate_{\stateidx+1},\action) \geq 0$ $\forall \stateidx = 1,\dots,\statedim-1$ \ $\forall \action$.
    \item $\cost(\indicatorstate_\statedim,\action) - \cost(\indicatorstate_\stateidx,\action) \geq (1-\discountfactor)\sum_{\observation}(\cost(\indicatorstate_\statedim,\action)\observationmatrix_{\statedim,\observation} - \cost(\indicatorstate_\stateidx,\action)\observationmatrix_{\stateidx,\observation})$ $\forall \stateidx = 1,\dots,\statedim$.
    \item $(1-\discountfactor)\sum_{\observation}(\cost(\indicatorstate_1\action)\observationmatrix_{1,\observation} - \cost(\indicatorstate_\stateidx,\action)\observationmatrix_{\stateidx,\observation}) \geq \cost(\indicatorstate_1,\action) - \cost(\indicatorstate_\stateidx,\action)$ $\forall \stateidx = 1,\dots,\statedim$.
    \item $\observationmatrix$ is totally positive of order 2 (TP2)\footnote{A matrix $A$ is TP2 if all second order minors of the matrix $A$ are positive.}.
\end{enumerate}
\textit{Discussion of Assumptions: } (S1) ensures that the states can be ordered such that taking action in some states is costlier. (S2) and (S3) ensure that the cost function is submodular in the belief. This makes the cost differential between continuing and stopping (herding) the highest for state $\indicatorstate_1$ and gives incentive to the agents to herd when approaching the state $\indicatorstate_1$. 

The TP-2 condition on the observation matrix in (S4) ensures consistency of the observations~\cite{bhattcontrolled2020}, i.e., there is some order to the observations. For example, for a 2x2 observation matrix, TP-2 implies that $B_{11}>B_{21}$ and $B_{22}>B_{12}$. If one of the observations is more likely for a particular state, then the other observation must be more likely for the other state. In our experimental results we consider a standard type-1 error based misclassification cost, however more general costs can be considered which satisfy submodularity. 

We now state the main structural result on the threshold structure of the optimal policy~\eqref{eq:optpolicy}.
\begin{theorem}\label{th:threshold}
    Consider the sequential decision problem of \bayesianagents\ for detecting state $\indicatorstate_1$ with the social welfare cost comprising \eqref{eq:socialwelfarecost} and the constrained decision rule of \eqref{eq:constraineddecisionrule}. Then, under Assumption (S1-S4), the constrained decision rule of \eqref{eq:constraineddecisionrule} is a threshold in the belief space with respect to a threshold switching curve that partitions the belief space $\probspace(\statedim)$. The optimal policy can be given by, 
    \begin{align}\label{eq:optthresholdpolicy}
        \policy^{*}(\prior) = \begin{cases}
           2 \ \text{(continue) if} \ \prior \in \region_2 \\
           1 \ \text{(stop) if} \ \prior \in \region_1 \\
        \end{cases},
    \end{align}
    where $\region_1$ and $\region_2$ are individual connected regions of $\probspace(\statedim)$.
\end{theorem}
The above theorem proves that under certain conditions on the cost function and observation matrix (S1-4), the optimal policy for solving the discounted social welfare cost optimal stopping time problem has a switching threshold curve. Such a switching threshold can be approximated by \textit{{set of lines}} which can be searched for efficiently using a policy gradient based approach~\cite{krishnamurthy_partially_2016}. \change{ $\region_1$ and $\region_2$ are not known to the LLMA and to the central controller, and the structural result only shows the existence of such disjoint regions. These regions are unknown, and we show in Section X that for $|\statespace|=2$, one can efficiently search for these regions using a policy gradient algorithm. Although standard reinforcement learning algorithms can be used, we study the extreme case to derive an efficient algorithm. This is computationally tractable, unlike optimizing for a non-structured policy where a finite approximation is used for a general infinite-dimensional policy.} This is computationally tractable, in contrast to optimizing for a non-structured policy since it's a finite approximation of a general infinite-dimensional policy.

\begin{remark}
\change{
   The LLM agents share their current private observations, and the stopping-time policy determines when the sequence of agents ceases sharing private information. This approach can delay herding while ensuring that not all private information is disclosed. The privacy loss can be quantified using an information-theoretic metric that captures the exact loss of private information. Investigating this in practical scenarios would be an intriguing direction for future research.}
\end{remark}

\subsection{Optimal Switching Between Different LLMs}
The formulation that we proposed in this section can also be used by the \bayesianagents\ to switch between LLMs of different sizes. LLMs with a higher number of parameters can process more tokens and give more accurate responses; however, they are more expensive and use more computational resources.  Therefore we can use an action space $\actionspace = \{ 1 = \llm_{\text{small}}, \llm_{\text{large}} \}$, where $\llm_{\text{large}}$ is an LLM with significantly ($10$\change{$\times$}) more parameters than $\llm_{\text{small}}$. 

Since the LLM with more number of parameters ($\llm_{\text{large}}$) can give a more accurate observation, by the arguments of Section~\ref{sec:herdingcvg}, we can use it to delay herding. However, the cost of using $\llm_{\text{large}}$ is higher, and accounts for the privacy cost $\cost$ considered in this section. And then, once we have sufficiently many good responses, the \bayesianagents\ can switch to a smaller LLM, which will lead to a quicker information cascade but can still provide low-dimensional readings useful for analytical purposes at a lower cost. Such a problem is often referred to as a quickest change detection problem~\cite{krishnamurthy_partially_2016}.

\subsection{Summary}
When a central entity is designing and deploying \bayesianagents\ for sequential Bayesian inference on online platforms, the entity needs to ensure such \bayesianagents\ do not have a bias due to herding. 
In this section, we solved the problem of \bayesianagents\ optimally herding in Bayesian social learning framework to announce a particular state, such that their opportunistic cost and cost of obtaining and sharing the observation is balanced. The Bayesian agents considered in this section were cooperative and shared the same socialistically optimal policy. We also briefly discussed how such a scheme can be used to optimally switch between different sizes of LLMs to achieve an optimal tradeoff between cost and accuracy of estimation. In the next section we consider the problem in a different setting, where a central controller can incentivize autonomous \bayesianagents, and needs to control the incentives for improved state estimation. 
\section{Optimal Stopping Time Control for Autonomous LLMAs}\label{sec:incentivizedherding}

The \bayesianagents\ used for Bayesian inference can often be from different third-party services, each of which requires an incentive to perform the task. Motivated by controlling bias in such \bayesianagents, this section considers the problem of the central learner optimally incentivizing a sequence of autonomous \bayesianagents\ to delay their herding and obtain more accurate estimates. We first formulate the optimization problem of the central controller as a discounted cost POMDP and then show that under structural assumptions on the cost and observation matrix, the optimal policy has a threshold structure with respect to the public belief. This structure is exploited in Algorithm~\ref{alg:stochasticapprox} to approximate the optimal incentivization policy of the central controller.  The schematic of the setup considered in this section is illustrated in Figure~\ref{fig:autonomousschematic}.
\begin{figure}[h]
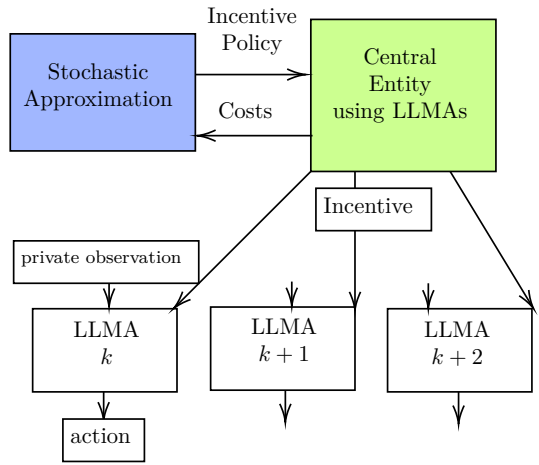

    \centering
    {\include{plots/autonomous}}
    \caption{Schematic of stochastic control approach for optimally delaying herding in a sequence of incentivized autonomous LLMAs. This is motivated by a central entity deploying third-party \bayesianagents\ for Bayesian inference. }
    \label{fig:autonomousschematic}
\end{figure}
\subsection{Motivation. Optimal Incentivization of Third Party LLMAs.}
There are already several third-party services that offer LLM agents as a service~\cite{agentforce}, and such agents can be used for the task of performing Bayesian inference on an online platform. Each of these agents has a cost of processing a query associated with it, and each agent can be asked to give more accurate responses (by performing more computations or by using a larger model). Therefore, we can incentivize the \bayesianagents\ to share more accurate private information (low-dimensional outputs). However as shown in Section~\ref{sec:sociallearning}, \bayesianagents\ are still prone to herding, therefore we propose an optimal stopping time formulation for optimally herding and at the same time minimize the cost incurred by the central entity.

\subsection{Optimization problem of the central controller with incentivized \bayesianagents}

We now consider a case where the \bayesianagent\ are incentivized by a central controller, and the cost function is now dependent on the incentive $\incentive$ as well. Specifically, we consider a cost function $\cost$ as, 
\begin{align}\label{eq:incentivizedcost}
    \cost(\statevar,\action,\incentive) = \costconstantone_\action \indicator(\statevar\neq \action) + \actioncost_\action + \costconstanttwo_\action\incentive,
\end{align}
here, $\costconstantone_\action>0$, $\actioncost_\action>0$ and $\costconstanttwo_\action<0$, $\action \in \actionspace$ are the coefficient for the misclassification cost, the cost incurred in performing the action and the coefficient for the incentive, respectively. The cost accounts for the cumulative cost that the autonomous \bayesianagent\ incurs while classifying the textual input and the incentive received for the same. 

The cost incurred by the central controller for performing information fusion and incentivization is assumed to be linear in the incentive and is taken to be, 
\begin{align}\label{eq:fusioncost}
    \fusioncost(\incentive_\timeindex,\timeindex) = \incentive_\timeindex - \fusionweight(\timeindex)\indicator(\action_\timeindex =\observation_\timeindex|\prior_{\timeindex-1}),
\end{align}
where $\fusionweight(\timeindex)$ is the function that determines the coefficient of the reward associated with the \bayesianagent\ revealing the observation (and not herding). Generally, $\fusionweight$ is decreasing in $\timeindex$ since the benefit of a new observation has diminishing returns as more observations are made available. 

We now discuss our stochastic control approach to optimally incentivize the \bayesianagents, such that the fusion cost of the central controller is minimized. We consider the following natural filtration, which contains all information known to the central controller at time $\timeindex$, namely, the initial prior, actions of the \bayesianagents\ and the incentives by the central controller,
\begin{align*}
    \fusionfiltration_\timeindex = \sigma(\{\prior_0,\action_0,\dots,\action_\timeindex,\incentive_1,\dots,\incentive_{\timeindex-1}\}).
\end{align*}
With the estimates of the cost of the individual \bayesianagents\ and observation matrices, the central controller can use $\fusionfiltration_\timeindex$ to compute the public belief using~\eqref{eq:priorupdate}. If the cost and observation matrices are not available, the central control can learn the optimal policy using policy gradient algorithm, as discussed in the next section. Similar to the previous subsection, $\prior_\timeindex$ can be shown to be a sufficient statistic for the filtration $\fusionfiltration_\timeindex$ and therefore the incentivization policy $\incentivepolicy:\probspace(\statespace)\to\R^{+}\cup\{0\}$ of the central controller determines the incentives as, 
\begin{align*}
    \incentive_{\timeindex+1}  = \incentivepolicy(\prior_\timeindex).
\end{align*}
Therefore, the discounted cumulative cost of the central controller with a discount factor of $\discountfactor$ can be written as,
\begin{align}\label{eq:avgcostfusion}
    \avgcost_{\incentivepolicy}(\prior) = \expectation\{\sum_{\timeindex=0}^{\infty} \discountfactorfusion^\timeindex \fusioncost_{\incentivepolicy}(\incentive_\timeindex,\timeindex) \}.
\end{align}
The expectation is with respect to the observations and the randomized incentive policy. The optimal incentive policy $\incentivepolicy^*$ is then the policy which achieves the minimum cost,  
\begin{align}\label{eq:optincentivepolicy}
    \avgcost_{\incentivepolicy^*}(\prior) = \inf_{\incentivepolicy\in\policyspace}\avgcost_{\incentivepolicy}(\prior).
\end{align}
Although classical methods like value iteration can be used to solve the continuous-valued optimization problem of~\eqref{eq:optincentivepolicy}, for the simple case of $|\actionspace| = |\statespace| = 2$, we show in the next subsection that the optimal policy has a threshold structure which can be searched much more efficiently using a policy gradient algorithm. Note that this reduction is still practical for failure state detection e.g. bad product identification using product reviews on online platforms.   
\subsection{Structural Results}
To show structural results, in this section, we consider the simplification that $|\actionspace| = |\statespace| = 2$. To show the structural results, we make the following assumption on the augmented cost function of the \bayesianagent\ given in~\eqref{eq:incentivizedcost}.
\begin{enumerate}[start=5,label={(\bfseries S\arabic*):}]
    \item The cost function is submodular in $(\statevar,\action)$ for all incentives $\incentive$, i.e., for $|\actionspace| = 2$, $|\statespace| = 2$,  $\cost(1,1,\incentive) + \cost(2,2,\incentive) \leq \cost(1,2,\incentive) + \cost(2,1,\incentive)$.
\end{enumerate}

We consider the following incentive function,
\begin{align}\label{eq:incentivefunction}
    \incentivefunction(\observation,\prior) = \frac{\costconstantone_2 - \costconstantone_1 }{\costconstanttwo_2 - \costconstanttwo_1}\frac{\observationmatrix_\observation\prior}{\indicator \observationmatrix_\observation\prior} + \frac{\actioncost_2 - \actioncost_1}{\costconstanttwo_2 - \costconstanttwo_1}.
\end{align}
The above incentive function comes naturally when the cost function of the individual \bayesianagents\ (\eqref{eq:incentivizedcost}). The derivation is also there in the appendix. We next state results on the structure of the optimal incentivization policy of~\eqref{eq:optincentivepolicy}.
\begin{theorem}\label{th:fusioncenter}
    Under (S4) and (S5), the optimal incentive policy $\incentivepolicy^*:\probspace(\statespace)\to\R^{+}\cup\{0\}$ of~\eqref{eq:optincentivepolicy} is a threshold with respect to the public belief and can be computed as, 
    \begin{align}\label{eq:thresholdincentivizedpolicy}
        \incentivepolicy^*(\policy) = \begin{cases}
            0, & \text{if} \ \prior(2) \in [0,\thresholdprior] \\
            \incentivefunction(\observation,\prior), &  \text{if} \ \prior(2) \in [\thresholdprior,1]
        \end{cases},
    \end{align}
    where $\incentivefunction$ is the incentive function of the central controller from~\eqref{eq:incentivefunction} and $\thresholdprior$ is the threshold value. 
\end{theorem}


Similar to the result of Theorem~\ref{th:threshold}, this theorem shows that the optimal incentivization policy of the central controller is threshold in the public belief. Since the state space has cardinality $2$, the threshold switching curve becomes a single threshold point $\thresholdprior$, and can be efficiently searched.

Theorem~\ref{th:fusioncenter} informs that to minimize the discounted cost of~\eqref{eq:avgcostfusion}, the central controller should incentivize (using the scheme of~\eqref{eq:incentivefunction}) only when the public belief is not \textit{too strong} in favor of state 1. This supports the intuition that more incentive would be required for a stronger public belief (\eqref{eq:incentivefunction}). 

Further,  the incentivization function of~\eqref{eq:incentivefunction} and an optimal policy with a threshold structure of~\eqref{eq:thresholdincentivizedpolicy} implies that the incentive sequence by the central controller is a sub-martingale. This result and a concentration inequality type bound on the cumulative incentive spent are formalized in the next result. 

\begin{theorem}\label{th:submartingale}
Considered the controlled incentivized fusion of information from \bayesianagents\ where the cost function is~\eqref{eq:avgcostfusion} and the optimal incentive policy, $\optincentivepolicy$ satisfies~\eqref{eq:optincentivepolicy} with the incentive function of~\eqref{eq:incentivefunction}. 
    Under (A1) the optimal incentive sequence $\incentive_\timeindex = \optincentivepolicy(\prior_{\timeindex-1})$ is a  sub-martingale, i.e., $\incentive_\timeindex \geq \expectation\{\incentive_{\timeindex-1}\}$.  Further, the cumulative incentive spent is such that in a sample path is such that, 
    $\probabilitymeasure(\sum_{1\leq\timeindex\leq \timehorizon}\incentive_\timeindex \geq  \budget)  \leq \frac{\timehorizon}{\budget}$, where $\budget$ can be considered as a budget constraint. 
\end{theorem}
The above theorem characterizes the nature of a sample path of incentivization and secondly provides a bound on the probability that the cumulative incentive exceeds the budget $\budget$.  This helps analyze the deviation of the total expenditure from the budget $\budget$, which is a constraint of the central controller. 
\subsection{Summary}
 Autonomous \bayesianagents\ are already offered by third-party entities as a service~\cite{agentforce,openaiagents,whitepaperagents}. They have a unit monetary cost associated with using them for any application. Motivated by such \bayesianagents, we studied stochastic control of autonomous \bayesianagents\ who are incentivized by a central controller to perform Bayesian inference. We showed structural results on the optimal incentive policy and a concentration inequality, which characterized the probability that the central controller would exceed their budget. The next section proposes a policy gradient algorithm that exploits the structural results of Theorem~\ref{th:threshold} and Theorem~\ref{th:fusioncenter} to search for the corresponding threshold policy of the social welfare \bayesianagent\ and the central controller, respectively.
\section{Policy Gradient for estimating the optimal policy to control herding in LLMAs}\label{sec:stochasticapproximation}
We propose a policy gradient algorithm that searches for the optimal policy for the stochastic control of the social welfare \bayesianagents\ \eqref{eq:optpolicy} and the central controller \eqref{eq:optincentivepolicy}, which have the threshold structure of~\eqref{eq:optthresholdpolicy} and~\eqref{eq:thresholdincentivizedpolicy}, respectively.
\subsection{Motivation. Efficiently Estimating the optimal stopping time policy.}
When 
\bayesianagents\ are deployed on online platforms to perform Bayesian inference for various purposes, one needs to control for the herding the \bayesianagents\ exhibit. For this purpose, we proposed two optimal stopping formulations for centrally controlled and autonomous \bayesianagents. The efficiency of estimation of the optimal policy parameters is especially important in real-life applications where the time to update the policy parameters is limited. Another constraint is that in real-life systems, access to the observation matrices is limited; therefore, the estimation has to be done without the knowledge of the system parameters. Therefore, we use structural results on the optimal policies for the stopping time formulations and use policy gradients to estimate the optimal policy. 
\subsection{Policy Gradient Algorithm}
Searching for a hard threshold of the form 
~\eqref{eq:optthresholdpolicy} and~\eqref{eq:thresholdincentivizedpolicy} can be formulated as a discrete stochastic optimization problem. However, in this section, we relax the problem to a continuous stochastic optimization by approximating the hard threshold policy by a sigmoidal of the form, 
\begin{align}\label{eq:approximatepolicy}
   \hat{\policy}(\prior;\policyparameter) =  \frac{1}{1+\exp({\frac{\prior-\policyparameter}{\approximationparameter})}},
\end{align}
where $\policyparameter$ is the policy parameter representing the threshold, and \change{$\approximationparameter\in(0,1]$} is the approximation parameter, and the policy converges to a hard threshold policy at $\policyparameter$ as $\approximationparameter\to0$. To restrict the policy parameter to be in the range $[0,1]$, we can reparameterize it as $\sin^2(\theta)$. The approximate parameterized policy of~\eqref{eq:approximatepolicy} can be used to obtain an approximate value of the cumulative value function (of either~\eqref{eq:avgcostfusion} or~\eqref{eq:socialwelfarecost}) by interacting with the system using the policy for $\episodelength$ interactions.

For the social welfare cost of~\eqref{eq:socialwelfarecost}, we can compute the approximate cost with respect to a policy by 
\begin{align}\label{eq:approxsocialwelfarecost}
\begin{split}
  &\approxcost(\hat{\policy}(\cdot,\policyparameter_\timeindex))  \\= &\sum_{\timeindex = 1}^{\episodelength} \hat{\policy}(\prior_\timeindex,\policyparameter_\timeindex)\left( \discountfactor^{\timeindex-1} \cost(\statevar,\action_\timeindex) + \discountfactor^{\timeindex-1} \delaycost  \indicator(\statevar=\indicatorstate_1)\right) +
\\&
    \ \discountfactor^{\episodelength-\sum_{\timeindex = 1}^{\episodelength} \lfloor 0.5+ \hat{\policy}(\prior_\timeindex,\policyparameter_\timeindex) \rfloor } \left[ \errorcost\indicator(\statevar\neq \indicatorstate_1 ) +\ \frac{\min_{\action \in \actionspace}  \cost(\statevar,\action)}{1-\discountfactor}  \right]
 ,
\end{split}
\end{align}
where $\sum_{\timeindex = 1}^{\episodelength} \lfloor 0.5+ \hat{\policy}(\prior_\timeindex,\policyparameter_\timeindex) \rfloor$ just computes the empirical stopping time in place of the stopping time $\stoppingtime$ in~\eqref{eq:socialwelfarecost}. 
Note that in the above equation, the filtration and the expectations have been replaced with the realized cost, and hence, this is a noisy estimate of the true expected cost of~\eqref{eq:approximatecost}.

For fusion cost of~\eqref{eq:avgcostfusion}, the approximate cost is given by, 
\begin{align}\label{eq:approximatecost}
    \approxcost(\hat{\policy}(\policyparameter_\timeindex)) = \sum_{\timeindex=1}^{\episodelength} \fusioncost(\incentive_\timeindex,\timeindex).
\end{align}

We now describe our simultaneous pertubation based policy gradient algorithm, which is summarized in Algorithm~\ref{alg:stochasticapprox}. We perform $\timehorizonspsa$ learning episodes, each of which is initialized with the same conditions, and we adaptively update the policy parameter $\policyparameter_\timeindexalt$ across these episodes. For each episode $\timeindexalt$, we perturb the policy parameter by $\pm \delta$ and obtain two policy parameters $\policyparameter^{+}_\timeindexalt$ and $\policyparameter^{-}_\timeindexalt$. We then approximate the cost function for both the parameters using~\eqref{eq:approximatecost} by performing $\episodelength$ interactions each. We approximate the gradient by finite-difference method (step 6 of Algorithm~\ref{alg:stochasticapprox}). Finally, the policy parameter uses a gradient descent step with a step size $\policyparameter_\timeindexalt$. If the parameters of the system are known to be more or less constant, the step size is decreasing to ensure asymptotic convergence~\cite{kushner2003stochastic}, else a constant step-size can be used to track changes in the true threshold parameter.
\begin{algorithm}
    \begin{algorithmic}[1]
        \State \textbf{Input: } Initial Parameter $\policyparameter_0$, Perturbation $\perturbationsize$, Time Horizon $\timehorizonspsa$, Step Sizes $(\stepsize_\timeindexalt)$, Episode Length $\episodelength$
        \State \textbf{Output: } Terminal Parameter $\policyparameter_\timehorizonspsa$
        \For{$\timeindexalt$ in $1,\dots,\timehorizonspsa$}
        \State Perturb parameters $\policyparameter_\timeindexalt^{\pm} \gets \policyparameter_{\timeindexalt} \pm \perturbationsize$.
        \State Approximate cost with~\eqref{eq:approximatecost} using $\episodelength$ interactions with $\policyparameter_\timeindex^{+}$ and $\policyparameter_\timeindex^{-}$, $\approxcost(\hat{\policy}(\policyparameter_\timeindex^{+}))$ and $\approxcost(\hat{\policy}(\policyparameter_\timeindex^{-}))$.
        \State Approximate gradient $\approxgradient \avgcost \gets \frac{\approxcost(\hat{\policy}(\policyparameter_\timeindex^{+}))- \approxcost(\hat{\policy}(\policyparameter_\timeindex^{-}))}{2\perturbationsize}$.
        \State Update parameter using $\policyparameter_{\timeindexalt+1} = \policyparameter_{\timeindexalt} - \stepsize_\timeindexalt (\approxgradient \avgcost)_\timeindexalt$.
        \EndFor{}
    \end{algorithmic}
    \caption{Stochastic gradient algorithm for estimating optimal policy}
    \label{alg:stochasticapprox}
\end{algorithm}
\begin{remark}
    Apart from the parameters of the algorithm, the algorithm only requires the approximate reward with a particular policy. If the cost function is not known to the controller (which is the case when they are incentivized), then the framework from Section~\ref{sec:ribum} can be used to estimate the utility of the individual \bayesianagents\ using Algorithm~\ref{alg:maxmargin} or Algorithm~\ref{alg:sparse}. The negative of the utility can be used as the reward function.  This can account for the misclassification cost. Different incentive regimes can be considered as different environments to exactly obtain the cost function for an incentivized autonomous \bayesianagent.
\end{remark}
\begin{remark}
    The above algorithm does not need access to the probability distributions $\observationmatrix$, which is expensive to obtain, especially in the incentivized case. Also, if the underlying parameters of the setup are evolving on a slower timescale, the above algorithm can be run with a constant step size to make the policy parameters track the changes in the system. These are the key advantages of using Algorithm~\ref{alg:stochasticapprox} compared to the value iteration or other sub-optimal methods for solving our proposed non-standard POMDPs. The computational complexity of Algorithm~\ref{alg:stochasticapprox} for each iteration is $\orderof(\episodelength)$. 
\end{remark}
\begin{remark}
\change{
   For the centrally controlled case, only a single parameter needs to be communicated across the LLMAs to learn the optimal stopping time policy. Further, since the LLMAs act in a sequential fashion, such a parameter can be communicated either by the central controller or by the previous LLMA, making this a scalable approach. For the incentivization policy too, a single parameter needs to be estimated by the central entity employing the autonomous LLMAs, which makes the policy estimation protocol independent of the number of agents, and therefore scalable.}
\end{remark}

\subsection{Summary} 
This section proposed a computationally efficient policy gradient algorithm to estimate the optimal policy of the stopping time problems for \bayesianagents\ of Section~\ref{sec:quickesttimeherding} and Section~\ref{sec:incentivizedherding}. This algorithm does not need the underlying parameters of~\bayesianagents\ and is adaptive to changes to the underlying parameters. The next section presents numerical studies that illustrate how these models can be used to control LLMAs. 

\section{Numerical Results: Sequential Bayesian Sentiment Analysis using LLMAs}\label{sec:numerical}
 Our numerical studies demonstrate how the Bayesian social learning framework and the stochastic control approach can be used to perform more accurate and efficient state estimation in the described applications. We present numerical studies on applications related to two real-life datasets: a hate-speech dataset and a product review dataset. These numerical experiments build on our past work to build robust hate speech classification using covert federated learning~\cite{jain2024controlling,jaincontrol}. The reproducible code, appendix with proofs, and the dataset link are on github.com/aditj/sociallearningllm.
\subsection{Motivation} 
LLMs are already used in different real-world applications, including on e-commerce platforms, to provide an overview of the reviews and on social media platforms to flag malicious content; therefore, motivated by the real world, we present numerical experiments where LLMAs are used to perform Bayesian inference on different real-world datasets.  

We first describe the two main real-world tasks and datasets on which the numerical results are presented. Then, we show how our construction of a single \bayesianagent\ leads to interpretability. We extend the exemplary study from Section~\ref{sec:ribum} and conduct more extensive experiments. Next, we show numerical results on herding phenomena in a sequence of \bayesianagents. Finally, we show how the optimal policy for optimal stopping has a threshold structure. We also show the efficacy of policy gradient algorithm in a simulated setting. 
\subsection{Task Description}
For both tasks, the \bayesianagent\ is an online detection mechanism equipped with a large language model (LLM) sensor that analyses comments and flags the state sequentially. The LLM is used to parse the text and obtain a list of appropriate features from the finite-dimensional feature space $\observationspace$, which the platform could design. The features contain information about the text and comparisons of the text with a given context. In our setup, the Bayesian engine of the \bayesianagent\ consists of a likelihood parameterized by a neural network. For a discrete distribution, the likelihood neural network uses restricted Boltzmann machine (RBM) to generate samples from the likelihood with $1000$ samples of annotated comments (paragraphs). The experiment results are averaged over $N_{MC}$ independent runs mentioned along with each experiment. The posterior can be updated using~\eqref{eq:bayesrule} using the likelihood and the closed form prior of~\eqref{eq:priorupdate} from the previous step. The \bayesianagent\ takes an action according to~\eqref{eq:action}. The action is classifying whether the user has hateful intent, and the cost accounts for the misclassification of the state.  We use Mixtral 7B, an open-source mixture of experts LLM with 7 billion parameters~\cite{jiang2023mistral7b}, LLaMA-3~\cite{touvron2023llamaopenefficientfoundation} with 70b parameters, and ChatGPT-4o which is a closed source LLM. The details of which LLM was used for this experiment is in the Appendix. We query using the TogetherAI API for the open-source LLMs (Mixtral and LLaMA) and OpenAI API for ChatGPT.  

\subsubsection{Hate Speech Classification}
\textit{Motivation: }Flagging users who spread toxic content online is a significant challenge. The state $\statevar_\timeindex$ represents the category of peddlers classified based on the intensity and type of content they are propagating. For example, the state could be 3-dimensional, indicating the hate-intent of the user (hateful or not), the hate speech intensity scale~\cite{bahador_classifying_nodate}, and the particular group the hate speech is directed towards. The noisy observations are the text comments from the user that inform about the state and are from a high-dimensional observation space. 

In a social network, there may be multiple  \bayesianagents\ deployed to flag malicious users and decrease the propagation of hate speech. The flags by the previous \bayesianagent\ are visible, but the private observations are not due to computation and privacy restrictions (so that the LLM can not be fine-trained on the text or the feature mappings). Since the observations are generated sequentially, the \bayesianagent\ may use the same LLM but with a different context and for a different text observation. Therefore, a single \bayesianagent\ can be viewed as a sequence of \bayesianagents\ learning from their private observations and the past actions of previous \bayesianagents. 

The state space models the state of the user, $\statespace = \{0= \text{non-hateful}, 1 = \text{hateful}\} \times \{1,2,3,4,5\}$, where the first dimension corresponds to whether a user is a hate speech peddler (HSP) or not, and the second dimension to the intensity of the toxicity of the speech (evaluated by crowdsourcing~\cite{sachdeva_measuring_2022}). For the numerical result, we consider an augmented state space with $6$-states, $\statespace=\{0=\text{non-hateful},1,2,3,4,5\}$, where the last $5$ entries denote an HSP of different intensities as described later. 

The high-dimensional observations of the state are in the form of text comments posted on online platforms. An LLM is used to parse the text observations by prompting the LLM with the text and a system prompt to return an output belonging to an observation space $\observationspace$, which contains the following binary variables: a) targetted towards someone and b) contains explicit words c) indicate violence d) has bias e) is dehumanizing f) is genocidal. We augment the output of the LLM to an observation space of cardinality $|\observationspace| = 6$. The details of this augmentation, along with additional experimental details, are in the appendix.  

The action space $\actionspace$ is considered the same as the state space, and the cost function which accounts for the misclassification of an HSP is given by,
\begin{align}\label{eq:misclassificationcost}
    \cost(\statevar,\action) &= \indicator(\statevar\neq 0)[\indicator(\action=0)+ |\statevar-\action|],
\end{align}
The first time accounts for the misclassification of a hateful user, and the second term accounts for the difference in intensity.  
We use the measuring hate speech dataset from~\cite{sachdeva_measuring_2022}, which contains \change{$40,000$} annotated comments from Twitter, Reddit, and Gab. The annotations are performed by crowdsourcing and indicate if the comments contain hate speech and the intensity of the toxicity exhibited on a scale of 1 to 5, measured using a Rasch measurement scale~\cite{sachdeva_measuring_2022}. 

Since the data is anonymized, we consider a synthetic user construction. In a span of $\timehorizon$ textual comments, a hate speech peddler (HSP) contains hate speech text from one of the intensity levels. Hence, there are $6$ types of users: non-HSP and HSP with intensity from 1 to 5, each with $\timehorizon=100$ comments of the corresponding intensity.

\subsubsection{Product Quality Identification}
\textit{Motivation: } Identifying products that are of poor quality early on is important in many e-commerce applications. While the current state-of-the-art relies on a complaint-driven interface that involves a human service agent to identify a bad product, there has been some preliminary work that suggests using existing reviews from the customers to determine the quality of the product~\cite{marketingreviews}. Product reviews are already used to provide a summary of the product on major e-retailers like Amazon and even influence customer decisions~\cite{marketingreviews}. Therefore, we propose using a sequence of \bayesianagents\ and our stopping time formulation to identify bad products.  

In the case of an e-commerce platform, the restriction on the sharing of the private observation is motivated by the computational constraints that the \bayesianagents\ will have when performing Bayesian inference on millions of products and billions of reviews. The public belief can be efficiently computed based on the output (action) of each of the \bayesianagent.

We use the Amazon Review dataset, which has $233.1$ million~\cite{ni-etal-2019-justifying} text reviews for products from $29$ categories. Each review also has a rating from $1$ to $5$. We discuss more details about the dataset in Appendix~\ref{app:amazon}.
We consider the state as the quality of the product, with the state space given by $\statespace = \{0 = \text{``bad-quality''},1 = \text{``high-quality''}\}$. We consider the Beauty and Electronics categories and sample $5000$ reviews from each. Each review has a text and a rating associated with it. We consider the quality of the product as described above. 

We now detail on what we mean by \textit{quality} and the definition of the states. Since we do not have access to any additional dataset other than ratings, we consider products with at least $1000$ ratings and compute the quality state ex-ante using these ratings. Therefore, the quality here is more representative of the perceived value for money~\cite{marketingvalue}. We consider the following ranges $[0,3.3),[3.3,5]$ of the average review of \textit{all} the ratings to assign a product quality as bad, medium, and good, respectively. To clarify, we just use all $1000$ ratings to compute the state; the observations are sampled one at a time and analyzed by the \bayesianagent. 

The low-dimensional observation obtained using the LLM of the \bayesianagent\ belongs to an observation space of cardinality $|\observationspace| = 16$. We obtain the low-dimensional observation by asking the following questions of the text review:
\begin{enumerate}
    \item Does the review mention any specific problems or defects with the product? (\textit{defects})
    \item Does the review mention any positive attributes regarding the product's durability or reliability? (\textit{durability})
    \item Does the review indicate that the product meets or exceeds the user's expectations? (\textit{expectations})
    \item Would the reviewer recommend this product to others? (\textit{recommend})
\end{enumerate}

These questions were designed by us by qualitatively analyzing what could predict the perception of quality of a product~\cite{marketingreviews,marketingvalue}. These features for the different states are plotted in Figure~\ref{fig:interpretabilityproduct}. It can be seen that for the good product, the overall decrease in \textit{defects} is more rapid, and even when in reviews with overall ratings $2$ and $3$, there are substantially more reviews that mention \textit{durability} positively, than the bad product. However, an analyst could consider an alternate design of the observations depending on the need.

 We consider the misclassification cost as follows,
\begin{align}\label{eq:misclassificationcostproduct}
    \cost(\statevar,\action) = \indicator(\statevar=0)\indicator(\action\neq 0)+\indicator(\statevar\neq0)|\statevar-\action| .
\end{align}

\begin{figure}
    \centering
    \resizebox{\columnwidth}{!}{\input{plots/analysis_plot_product.pgf}}
    \caption{Interpretability of the outputs of an LLM Sensor for Product Quality Identification. There are $4$ different features that the LLM of the \bayesianagent\ extracts. The left subplot corresponds to reviews from a good product, and the right subplot corresponds to reviews from a bad product (defined in the text). We analyze text reviews with ratings from $1$ to $5$ for $4$ different attributes and plot the proportion of samples where the attribute is true. It can be seen that the features, indeed, are consistent with the overall rating provided by the reviewer for the two different kinds of products and can be used for a more fine-grained interpretable analysis. }
    \label{fig:interpretabilityproduct}
\end{figure}

\subsection{Numerical Evidence showing \bayesianagents\ are \ribum}
\begin{figure}
    \centering
    \includegraphics[width=0.8\columnwidth]{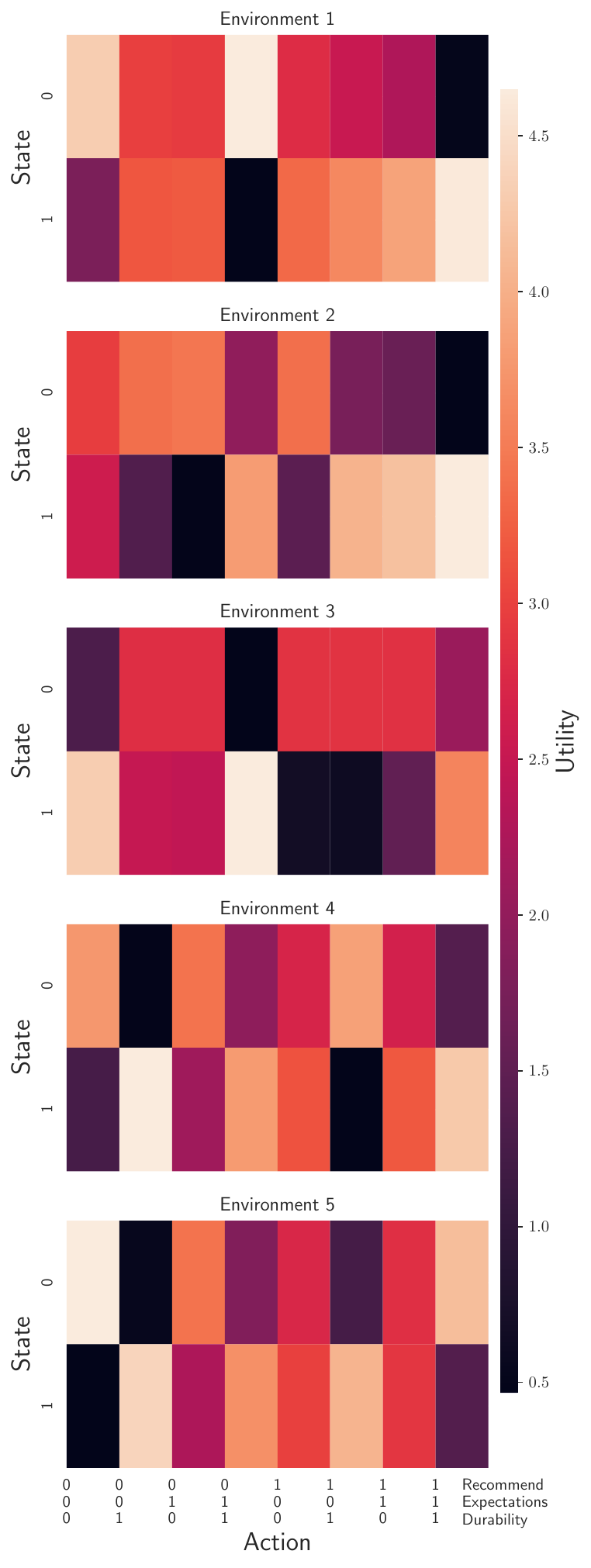}
    \caption{Reconstructed max-margin utility of the LLM of an \bayesianagent\ performing Bayesian inference for product quality identification using text reviews. 
    The reconstructed utilities offer an interpretable way to analyze the behavior of the LLM. 
    Here, the states represent the true quality of the product (0 = bad, 1 = good). The actions are the low-dimensional output of the LLM corresponding to different features in the input text reviews.  We observe that for environment $1$, when reviews with a rating $1$ are considered, the difference between utilities for both states is negligible. However, the contrast increases as the ratings considered in the environment are increased.} 
    \label{fig:utilityproduct}
\end{figure}
We perform the max-margin-based utility reconstruction using Algorithm~\ref{alg:maxmargin} when the agent optimizes a utility in order to achieve product quality identification and illustrate the reconstructed utilities in Figure~\ref{fig:utilityproduct}.  

For illustration purposes, we consider \textit{recommend, durability} and \textit{expectations} as the actions, hence making the action space $|\actionspace| = 8$. To show our methods can be used with blackbox LLMs, we use \texttt{ChatGPT-4o-mini} for this experiment. The state space is $\statespace=\{0 = \textit{bad \ product}, 1 =\textit{good \ product}\}$. We consider the different environments to have different ratings; i.e., the people who rate products at $1$ are considered as part of a single environment. The utilities can be used to interpret the behavior of the LLM and can be used in lieu of explicit utilities of the form~\eqref{eq:misclassificationcostproduct}.

\subsection{Herding In \bayesianagents}
\begin{figure}[t!]
    \centering
    \includegraphics[width=\linewidth]{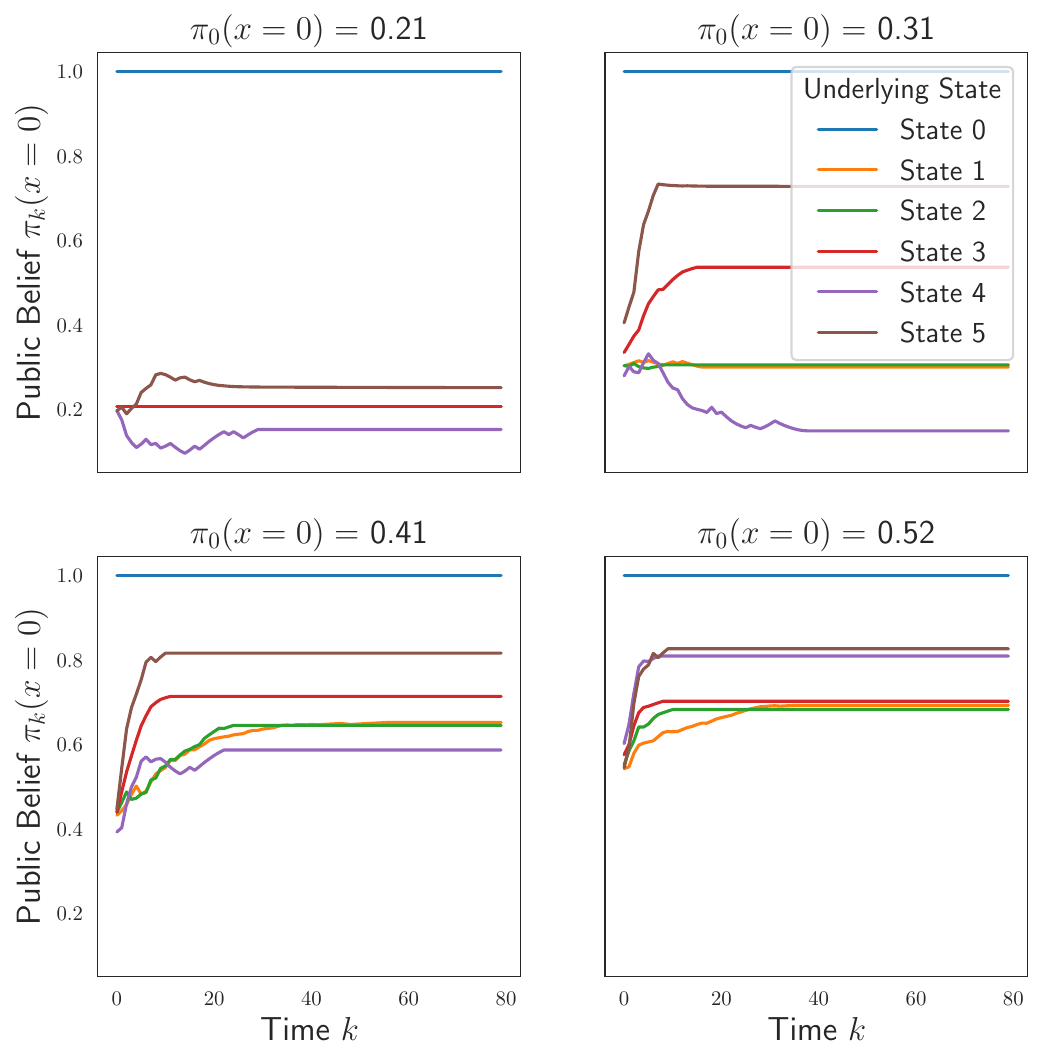}
    \caption{Public prior freezes in finite time when \bayesianagents\ learn using social learning protocol of Algorithm~\ref{alg:sociallearning} for detecting hate-speech peddlers. The time taken (average over $10$ runs) to form an information cascade is different for different values of the initial public belief $\prior_0(\statevar=0)$ and different true underlying states. Note that for a public prior $\prior_0(\statevar=0) \geq 0.41$ only a few observations are enough for the LLMAs to herd to a strong belief on the wrong state (state $0$) regardless of the true underlying state. }
    \label{fig:priorcascade}
\end{figure}
\begin{figure}[t!]
    \centering
    \includegraphics[width=\linewidth]{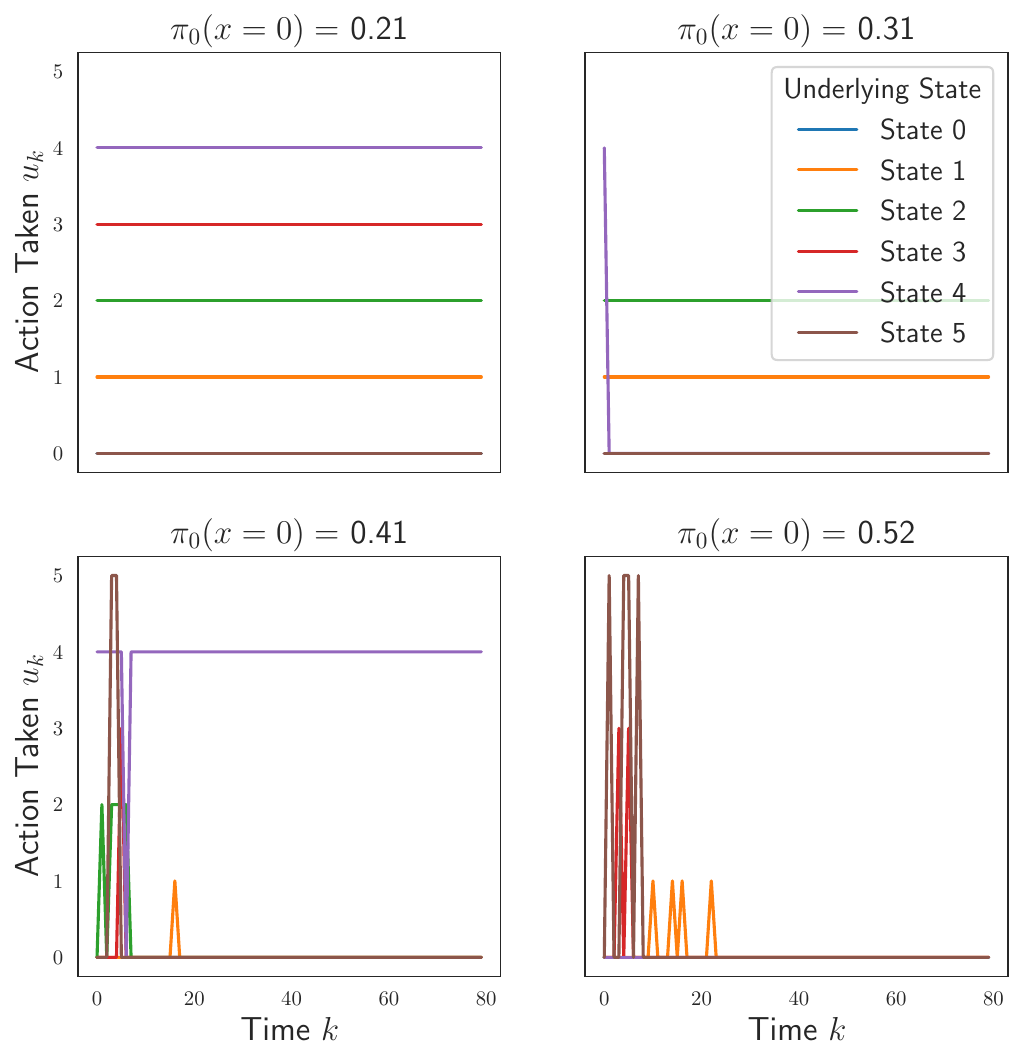}
    \caption{Sample path for the actions taken by \bayesianagents\ under different initial priors $\prior_0(\statevar=0)$ and true underlying state $\statevar$. Observation $4$ has a stronger observation likelihood and hence is more robust under an increasing prior on state $0$. }
    \label{fig:samplepathherding}
\end{figure}
We report our results with the initial public belief (specifically the prior probability of state 0). Due to Theorem~\ref{th:herding}, the initial public belief is sufficient to identify intermediate public belief regions where the information cascades are observed.

We first study the freezing of public belief when an information cascade (Definition~\ref{def:infocascade}) happens in a sequence of \bayesianagents\ performing Bayesian inference for hate-speech classification. We consider $80$ timepoints, and the results average over $10$ runs. The result is illustrated in Figure~\ref{fig:priorcascade} for different values of initial public belief. In each subplot, the different lines correspond to the different true underlying states. It can be seen that the public belief for state 0 converges to higher values as the initial prior for state 0 is increased. Hence, the initial prior decides what the prior will freeze.

We show, for a sample path realization, how  \bayesianagents\ herd in the actions that they take in Figure~\ref{fig:samplepathherding}. Each of the subplots corresponds to a different initial public belief, and the lines within each subplot are for a different underlying true state.  The reason why it is more resilient to herding when in state 4 is that there is a stronger observation likelihood. All the other states, for a strong enough prior ($0.52$), herd to an incorrect state ($0$). Hence classifying an HSP as a non-HSP.  

 Figure~\ref{fig:herdingregions} illustrates the regions of herding a case with $3$ states.  The true underlying state in this simulation is $0$, and therefore, it can be seen that most of the initial public belief corresponds to predicting state $0$. We assume an observation matrix with $3$ observations, with an observation matrix 
$B = [0.7,0.2,0.1;0.1,0.7,0.2;0.2,0.1,0.7]$ and assume an identity utility function for the \bayesianagents. However, the top-left and bottom-left corners of the triangle show that herding to the wrong state can happen if the public belief is too strong.

\begin{figure}[h!]
    \centering
\includegraphics[width=\linewidth]{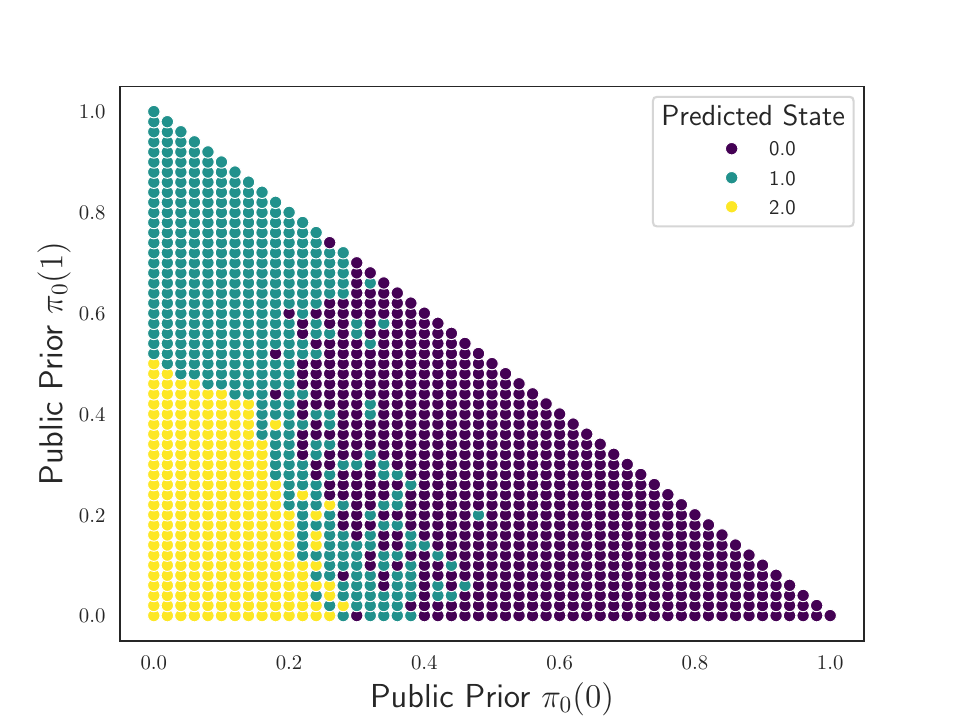}
    \caption{Regions in belief space where LLMAs herd during Bayesian social learning with $|\statespace| = 3$ states. Even though the underlying state is $\statevar=0$, in the bottom-left and top-left regions, the actions are $2$ and $1$, respectively, because of a strong public belief.  }
    \label{fig:herdingregions}
\end{figure}
\subsection{Optimal Stopping for delaying herding in \bayesianagents}
\begin{figure}[h!]
    \centering
    \resizebox{\columnwidth}{!}{\input{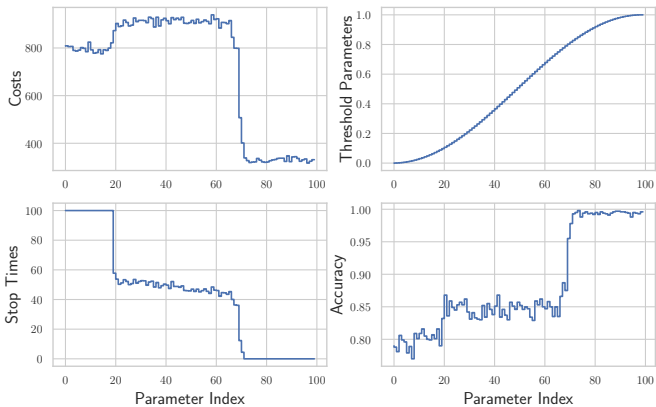}}
    \caption{The cost function (top-left) corresponds to different values of the threshold parameter (top-right) of the policy of~\eqref{eq:thresholdpolicy}. Each policy leads to a different stopping time (bottom-left) and corresponding accuracy (bottom-right). For this experiment, it can be seen that increasing the stopping time decreases the cost.  The sudden jumps in the cost (and stopping time) are due to the transition in the resulting policy parameter from a region of learning to herding. }
    \label{fig:parametersweep}
\end{figure}

\begin{figure}[h!]
    \centering
    \resizebox{\columnwidth}{!}{\includegraphics{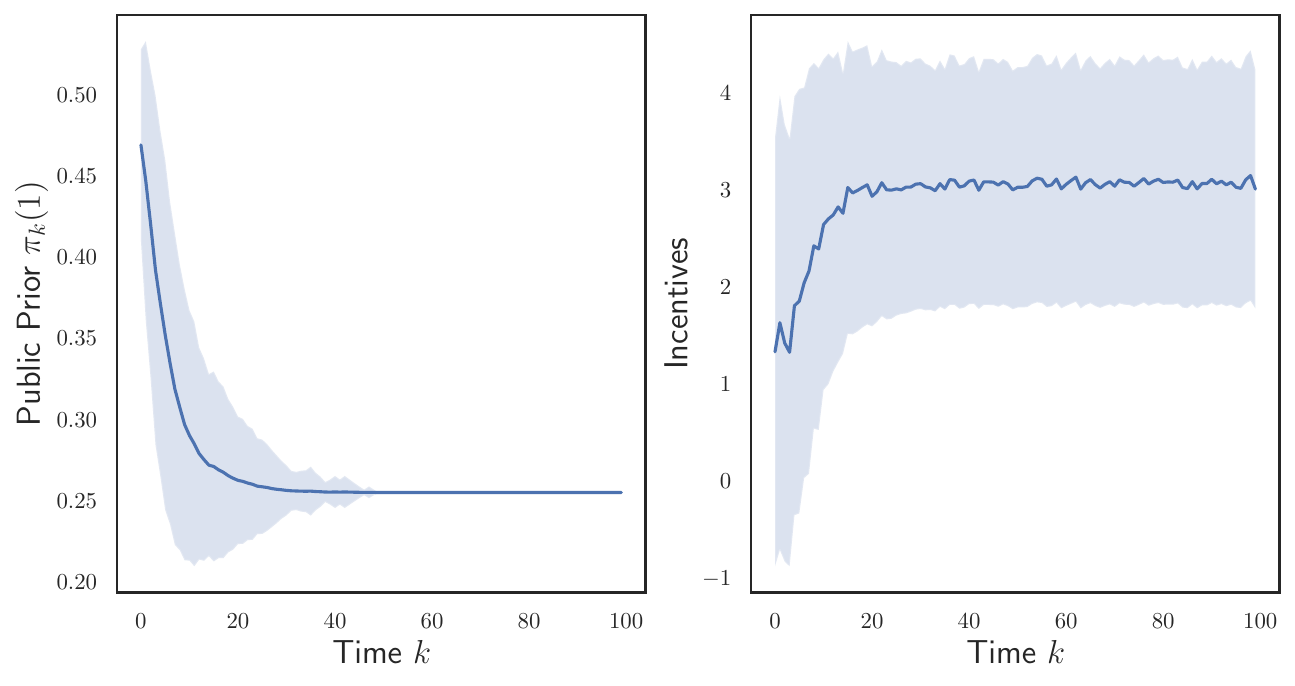}}
    \caption{Incentivized stochastic control of autonomous \bayesianagents.  The public belief (left) converges slower as a result of incentivization.  Also note that the incentive (right) is a supermartingale (as proved in Theorem~\ref{th:submartingale}). This is for a policy of the form~\eqref{eq:thresholdpolicyincentivized} threshold parameter of $\threshold=0.4$.}
    \label{fig:incentivized}
\end{figure}
Next, we study the optimal stopping of \bayesianagents\ first when these agents are centrally controlled and next when these agents are autonomous. 

For both experiments, we consider the product quality identification task. To make sure the actions and the observations are of the same dimensions, we assume that the observation space is also two-dimensional $|\observationspace|=2$. The observation matrix was taken to be $B = [[0.7,0.3],[0.3,0.7]]$. Note this is necessary since here we ask the \bayesianagents\ to either reveal the private observation or to take action using Eq.~\ref{eq:action}, and update in Eq.~\ref{eq:priorupdate} needs to be consistent. 

For the case when the \bayesianagents\ are centrally controlled, and the simplified system model, policies that have a switching threshold curve (Eq.~\ref{eq:optpolicy}) with respect to the belief space can be represented as 
\vspace{-2mm}
 \begin{align}\label{eq:thresholdpolicy}
        \policy(\prior) = \begin{cases}
           2 \ \text{(continue) if} \ \prior(0) \leq \threshold \\
           1 \ \text{(stop) if} \ \prior(0) > \threshold  \\
        \end{cases},
    \end{align}
where $\threshold$ is the threshold parameter.
The true state is uniformly sampled from $\{0,1\}$. The delay cost is taken as $\delaycost=10$, and the misclassification cost is taken as $\errorcost=50$. The discount factor is taken as $\discountfactor=0.99$, and the infinite horizon is approximated by a horizon of length $100$. The cost for each indiviual \bayesianagents, $\cost(\statevar,\action) = \indicator(\statevar\neq \action)$. 

In Figure~\ref{fig:parametersweep}, we perform a parameter sweep over $\threshold\in[0,1]$ with a grid cardinality of $100$, and evaluate three different aspects. The results are averaged over $100$ independent runs. In the top-left corner, we evaluate the cost function as a function of the threshold parameter. We observe that the cost is minimum for a higher parameter value. And a higher parameter corresponds to a smaller stop time. The accuracy is computed as the proportion of times the correct state is predicted. The sudden jumps in the cost are due to the fact that a small change in the policy parameter is enough to shift the public belief shifts from a region of learning to herding. 

Next, we look at \bayesianagents\, which are incentivized and autonomous. Here, the  incentivization policy is of the form,
\begin{align}\label{eq:thresholdpolicyincentivized}
        \policy(\prior) = \begin{cases}
          0 \ \text{(stop incentivizing) if}  \ \prior(0) \leq \threshold \\
           \incentivefunction(\prior,\observation) \ \text{(incentivize) if} \ \prior(0) > \threshold  \\
        \end{cases},
    \end{align}
    where $\threshold$ and $\incentivefunction$ is of the form~\eqref{eq:incentivefunction}. We consider the parameters $\costconstantone_2 = 1.3$, $\costconstantone_1 = 0.8$, $\actioncost_2=0.5$, $\actioncost_1 = 0.1$, $\costconstanttwo_2 =0.5$ and $\costconstanttwo_2 =0.2$. The cost function $\cost$ is the same as the previous part, and the composite cost function is given by~\eqref{eq:incentivizedcost}. 
    
We first illustrate our result of Theorem~\ref{th:submartingale} in Figure~\ref{fig:incentivized}, where it can be seen that the sequence of incentives is a submartingale sequence. We fix the threshold parameter of~\eqref{eq:thresholdpolicyincentivized} to $\theta=0.4$. The public belief converges much slower due to the modified cost function (the additional incentive term in~\ref{eq:incentivizedcost}); hence, using our stopping time formulation, we can extract more private observations, which can be later used to get more accurate estimates.  
\begin{figure}[ht!]
    \centering
    \includegraphics[width=0.8\linewidth]{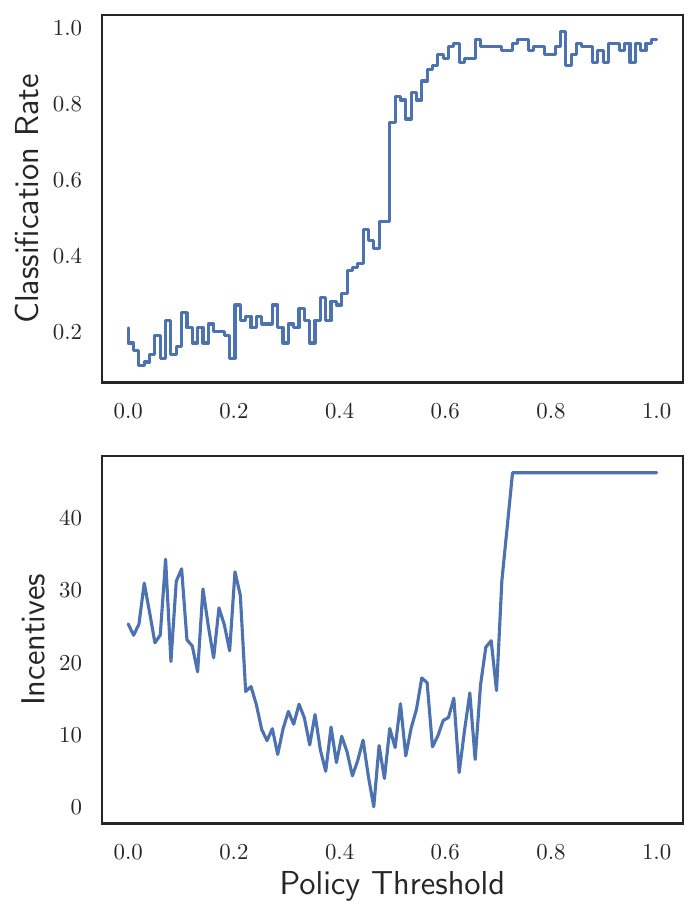}
    \caption{The cumulative incentive spent (bottom) and the achieved classification rate (top) for different values of threshold parameter $\threshold$.  The policy threshold of around $0.7$ achieves a tradeoff between the classification rate and the cumulative incentive expenditure. The incentives are constant after a particular parameter because of herding.   }
    \label{fig:parametersweepincentive}
\end{figure}

Figure~\ref{fig:parametersweepincentive} shows the value of the cost for the different values of the policy threshold $\threshold$. We maintain the same system parameters as the previous two experiments. It can be seen that both the classification rate and the cumulative incentive go up as the policy threshold increases in this example. The sudden jump in the incentive is again due to the sudden switch between the prior region from learning to herding. However, also note that the threshold around 0.6 has an incentive in the range [5,10], and the classification rate is still decent (0.9-1), therefore showing how there is an optimal threshold that optimally achieves a tradeoff between the cumulative incentive expenditure and classification performance. 
\subsection{Illustration of Stochastic Gradient Algorithm}
Finally, we show how the proposed policy gradient algorithm (summarized in Algorithm~\ref{alg:stochasticapprox}) can be used to optimally estimate the policy parameters for optimal stopping in incentivized \bayesianagents\ with system parameters from the previous section. We run $100$ iterations, each of which uses a cost function averaged over $100$ independent runs. We set a linearly decreasing step size from $\stepsize=0.05$ to $0.005$ and set the parameter perturbation as $\perturbationsize=1$ and the approximation factor for the sigmoidal policy of~\eqref{eq:approximatepolicy} as $\approximationparameter=0.3$. The discount factor is $\discountfactor=0.99$.
\begin{figure}[ht!]
    \centering
    \includegraphics[width=0.8\linewidth]{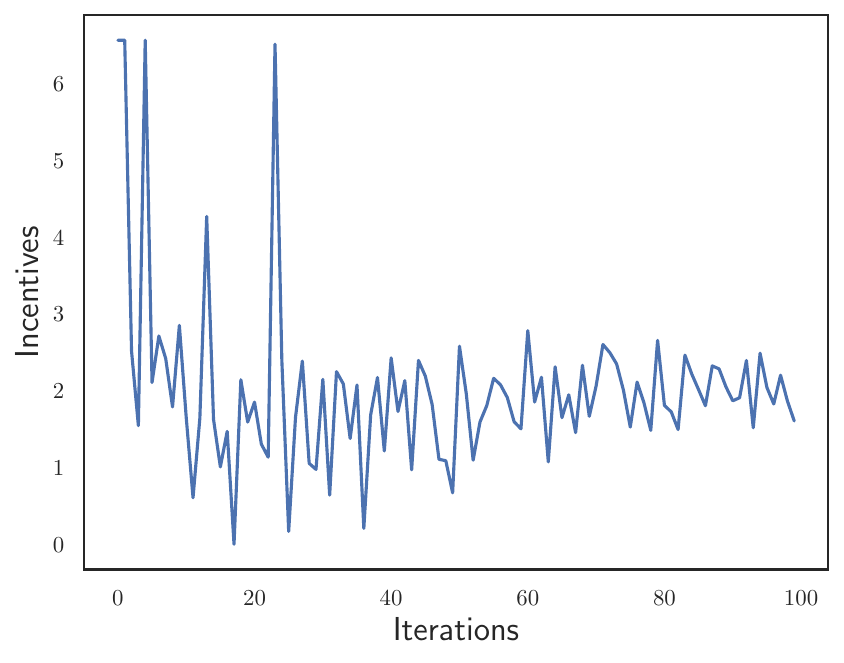}
    \caption{The approximate cumulative incentives corresponding to parameters of a policy gradient for estimating the optimal threshold parameters of~\eqref{eq:optincentivepolicy}. The incentive of the final iterate is much smaller than the incentive expenditure of a too-small or too-large parameter from Figure~\ref{fig:parametersweepincentive}. }
    \label{fig:stochasticapprox}
\end{figure}

Figure~\ref{fig:stochasticapprox} presents how the cumulative incentive expenditure changes with each policy parameter update. It can be seen that after a few iterations, the iterates converge to parameter value such that the incentive expenditure is in the range [5,15], which is close to the optimal incentive from Figure~\ref{fig:parametersweepincentive}.  

 \section{Discussion and Future Work}\label{sec:conclusion}
With the rapid adoption of LLMs across science and engineering and the emergence of LLM-based agents for automating diverse workflows\cite{openaiagents}, it is crucial to systematically study their behavior and interactions. This work addresses this need by proposing interpretable models and stochastic control algorithms to engineer systems of \bayesianagents\ that perform Bayesian inference across a range of applications. Our study highlights the versatility of LLMAs across various fields, aiming to inspire practitioners from diverse disciplines to explore and integrate these tools into their workflows.

In this conclusion section, we summarize how the two layers of abstraction explored in this paper using Bayesian revealed preferences and Bayesian social learning provide a sound basis for driving research in understanding and controlling the behavior of LLM agents. We discuss the implications of our methods and their potential applications in studying interacting \bayesianagents. We close with promising research directions at the intersection of signal processing, network science, and machine learning to enhance the capabilities of \bayesianagents.

\subsection{Summary}
We first presented a Bayesian sensor model for constructing a \bayesianagenttext\ (\bayesianagent). This model was a composition of a) an LLM, which was used as an interpretable and configurable sensor for high-dimensional data, and b) a Bayesian engine to represent and update the belief of the underlying state. Second, we put forth the necessary and sufficient conditions for a \bayesianagent\ to be a rationally inattentive Bayesian utility maximizer (\ribum). We present algorithms for estimating the utility function of a \bayesianagent\, which is \ribum. The reconstructed utility naturally leads to the interpretability of the actions of the \bayesianagent. These methods are applicable to both the Bayesian sensor model of \bayesianagent\ we propose and off-the-shelf \bayesianagents\ and LLMs. Thirdly we look at Bayesian social learning \change{in} a sequence of \bayesianagents. We show that this sequence of agents herd in their actions even if they share their private observations. Then, we formulate optimal stopping problems for failure state detection, which optimally delays herding and improves estimation by allowing the private observations to be shared. The optimal stopping problem is formulated both for a case when the \bayesianagents\ are centrally controlled and when they are autonomous. We show that under relatively mild conditions on the observation matrix and the cost function of the \bayesianagent\, the optimal policy for the optimal stopping problem has a threshold structure. A policy gradient algorithm is proposed to estimate the optimal policy efficiently without the knowledge of system parameters. Finally, we show numerical experiments demonstrating the various claims and frameworks proposed in the study. 

\subsection{Insights}
We conclude the paper with the following key takeaways:
\begin{enumerate}
    \item The Bayesian sensor construction of a \change{\bayesianagent} lends to a lot more interpretability when performing sequential estimation using a large language model, more than if Bayesian inference was performed with just embeddings (which might be more accurate but not interpretable). 
    \item The Bayesian revealed preferences framework is a systematic way to obtain utility (or cost) functions for blackbox \bayesianagents. It's especially useful when the state and action space are small, but the \bayesianagent\ operates in tens of different environments.
    \item For a practitioner, Bayesian social learning might seem simplistic, however the important point is that even a simplistic sequence of \bayesianagents, herding is an undesirable and inevitable phenomena. Therefore, care needs to be taken when creating networks of \bayesianagents\ that use each other's knowledge to avoid data incest. 
    \item The stopping time formulation is especially useful when deploying \bayesianagents\ to detect a failure state. The quickest change detection discussed briefly can be practically used to switch optimally between a large (more accurate but more expensive) and a small (less accurate and less expensive) LLM. 
    \item Finally, our numerical experiments are conducted on real-world datasets but are limited to public textual data. Using vision-language models, the framework proposed in this paper can be used for more sophisticated bio-medical data (composed of images, audio, etc.).  
\end{enumerate}
\subsection{Limitations}
The Bayesian sensor model developed for a \bayesianagent\ requires domain knowledge to design the prompts (features) that are input into the LLM, producing a low-dimensional output; this knowledge can be obtained from a larger LLM. However, the Bayesian revealed preference framework becomes computationally infeasible in scenarios with large state and action spaces. Our formulation of stopping time and the structural results presented are focused solely on detecting a single state, specifically for identifying a failure state. While our framework is Bayesian inference centric, \bayesianagents\ in practice can undertake more complex tasks, indicating a need to extend this work to more general settings. \change{In the Bayesian revealed preferences framework considered in this paper, access to the true action posterior is assumed; however, in practice, only empirical estimates may be available. Additionally, the applications discussed in this work—finance, online content moderation, and product analysis—may limit the generalizability of the framework to other domains.}

 \subsection{Research Directions and Future Work}
 With technological innovation in LLMs, many players in the market have launched their own LLM agents. Such agents often have more functionality than those modeled in this paper, and hence, controlling the behavior of these agents would be more challenging. Further, these agents will inevitably interact with each other through the content they generate or on online platforms. Such interaction would lead to a Bayesian update in their beliefs either explicitly by design as modeled in the paper or implicitly as these agents improve. It is therefore important to consider the models presented in this study in future and related research. Given the rise of \bayesianagents, we find the following four research directions particularly exciting.

\subsubsection{Applicability in other paradigms}
Exploring the methods presented in this work within different LLM paradigms—such as retrieval-augmented generation (RAG), planning, and fine-tuned agents—would be a compelling area of study. RAG involves retrieving relevant information from a database to provide context for an LLM, enhancing the accuracy and relevance of its responses. The models of \bayesianagents\ discussed in this paper could be effectively applied in scenarios where multiple agents operate over a knowledge graph to collaboratively retrieve pertinent information~\cite{inforetrieval}, particularly when there is a prior defined over the knowledge graph. Additionally, in the context of fine-tuned agents, where agents continuously learn from a private dataset, the framework outlined in this paper can be adapted to accommodate evolving \bayesianagents\ and heterogeneous agents.

\subsubsection{\change{Application in Education and Healthcare}}
\change{
The framework proposed in this paper has broader applicability than content moderation, finance, and opinion mining. We briefly discuss the applications of social learning of LLMAs in education and healthcare. In education, there are many potential use cases of LLMs, including customized evaluation, interactive tutorials, and LLM-assisted group discussions 
}~\cite{chatgpteducation}. \change{Therefore, a sequence of LLMAs can be used to analyze different contents, and studying the control of LLMAs is important to reduce the bias, for e.g., while grading students. In healthcare, LLMAs are already used for patient interaction to provide instructions for self-care, and to route the patient to the appropriate healthcare provider. However, multiple LLMAs can be used to analyze a patient's medical history and provide useful feedback to the doctor. Recently, authors in}~\cite{qiu_llm-based_2024}\change{ have further discussed various opportunities for LLMAs in the healthcare domain.}

\subsubsection{Search for a unified interpretable models}
The interpretable models presented in this paper serve as blackbox models for the \bayesianagents. However, significant research has focused on glassbox (or transparent) models that leverage mechanistic interpretability or utilize explainable features for transformers, the neural architecture underlying LLMs. Integrating these two approaches into a unified model would be advantageous, allowing for explanations of both the \bayesianagent\ interactions with their external environment and the reasoning behind their behavior. Such a framework could be a valuable tool for analyzing various challenges related to the real-world deployment of \bayesianagents, including reliability and safety.

\subsubsection{Network of \bayesianagents\ and Data Fusion}
The \bayesianagents\ in this study are static, homogeneous, and arranged in a line graph. \change{We have considered a homogenous society of LLM agents for notational convenience and for ease of analysis. Our framework
allows for different LLM agents can have different observation matrices, this can model LLM agents who have different LLMs, prompts or pre-training data.  Additionally, LLMAs could possess asymmetric private observations and be fine-tuned on individualized datasets, creating a network akin to a distributed mixture of millions of experts}~\cite{he2024mixturemillionexperts}. \change{Another direct extension of our framework accommodates LLM agents with different observation and action spaces, provided that each agent is aware of the observation and action spaces of the preceding agent. A generalized network of LLM agents could also be utilized as Bayesian sensors for distributed state estimation. Furthermore, societies of heterogeneous LLM agents communicating with one another warrant further investigation.} These \bayesianagent\ sensors could then perform data fusion using standard techniques, with careful measures to prevent data incest. \change{This paper has looked at stochastic control methods for controlling herding in a sequence of interacting LLMAs. We have formulated two non-standard POMDPs for the centrally controlled and incentivized case. Further to derive structural results, we have looked at the extreme case where there are only two possible actions. These structural results naturally lead to an efficient distributed policy gradient algorithm for searching the optimal policy. It would be interesting to see if newer results from sociology can be used to analyze and control more sophisticated behavior of LLMAs when interacting with each other in heterogeneous societies. }

\subsubsection{Human In The Loop with \bayesianagents}
In applications such as finance, healthcare, and content moderation, \bayesianagents\ frequently interact with humans to receive feedback that can greatly enhance task performance~\cite{humanaicollab}. This necessitates advanced models that account for human-agent interactions to better align cost functions, communication protocols, and sensing mechanisms. As these \bayesianagents\ become more widely deployed across various applications, they will form natural networks with humans. Investigating how these \bayesianagents\ can share private information both reliably and securely to enhance usability and efficiency presents a compelling and open research challenge.

\bibliographystyle{ieeetr}
\bibliography{main}

 \appendix
\section{Proofs}

\subsection{Proof of Theorem 1}
{\it Proof of the necessity of NIAS and NIAC:}	
The proof follows closely from proof of~\cite{pattanayak}, and is reproduced for reader's aid. Similar proofs can also be found in ~\cite{krishnamurthy_partially_2016,kaplinanddeen}.
 \begin{enumerate}
 	\item NIAS~(\ref{eq:nias}): For environment $\environment\in\environmentspace$, define the subset $\obsset_{\action}\subseteq\obsset$ so that for any observation $\observation\in\obsset_{\action}$, given posterior probability mass function (pmf) $\probabilitymeasure_{\dpiter}(\statevar|\observation)$, the choice of action is $\action$ \eqref{eq:actionutilitymaxm} is maximum. Define the revealed posterior pmf given action $\action$ as $\probabilitymeasure_{\dpiter}(\statevar|\action)$. The revealed posterior pmf is a stochastically garbled version of the actual posterior pmf $\probabilitymeasure_{\dpiter}(\statevar|\observation)$, that is,
 	\begin{equation}\label{eqn:revpos}
 	  \probabilitymeasure(\statevar|\action) = \sum_{\observation\in\obsset} \frac{p_{\dpiter}(\statevar,\observation,\action)}{p_{\dpiter}(\action)} = \sum_{\observation\in\obsset} p_{\dpiter}(\observation|\action)p_{\dpiter}(\statevar|\observation).
 	\end{equation} 
 	Since the optimal action is $a$ for all $\observation\in\obsset_{\action}$, \eqref{eq:actionutilitymaxm} implies:

 	\begin{align*}
 	&\hspace{-0.4cm}\implies \sum_{\observation\in \obsset_{\action} } p_{\dpiter}(\observation|\action) \sum_{\statevar\in \statespace} p_{\dpiter}(\statevar|\observation) (\utility_{\dpiter}(\statevar,\actionalt)-\utility_{\dpiter})\leq 0 \\
    & \hspace{-0.4cm}\implies \sum_{\observation\in \obsset} p_{\dpiter}(\observation|\action) \sum_{\statevar\in \statespace} p_{\dpiter}(\statevar|\observation) (\utility_{\dpiter}(\statevar,\actionalt)-\utility_{\dpiter})\leq 0 \\
    &\hspace{-0.4cm}\quad\quad~~(\text{since }p_{\dpiter}(\observation|\action)=0,~\forall \observation\in\obsset\backslash\obsset_{\action})\\
  	&\hspace{-0.4cm}\implies\sum_{\statevar\in \statespace} p_{\dpiter}(\statevar|\action) (\utility_{\dpiter}(\statevar,\actionalt)-\utility_{\dpiter})\leq 0~(\text{from }\eqref{eqn:revpos}).
 	\end{align*}  
 	The last equation is the NIAS inequality~\eqref{eq:nias}. 
  
 	\item NIAC~(\ref{eq:niac}): Let $\runcostinst_{\dpiter}=\infoacquisitioncost(\attfunagent{\dpiter})>0$, where $\infoacquisitioncost(\cdot)$ denotes the information acquisition cost of the collection of agents $\dpset$. Also, let $J(\attfunagent{\dpiter},\utility_{\dpiter})$ denote the expected utility of the \ribum\ in environment $\dpiter$  given attention strategy $\attfunagent{\dpiter}$ (first term in RHS of \eqref{eq:observationoptimization}). Here, the expectation is taken wrt both the state $\statevar$ and observation $\observation$. It can be verified that $J(\cdot,\utility_\dpiter)$ is convex in the first argument. Finally, for the environment $\dpiter$ , we define the revealed attention strategy $\attfunagent{\dpiter}'$ over the set of actions $\actionspace$ as $\attfunagent{\dpiter}'(\action|\statevar) = \probabilitymeasure_{\dpiter}(\action|\statevar),~\forall\action\in\actionspace$, where the variable $\probabilitymeasure_{\dpiter}(\action|\statevar)$ is obtained from the dataset $\dataset$. Clearly, the revealed attention strategy is a stochastically garbled version of the true attention strategy since
 	\begin{equation}\label{eqn:revattfun}
 	    \attfunagent{\dpiter}'(\action|\statevar) = \probabilitymeasure_{\dpiter}(\action|\statevar) = \sum_{\observation\in\obsset}p_{\dpiter}(\action|\observation)\attfunagent{\dpiter}(\observation|\statevar).
 	\end{equation}
 	From Blackwell dominance~\cite{blackwelldominance} and the convexity of the expected utility functional $J(\cdot,\utility_\dpiter)$,  it follows that: 
 	\begin{equation}
 	    J(\attfunsymb_\dpiter',\utility_\dpitertwo) \leq J(\attfunsymb_\dpiter,\utility_\dpitertwo), \label{eqn:ineq_BLK}
 	\end{equation}
 	when $\observationmatrix_{\dpiter}$ Blackwell dominates $\observationmatrix_{\dpiter}'$. The above relationship holds with equality if $\dpiter=\environmentalt$ (this is due to NIAS \eqref{eq:nias}).

We now turn to condition \eqref{eq:observationoptimization} for optimality of attention strategy.

The following inequalities hold for any pair of agents $\environmentalt\neq\dpiter$:
    \begin{align}
        &J(\attfunsymb_\dpiter',\utility_{\dpiter}) - \runcostinst_{\dpiter}\overset{\eqref{eqn:ineq_BLK}}{=} J(\attfunagent{\dpiter},\utility_{\dpiter}) - \runcostinst_{\dpiter}\nonumber\\
        \overset{\eqref{eq:observationoptimization}}{\geq}~&J(\attfunagent{\dpitertwo},\utility_{\dpiter}) - \runcostinst_{\dpitertwo} \overset{\eqref{eqn:ineq_BLK}}{\geq}~J(\attfunagent{\dpitertwo}',\utility_{\dpiter}) - \runcostinst_{\dpitertwo}\label{eqn:NIAC_proof}.
    \end{align}
    This is precisely the NIAC inequality~\eqref{eq:niac}.
\end{enumerate}

\noindent{\em Proof for sufficiency of NIAS and NIAC:}
Let $\{\utility_\dpiter,\runcostinst_\dpiter\}_{\dpiter=1}^\numenvironments$ denote a feasible solution to the NIAS and NIAC inequalities of Theorem~\ref{th:nscribum}. To prove sufficiency, we construct an \ribum\ tuple as a function of dataset $\dataset$ and the feasible solution that satisfies the optimality conditions~\eqref{eq:actionutilitymaxm},\eqref{eq:observationoptimization} for \ribum\ \eqref{eq:ribum}.

Consider the following \ribum\ tuple:
 \begin{align}
     &(\dpset,\statespace,\obsset=\actionspace,\actionspace,\prior,\infoacquisitioncost,\{\probabilitymeasure_\dpiter(\action|\statevar),\utility_\dpiter,\dpiter\in\dpset\}),\text{ where}\nonumber\\
     &\infoacquisitioncost(\probabilitymeasure_(\action|\statevar))  = \max_{\dpiter\in\dpset} \runcostinst_\dpiter + J(\probabilitymeasure_(\action|\statevar),\utility_\dpiter) - J(\probabilitymeasure_\dpiter(\action|\statevar),\utility_\dpiter).\label{eqn:construct_cost}
 \end{align}

 In \eqref{eqn:construct_cost}, $C(\cdot)$ is a convex cost since it is a point-wise maximum of monotone convex functions. Further, since NIAC is satisfied, \eqref{eqn:construct_cost} implies $\infoacquisitioncost(\attfunagent{\dpiter})=\runcostinst_{\dpiter}$. It only remains to show that inequalities ~\eqref{eq:actionutilitymaxm} and \eqref{eq:observationoptimization} are satisfied for all environments in $\dpset$.
 \begin{enumerate}
 \item {\em NIAS implies \eqref{eq:actionutilitymaxm} holds.} This is straightforward to show since the observation and action sets are identical.
 From NIAS \eqref{eq:nias}, we know that for any environment $\dpiter\in\dpset$, actions $\action,\actionalt\in\actionspace,\action\neq\actionalt$, the following inequalities hold.
 \begin{align*}
     &\sum_{\statevar}\probabilitymeasure_{\dpiter}(\statevar|\action)( \utility_{\dpiter}(\statevar,\actionalt) - \utility_{\dpiter}) \leq 0\\
     \implies& \sum_{\statevar}\probabilitymeasure_{\dpiter}(\statevar|\observation=\action)( \utility_{\dpiter}(\statevar,\actionalt) - \utility_{\dpiter}) \leq 0\\
     \implies& \action \in \argmax_{\actionalt\in\actionspace}  \sum_{\statevar}p_{\dpiter}(\statevar|\observation) \utility_{\dpiter}(\statevar,\actionalt) \implies \eqref{eq:actionutilitymaxm}.
 \end{align*}
 \item {\em Information Acquisition Cost \eqref{eqn:construct_cost} implies \eqref{eq:observationoptimization} holds.} Fix environment $\dpitertwo\in\dpset$. Then, for any attention strategy $\probabilitymeasure(\action|\statevar)$, the following inequalities hold.
 \begin{align*}
     &\infoacquisitioncost(\probabilitymeasure(\action|\statevar))  =  \max_{\dpiter\in\dpset}~ \runcostinst_{\dpiter} + J(\probabilitymeasure(\action|\statevar),\utility_{\dpiter}) - J(\probabilitymeasure_\dpiter(\action|\statevar),\utility_{\dpiter}) \\
     &\hspace{-0.4cm}\implies J(\probabilitymeasure_\dpitertwo(\action|\statevar))  -\runcostinst_\dpitertwo \geq J(\probabilitymeasure(\action|\statevar)) - C(\probabilitymeasure(\action|\statevar)),~\forall~\probabilitymeasure(\action|\statevar)\\
     &\hspace{-0.4cm}\implies \probabilitymeasure_\dpiter(\action|\statevar)\in \argmax_{\probabilitymeasure(\action|\statevar)} J(\probabilitymeasure(\action|\statevar),\utility_{\dpiter}) - \infoacquisitioncost(\probabilitymeasure(\action|\statevar))\\&=\eqref{eq:observationoptimization}.
 \end{align*}
 \end{enumerate}

\subsection{Derivation of social learning filter}\label{app:sociallearningfilter}
Let the posterior as $\posterior_\timeindex(\stateidx) = \probabilitymeasure(\statevar_\timeindex=\stateidx|\action_1,\dots,\action_{\timeindex})$.
Let $\normalizationfactor(\posterior_{\timeindex-1},\action_{\timeindex}) = \sum_{\stateidx}\sum_{\observation} \probabilitymeasure(\action_\timeindex|\observation_\timeindex=\observation,\posterior_{\timeindex-1})\probabilitymeasure(\observation_\timeindex|\statevar_{\timeindex}=\stateidx)\sum_{\stateidxalt}\transitionmatrix_{\stateidx,\stateidxalt} \posterior_{\timeindex-1}(\stateidxalt)$ be the normalization factor.

\begin{align*}
    &\posterior_\timeindex(\stateidx)  \\&= \frac{1}{\normalizationfactor(\posterior_{\timeindex-1},\action_\timeindex)}\probabilitymeasure(\action_\timeindex|\statevar_{\timeindex}=\stateidx,\action_1,\dots,\action_{\timeindex-1})\\&\sum_{\stateidxalt}\transitionmatrix_{\stateidx,\stateidxalt} \probabilitymeasure(\statevar_{\timeindex-1}=\stateidxalt|\action_1,\dots,\action_{\timeindex-1})\\
    &=\frac{1}{\normalizationfactor(\posterior_{\timeindex-1},\action_\timeindex)}\sum_{\observation}\probabilitymeasure(\action_\timeindex|\observation_{\timeindex}=\observation,\posterior_{\timeindex-1})\probabilitymeasure(\observation_\timeindex=\observation|\statevar_\timeindex=\stateidx)\\&\sum_{\stateidxalt}\transitionmatrix_{\stateidx,\stateidxalt} \posterior_{\timeindex-1}(\stateidxalt).
\end{align*}
which completes the derivation.
  \subsection{Proof for Theorem 2}
    \begin{proof}
    Define $\priorratio_\timeindex(\stateidx,\stateidxalt) = \log(\prior(\stateidx)/\prior(\stateidxalt)), \stateidx,\stateidxalt \in \statespace$. From~\eqref{eq:priorupdate} we have, $\priorratio_{\timeindex+1}(\stateidx,\stateidxalt) = \priorratio_{\timeindex}(\stateidx,\stateidxalt) + \actionratio_{\timeindex}(\stateidx,\stateidxalt)$ where $\actionratio_{\timeindex}(\stateidx,\stateidxalt) = \log(\probabilitymeasure(\action_\timeindex|\statevar=\stateidx,\prior_\timeindex)/\probabilitymeasure(\action_\timeindex|\statevar=\stateidxalt,\prior_\timeindex)$. 

    The probability of the actions given the state and prior can be written as, 
    \begin{align*}        \probabilitymeasure(\action|\statevar,\prior) = \sum_{\observation\in \observationspace} \prod_{\actionalt \in \actionspace \setminus \action} \indicator(\cost^\prime_\action \observationmatrix_\observation  \prior \leq \cost^\prime_{\actionalt} \observationmatrix_\observation  \prior) \observationmatrix_{\observation,\statevar}.
    \end{align*}
    Let $\Tilde{\observationspace}_\timeindex \subseteq \observationspace$ be a subset of the observation space for which the action $\action_\timeindex$ is suboptimal with respect to all other actions, i.e., $\prod_{\actionalt \in \actionspace \setminus \action_\timeindex}  \indicator(\cost^\prime_{\action_\timeindex} \observationmatrix_\observation  \prior > \cost^\prime_{\actionalt} \observationmatrix_\observation  \prior) = 1 \ \forall \observation \in \observationspace$. When information cascade (Def.~\ref{def:infocascade}) occurs, this set should be empty since no matter what the observation, the action $\timeindex$ should be optimal according to~\eqref{eq:action}.
    Also, rewriting $\actionratio_{\timeindex}(\stateidx,\stateidxalt)$,
    \begin{align*}
        \actionratio_{\timeindex}(\stateidx,\stateidxalt) = \log \left(\frac{1-\sum_{\observation\in\Tilde{\observationspace}_\timeindex}\observationmatrix_{\observation,\stateidx}}{1-\sum_{\observation\in\Tilde{\observationspace_\timeindex}}\observationmatrix_{\observation,\stateidxalt}}\right).
    \end{align*}
    Therefore when an information cascade occurs, $\actionratio_\timeindex(\stateidx,\stateidxalt) = 0, \ \forall \stateidx,\stateidxalt\in\statespace$ (Due to $\Tilde{\observationspace_\timeindex}$ being an empty set). Also if $\Tilde{\observationspace_\timeindex}$ is nonempty, then $\actionratio_\timeindex(\stateidx,\stateidxalt)>\constant$, where $\constant$ is a positive constant.   

    Let $\actionfiltration_\timeindex = \{\sigma(\action_1,\action_2,\dots,\action_\timeindex)\}$ denote the natural filtration, where $\sigma$ is the operator which generates the corresponding sigma field.

    $\prior_\timeindex(\stateidx) = \probabilitymeasure(\statevar=\stateidx|\action_1,\dots,\action_\timeindex) = \expectation[\indicator(\statevar = \stateidx)\vert\actionfiltration_\timeindex]$ is a martingale adapted to $\actionfiltration_\timeindex$ for all $\stateidx \in \statespace$. This follows by the application of smoothing property of conditional expectation, $\expectation[\prior_{\timeindex+1}(\stateidx)\mid \actionfiltration_\timeindex] = \expectation\left[\expectation[\indicator(\statevar = \stateidx)\mid\actionfiltration_{\timeindex+1}]\right] = \expectation[\indicator(\statevar=\stateidx)\mid\actionfiltration_\timeindex]$.

    Therefore, by the  martingale convergence theorem, there exists a random variable $\prior_\infty$ such that, $
    \prior_\timeindex \to \prior_\infty \ \text{w.p.} 1$. Therefore $\priorratio_\timeindex(\stateidx,\stateidxalt) \to \priorratio_\infty(\stateidx,\stateidxalt)$ w.p. 1., which implies there exists $\Tilde{\timeindex}$ such that $\forall \timeindex \geq \Tilde{\timeindex}$, $\vert \priorratio_\infty(\stateidx,\stateidxalt) - \priorratio_\timeindex(\stateidx,\stateidxalt) \vert \leq \constant/3$ and so,
    \begin{align}\label{eq:contra1}
        \vert \priorratio_{\timeindex+1}(\stateidx,\stateidxalt) - \priorratio_\timeindex(\stateidx,\stateidxalt) \vert \leq 2\constant/3, \forall \timeindex \geq \Tilde{\timeindex}.
    \end{align}
    We now prove the theorem by contradiction. Suppose a cascade does not occur, then for at least one pair $\stateidx \neq \stateidxalt, \stateidx,\stateidxalt \in \statespace$, $\probabilitymeasure(\action|\statevar=\stateidx,\prior)$ is different than $\probabilitymeasure(\action|\statevar=\stateidxalt,\prior)$. This would imply that the set $\Tilde{\observationspace}_\timeindex$ is nonempty and therefore, 
    \begin{align}\label{eq:contra2}
        |\priorratio_{\timeindex}(\stateidx,\stateidxalt)| = \vert \priorratio_{\timeindex+1}(\stateidx,\stateidxalt) - \priorratio_\timeindex(\stateidx,\stateidxalt) \vert \geq \constant.
    \end{align}
    ~\eqref{eq:contra1} and \eqref{eq:contra2} contradict each other. Therefore $\probabilitymeasure(\action|\statevar=\stateidx,\prior)$ is same for all $\stateidx\in\statespace$ and hence according to~\eqref{eq:priorupdate} information cascade occurs at time $\Tilde{\timeindex}$.
    \end{proof}
    \subsection{Proof for Theorem 3}
    \begin{proof}
    We prove the Theorem by showing that it satisfies the conditions of Theorem 12.3.4 of~\cite{krishnamurthy_partially_2016}. A more general proof can be found in~\cite{krishnamurthy_partially_2016,quickestchangedetectionitit}.

In order to verify the assumptions of Theorem 12.3.4 of~\cite{krishnamurthy_partially_2016}, we need to define first-order stochastic dominance (FOSD) and a submodular function. 
    We first define a Monotone Likelihood Ratio (MLR) ordering on a line and then define a submodular function with respect to this MLR ordering. We only need to consider the following lines, 
    \begin{align*}
\lines(\indicatorstate_\stateidx,\bar{\prior}) = \{\prior \in \probspace(\statedim)  : \prior = (1-\epsilon)\bar{\prior} + \epsilon \indicatorstate_\stateidx, 0 \leq \epsilon\leq 1\}\\, \bar{\prior} \in \mathcal{H}_\stateidx,
    \end{align*}
    where the state index is only between the extreme states, $\stateidx \in \{1,\statedim \}$ and,
     \begin{align*}
        \mathcal{H}_\stateidx = \{\prior \in \probspace(\statedim):\policy(\stateidx) = 0\}.
    \end{align*}
    To define the MLR ordering on a line, we first define the MLR ratio with respect to belief space, 
    \begin{definition}
       \textbf{(Monotone Likelihood Ratio (MLR) Order)} Let $\prior_1$,$\prior_2 \in \probspace(\statedim)$, then $\prior_1$ dominates  $\prior_2$ with respect to the MLR order ($\prior_1 \geq_r \prior_2$) if,
        \begin{align*}
\prior_1(\stateidx)\prior_2(\stateidxalt) \leq \prior_1(\stateidxalt)\prior_2(\stateidx), \ \ \stateidx < \stateidxalt, \stateidx,\stateidxalt \in \{1,\dots,\statedim\}.
        \end{align*}
    \end{definition}
    The following definition is for the MLR ordering the lines $\lines_{\indicatorstate_\stateidx}, \ \stateidx \in \{1,\statedim \}$ 
    \begin{definition}
        (\textbf{MLR Ordering on Line  $\geq_{\linesymb_\stateidx}$}) $\policy_1$ is greater than $\policy_2$ with respect to the MLR ordering on the line $\lines(\indicatorstate_\stateidx,\prior), \stateidx\in\{1,\statedim\}$ ($\prior_1 \geq_{\linesymb_\stateidx}\prior_2$), if $\prior_1,\prior_2 \in \lines(\indicatorstate_\stateidx,\bar{\prior})$ for some $\bar{\prior}$ and $\prior_1 \geq_r \prior_2$. 
    \end{definition}
    Finally, we are ready to define a submodular function on a line,
    \begin{definition}
        \textbf{Submodular Function on Line} For $\stateidx\in\{1,\statedim\}$, a function
    $\phi: \lines(\indicatorstate_\stateidx,\bar{\prior})\times \decisionspace \to \R$ is submodular if $\phi(\prior,\decision) - \phi(\prior,\bar{\decision}) \leq \phi(\bar{\prior},\decision) - \phi(\bar{\prior},\bar{\decision})$, for $\bar{\decision}\leq \decision, \bar{\prior} \leq_{\linesymb_\stateidx} \prior$.
    \end{definition}

  The following is used extensively to compare two beliefs and is a weaker condition than MLR ordering, 
    \begin{definition}
       (\textbf{First Order Stochastic Dominance (FOSD)}) Let $\prior_1$,$\prior_2 \in \probspace(\statedim)$, then $\prior_1$ first order stochastically dominates  $\prior_2$ ($\prior_1 \geq_s \prior_2$) if $\sum_{\stateidx=\stateidxalt}^{\statedim} \prior_1(\stateidx) \geq \sum_{\stateidx=\stateidxalt}^{\statedim} \prior_2(\stateidx) \ \forall \ \stateidxalt\in\statespace$. 
    \end{definition}

The stopping time problem with the social welfare cost of~\eqref{eq:socialwelfarecost} can be decomposed into two cost terms, each corresponding to the cost terms of the stopping time problem for partially observed Markov decision processes of Theorem 12.3.4 of~\cite{krishnamurthy_partially_2016}. 
\begin{align}\label{eq:bigcost}
\begin{split}
    \bigcost(\prior,1) &= \frac{1}{1-\discountfactor} \min_{\action} \cost_\action^{\prime} \prior, \\
    \bigcost(\prior,2) &= \sum_{\observation\in\observationspace} \cost_{\observation}^\prime \observationmatrix_\observation \prior + (\delaycost + (1-\discountfactor)\errorcost)\indicatorstate_1\prior - (1-\discountfactor)\errorcost.
\end{split}
\end{align}
$\bigcost(\prior,1)$ is the expected cost after the decision to herd has been made. Similarly, the first term in $\bigcost(\prior,2)$ is the expected cost when revealing private observations. The rest of the terms come from transforming the value function by the delay penalty costs~\cite{quickestchangedetectionitit}. 

   We now state the main assumptions of Theorem~12.3.4 of~\cite{krishnamurthy_partially_2016}, which are required for this to hold for the structural result of~\eqref{eq:optpolicy},
    \begin{enumerate}
        \item \textbf{(C)} $\prior_1 \geq_s \prior_2$ implies $\bigcost(\prior_1,\action) \leq \bigcost(\prior_2,\action)$.
        \item \textbf{(F1)} $\observationmatrix$ is totally positive of order 2 (TP2).
        \item \textbf{(F2)} $\markovkernel$ is totally positive of order 2 (TP2).
        \item (\textbf{S}) $\bigcost(\prior,\decision)$ is submodular on $[\lines(\indicatorstate_\statedim,\bar{\prior}),\geq_{\linesymb_\statedim}]$ and  $[\lines(\indicatorstate_1,\bar{\prior}),\geq_{\linesymb_1}]$.
    \end{enumerate}
    F1 follows from S4, and F2 follows from the fact that we only consider an identity transition matrix $\markovkernel$. And since the costs are linear, (C) follows by applying the definition of FOSD on equation~\eqref{eq:bigcost} and using assumption (S1). (S) follows from the definition of MLR ordering on the line, using the fact that $\observationmatrix$ is TP2 in the first term of $\bigcost(\prior,2)$ and using (S2) and (S3). Hence, the assumptions are verified, and the structural result is proved.
    \end{proof}
    \subsection{Proof of Optimal Stopping Theorem 4}
    This proof follows closely from the proof in \cite{bhattcontrolled2020}.
    From Lemma 1 of \cite{bhattcontrolled2020}, the value function can be expressed as, 
    \begin{align*}
        \valuefn(\prior) = \min\{0,\incentivefunction(\observation) -\fusionweight+\discountfactor\expectation\valuefn(\prior)\}.
    \end{align*}
    Further, Lemma 2 of \cite{bhattcontrolled2020} shows that the incentive function is decreasing. Denote $\valuefn_\dpiter$ to be the $\dpiter-$th  iterate of the value iteration algorithm~\cite{krishnamurthy_partially_2016} which iteratively converges to the value function $\valuefn(\prior)$. The iterates are given by,
    \begin{align*}
        \valuefn_{\dpiter+1}= \min\{0,\incentivefunction(\observation) -\fusionweight+\discountfactor\expectation\valuefn(\prior)\}.
    \end{align*}
    From the definition of first order stochastic dominance and Proposition 1 of \cite{bhattcontrolled2020} $\expectation\valuefn_\dpiter(\prior)$ is decreasing in $\prior$. 
    
    Therefore $\valuefn_\dpiter(\prior)$ is decreasing which implies $\valuefn(\prior)$ is decreasing. 
    Let $V(0)$ and $V(1)$ be the value at $\pi=[1,0]$ and $\pi=[0,1]$ respectively which makes $\expectation \valuefn(0) = \valuefn(0)$ and $\expectation\valuefn(1)=\valuefn(1)$. 
    \begin{enumerate}
        \item For $\valuefn = \incentivefunction(\observation) -\fusionweight+\discountfactor\expectation\valuefn(\prior)$, $\valuefn(0) = \frac{\incentivefunction(\indicatorstate_1)-\fusionweight}{1-\discountfactor}>0$ and $\valuefn(1) = \frac{\incentivefunction(\indicatorstate_2)-\fusionweight}{1-\discountfactor}<0$.
        \item For $\valuefn = 0$, $\valuefn(0) = \valuefn(1) = 0$
    \end{enumerate}
    Therefore, the value function decreases with a positive value of $[1,0]$ and a negative value of $[0,1]$. Therefore it must be 0 at some time. Since $\valuefn(\prior)$ is monotone in $\prior$, the set $E = \{\prior(2)|\valuefn(\prior) =  \incentivefunction(\observation) -\fusionweight+\discountfactor\expectation\valuefn(\prior)\}$ is convex. We choose a policy $\thresholdprior^*(2) = \{\hat{\thresholdprior}(2) | \hat{\thresholdprior}(2) > \prior(2) \forall\ \prior(2) \in E \}$. 
    \subsection{Proof of Submartingale Result Theorem 5}
    Consider the suboptimal policy where $\epsilon>0$,
    \[
   \hat{\incentivepolicy}(\prior)  = \begin{cases}
        \incentivefunction(\prior) - \epsilon &\text{if} \ \prior(2) \leq \thresholdprior^*(2)\\
         \incentivefunction(\prior) &\text{if} \ \prior(2) > \thresholdprior^*(2)\\
    \end{cases}.
    \]
    From Lemma 3 of \cite{bhattcontrolled2020} $\incentivefunction(\prior)$ is convex in $\prior$. We know that the public belief $\prior_\timeindex$ is a martingale from the proof of Theorem~\ref{th:herding}. Let $W_k =  \hat{\incentivepolicy}(\prior_{\timeindex-1})$. By Jensen's inequality for $\epsilon\to0$,
    \begin{align*}
    \expectation[W_{\timeindex+1}|\filtration_{\timeindex}] = \expectation[\incentivefunction(\prior_{\timeindex+1})|\filtration_{\timeindex}] \geq \incentivefunction(\expectation[\prior_{\timeindex+1}|\filtration_{\timeindex}])\\
    \geq \incentivefunction(\prior_{\timeindex}) \geq W_{\timeindex}.
\end{align*}
Therefore $W_\timeindex$ is a submartingale. Consider $H_\timeindex$ as the following sequence,
\begin{align*}
    H_\timeindex = \begin{cases}
        0 \ \text{if} \ \prior_{\timeindex-1}(2) &\leq \thresholdprior^*(2)\\
        1  \ \text{if} \ \prior_{\timeindex-1}(2) & \geq  \thresholdprior^*(2)
    \end{cases}.
\end{align*}
    Now, by properties of submartingales, $(HW)_{\timeindex}$ is a submartingale. This is exactly the incentive sequence, which is indeed a submartingale. 

    To show the second statement, we use Doob's martingale inequality on the submartingale sequence $\incentive_\timeindex$,
    \[
    { \probabilitymeasure\left[\max _{1\leq \timeindex\leq \timehorizon}\incentive_{\timeindex}\geq C\right]\leq {\frac {\expectation [{\max}(\incentive_{\timehorizon},0)]}{C}}} = {\frac {\expectation [\incentive_{\timehorizon}]}{C}}.
    \]
    Also note that $\sum_1^\timehorizon \incentive_{\timeindex} \leq \timehorizon\max _{1\leq \timeindex\leq \timehorizon}\incentive_{\timeindex}$. 
    Therefore the event $\{\sum_1^\timehorizon \incentive_{\timeindex}  \geq C\timehorizon \}\subseteq \{\timehorizon\max _{1\leq \timeindex\leq \timehorizon}\incentive_{\timeindex} \geq C\timehorizon \}$. 
    
    Hence 
    $\probabilitymeasure(\sum_1^\timehorizon \incentive_{\timeindex}  \geq C\timehorizon) \leq \probabilitymeasure(\timehorizon\max _{1\leq \timeindex\leq \timehorizon}\incentive_{\timeindex} \geq C\timehorizon)$ which along with the first inequality proves the statement.
 \appendix

    \section{Brief Experimental Details}\label{sec:expapp}
\subsection{Hyperparameters Of LLM at Inference Time}

As described in the main text, we use different LLMs.

Large Language models uses default parameters for inference to control the output. Temperature (default 1) influences randomness, with lower values producing more focused responses. Max tokens (default ~2048) limits response length. Top p (default 1) controls diversity, with lower values making the output more focused by considering only the most likely words.
\subsubsection{Mixtral-8x7B-v0.1}
We consider the maximum response tokens as 100, a temperature of 0.7, a top-p of 0.7, a top-k of 50, and a repetition penalty as 50.
\subsubsection{LLaMA-3-70b}
We consider the maximum response tokens as 200, a temperature of 0.7, a top-p of 0.7, a top-k of 50, and a repetition penalty as 80.

 We use the TogetherAI API to send the prompt to the LLM and receive the response for LLama3 and Mixtral.
\subsubsection{ChatGPT-4o}
Default parameters were used when inferencing ChatGPT using their API. 

\subsection{Dataset Description}
\subsubsection{Amazon Review Dataset}\label{app:amazon}
Amazon review data~\cite{ni-etal-2019-justifying}, including 233.1 million reviews spanning various product categories. It features detailed information such as ratings, review text, helpfulness votes, product descriptions, category information, prices, brands, and image features. It also has additional transaction metadata such as product color, size, and user-submitted images, as well as more detailed product landing page information like bullet-point descriptions and technical specifications.
\subsubsection{Measuring Hatespeech Dataset}
The Measuring Hate Speech corpus is a comprehensive dataset designed to evaluate hate speech while accounting for annotators’ perspectives~\cite{sachdeva_measuring_2022}. It includes 50,070 social media comments from platforms like YouTube, Reddit, and Twitter, labeled by 11,143 annotators from Amazon Mechanical Turk. Each comment is assessed on 10 ordinal labels such as sentiment, disrespect, and violence, which are aggregated into a continuous score using Rasch measurement theory. This approach allows for the statistical summarization of annotator disagreement and adjusts for labeling strictness, providing a nuanced measure of hate speech. The dataset also includes information on the identity group targets of each comment and annotator demographics, enabling detailed analyses of identity-related perspectives.
\subsubsection{JigsawAI Unintended Bias Dataset}
This dataset was used to train the likelihood network since we assume the actual dataset is not available during real-time inferencing.

The Civil Comments dataset, comprising approximately 2 million public comments from the now-defunct Civil Comments platform, provides a valuable resource for investigating online toxicity.  Jigsaw, the sponsor of this effort, facilitated human annotation of these comments for various toxic conversational attributes, including  "toxicity," "severetoxicity," "obscene," "threat," "insult," "identityattack," and "sexualexplicit."  Each comment's toxicity was assessed by up to 10 annotators who rated it on a scale ranging from "Not Toxic" to "Very Toxic," with the final toxicity label representing the fraction of annotators who deemed it toxic.  Furthermore, a subset of comments were labeled for identity mentions, such as gender, sexual orientation, religion, and race, to analyze the relationship between online toxicity and identity. This dataset offers significant potential for developing and evaluating models aimed at identifying and mitigating harmful online interactions. The dataset can be accessed here kaggle.com/c/jigsaw-unintended-bias-in-toxicity-classification/data.
\subsubsection{FNSPID}
FNSPID (Financial News and Stock Price Integration Dataset) is a large-scale dataset designed for financial market prediction. It integrates both numerical and textual data, comprising 29.7 million stock prices and 15.7 million financial news articles for 4,775 S\&P500 companies from 1999 to 2023~\cite{10.1145/3637528.3671629}.  Sourced from four major stock market news websites, FNSPID includes sentiment scores derived from the news articles, offering a unique resource for researchers to investigate the impact of news sentiment on market trends. The metric used in Figure~\ref{fig:interpretablefigfinance} is given by,
\[
\frac{C_{\timeindex+1}}{C_{\timeindex}} - \frac{C_{\timeindex}}{C_{\timeindex-1}},
\]
where $C_{\timeindex}$ is the close price at time $\timeindex$. The close price is the stock price at the end of the trading day.
    \subsection{System Prompts}
        \subsubsection{Financial Analysis Task}
        For the Financial Analysis Task of Example~\ref{ex:financial}.
        
       `` Analyze the article and answer the following questions based on the content:
        
Are there indications that recent or upcoming policy decisions could support market growth? (Yes/No)

Do statements from central banks suggest optimism about the economic outlook? (Yes/No)

Are there emerging trends or patterns that suggest a shift in market sentiment? (Yes/No)

Is there evidence of key technical levels acting as support for major indices? (Yes/No)

Are certain sectors or industries showing stronger performance compared to others? (Yes/No)

Do shifts in investor interest suggest a move toward specific sectors, such as technology or energy? (Yes/No)

Do recent economic data releases (e.g., employment, inflation, consumer sentiment) point toward growth? (Yes/No)

Are any indicators flashing signals that typically correlate with significant market moves (e.g., yield curves, commodity prices)? (Yes/No)

Is there evidence of a “risk-on” approach among investors? (Yes/No)

Do recent market movements suggest increased interest in safe-haven assets like gold or bonds? (Yes/No)

Are there global or geopolitical events mentioned that could influence market volatility? (Yes/No)

Could changes in international markets or currencies impact domestic market trends? (Yes/No)

Are recent corporate earnings or business announcements likely to influence market sentiment? (Yes/No)

Do specific companies or sectors appear to be driving recent market gains? (Yes/No) \textit{article}''

where \textit{article} contains the HTML page of the online news article. 

\subsubsection{Hate speech Classification Task}
    For the Hate speech Classification Task: 
    
    ``[INST]\\  Return a JSON with the following format for the given text:\\ \{`is\_insulting': Bool,\\`is\_dehumanizing':Bool,\\`is\_humiliating':Bool,\\`promotes\_violence':Bool,\\`promotes\_genocide':Bool,\\`is\_respectful':Bool\}\\ 
    Text: \{\textit{comment}\}[/INST]''
    
    where ``\textit{comment}'' contains the comment observation which needs to be analyzed. 

\subsubsection{Product Quality Identification Task}
    For the Product Quality Identification Task: 

    ``Analyze the following product review and provide a summary of the key points:
    
- Does the review mention any specific problems or defects with the product?

- Does the review mention any positive attributes regarding the product's durability or reliability?

- Does the review indicate that the product meets or exceeds the user's expectations?

- Would the reviewer recommend this product to others? \textit{review}
''

where \textit{review} contains the review which is being analyzed.
\subsubsection{Response}
The responses of the LLM includes the JSON response at the start and an explanation for the corresponding mapping. We truncated the output to include the JSON and got a discrete low-dimensional observation from the textual comment. 
\subsection{Observation Space}
  \subsubsection{Reducing the observation space}
  We reduce the observation space for the hate-speech classification task: 
  
Although the LLM output for text observation is of cardinality $2^6 = 64$, we reduce by considering the following order: respectful < insulting < dehumanizing < humiliating < violence < genocide~\cite{sachdeva_measuring_2022}. The binary map $\psi(z):\{0,1\}^6 \to 6 = \max\{ \stateidxalt:  \text{s.t.} \ z[\stateidxalt] = 1\} $ takes the observation as the highest severity present in the binary observation. 
\subsubsection{Likelihood Neural Network (Restricted Boltzmann Machine) } 
We use a subset of the labeled dataset to obtain a likelihood distribution.
We use restricted Boltzmann Machines (RBMs) to approximate the likelihood function $\probabilitymeasure(\observation|\statevar)$; we train a conditional RBM on observations obtained from different states (as defined in the main text). Each RBM has $|\observationspace|$ visible units and $4$ hidden units. We train the RBM using contrastive divergence for $100$ epochs and generate $1000$ samples using Gibbs sampling with $1000$ iterations. We obtain the approximate probabilities by empirically counting the samples. 
\appendix

\begin{IEEEbiography}[{\includegraphics[width=1in,height=1.33in,clip,keepaspectratio]{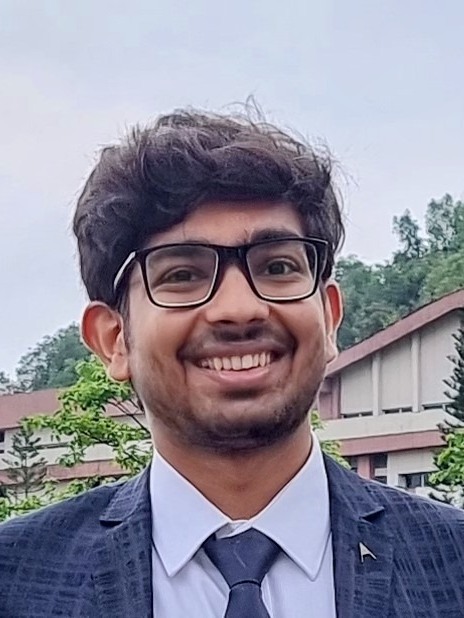}}]{Adit Jain}  received his Bachelor of Technology in Electronics and Communication Engineering from the Indian Institute of Technology, Guwahati in 2022 where he received the Institute Silver Medal. He is currently a graduate student at Cornell University. He has internship experience at Adobe Research and Goldman Sachs.

His research interests include structural results for episodic reinforcement learning, online learning, and high-dimensional linear bandits. His research is focused on applications in federated learning, large language model agents, and active learning. He is a recipient of the Data Science Fellowship graciously offered by the Cornell Center for Social Sciences.
\end{IEEEbiography}

\begin{IEEEbiography}[{\includegraphics[width=1in,height=1.33in,clip,keepaspectratio]{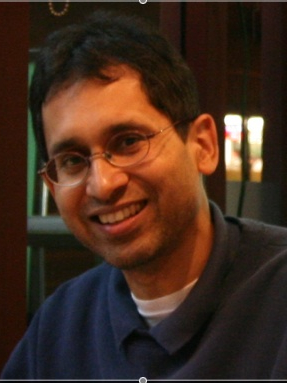}}]{Vikram Krishnamurthy}  (F'05) received the Ph.D. degree from the Australian National University
in 1992. He is  a professor in the School of Electrical \& Computer Engineering,
Cornell University. 

From 2002-2016 he was a Professor and Canada Research Chair
at the University of British Columbia, Canada. His research interests include statistical signal processing and stochastic control in social networks and adaptive sensing. He served as a Distinguished Lecturer for the IEEE Signal Processing Society and Editor-in-Chief of the \textsc{IEEE Journal on Selected Topics in Signal Processing}.

In 2013, he was awarded an Honorary Doctorate from KTH (Royal Institute of Technology), Sweden. He is the author of two books {\em Partially Observed Markov Decision Processes} and {\em Dynamics of Engineered Artificial Membranes and Biosensors}, published by Cambridge University Press in 2016 and 2018, respectively.
\end{IEEEbiography}

\EOD

\end{document}

%% file: input.tex
\newtheorem{theorem}{Theorem}
\newtheorem{definition}{Definition}
\newtheorem{corollary}{Corollary}
\newcommand{\expectation}{\mathds{E}}
\newcommand{\diag}{\mathsf{diag}}

\newcommand{\probabilitymeasure}{\mathds{P}}
\newcommand{\indicator}{\mathds{1}}
\newcommand{\R}{\mathbb{R}}
\newcommand{\argmin}{\arg\min}
\newcommand{\timeindex}{k}

\newcommand{\timeindexalt}{m}
\newcommand{\statespace}{\mathcal{X}}
\newcommand{\statedim}{X}
\newcommand{\statevar}{x}
\newcommand{\stateidx}{i}

\newcommand{\observationspace}{\mathcal{Y}}
\newcommand{\llmobservation}{z}
\newcommand{\observationmatrix}{B}
\newcommand{\observation}{y}

\newcommand{\prior}{\pi}
\newcommand{\priorupdate}{\mathcal{T}}
\newcommand{\actionprob}{R}
\newcommand{\posterior}{\eta}
\newcommand{\action}{u}
\newcommand{\decision}{a}
\newcommand{\actionalt}{\Tilde{u}}

\newcommand{\actionspace}{\mathcal{U}}
\newcommand{\cost}{c}

\newcommand{\bayesianagenttext}{large language model agent}
\newcommand{\bayesianagentstext}{large language model agents}
\newcommand{\bayesianagent}{LLMA}

\newcommand{\markovkernel}{P}
\newcommand{\herdtime}{K}
\newcommand{\policy}{\mu}
\newcommand{\avgcost}{J}
\newcommand{\discountfactor}{\rho}
\newcommand{\stoppingtime}{\tau}
\newcommand{\delaycost}{d}
\newcommand{\errorcost}{\Upsilon}
\newcommand{\indicatorstate}{e}
\newcommand{\filtration}{\mathcal{F}}
\newcommand{\actionfiltration}{\mathcal{G}}
\newcommand{\fusionfiltration}{\mathcal{H}}
\newcommand{\probspace}{\mathcal{P}}
\newcommand{\decisionspace}{\mathcal{D}}
\newcommand{\region}{\mathcal{S}}
\newcommand{\timehorizon}{T}
\newcommand{\entropyfunction}{\mathcal{E}}
\newcommand{\threshold}{\theta}
\newcommand{\lines}{\mathcal{L}}
\newcommand{\linesymb}{L}
\newcommand{\bigcost}{C}
\newcommand{\priorratio}{\Lambda}
\newcommand{\stateidxalt}{j}
\newcommand{\actionratio}{\Gamma}
\newcommand{\constant}{\kappa}
\newcommand{\bayesianagents}{LLMAs}
\newcommand{\llm}{\mathcal{L}} 
\newcommand{\ones}{\mathbf{1}}

\newcommand{\incentive}{p}
\newcommand{\incentivepolicy}{\nu}
\newcommand{\optincentivepolicy}{\incentivepolicy^*}
\newcommand{\costconstantone}{\alpha}
\newcommand{\costconstanttwo}{\omega}

\newcommand{\actioncost}{\Delta}
\newcommand{\policyspace}{\Pi}
\newcommand{\fusionweight}{g}
\newcommand{\fusioncost}{f}
\newcommand{\discountfactorfusion}{\rho}

\newcommand{\thresholdprior}{\Bar{\prior}}
\newcommand{\incentivefunction}{\chi}

\newcommand{\budget}{B}

\newcommand{\probabilityregion}{\mathcal{R}}

\newcommand{\numagentsprev}{L}

\newcommand{\policyparameter}{\Theta}
\newcommand{\timehorizonspsa}{H}
\newcommand{\perturbationsize}{\delta}
\newcommand{\episodelength}{T}
\newcommand{\approxgradient}{\hat{\nabla}}
\newcommand{\stepsize}{\beta}
\newcommand{\approxcost}{\hat{\avgcost}}

\newcommand{\approximationparameter}{\varepsilon}
\newcommand{\environment}{m}
\newcommand{\actionposterior}{p}
\newcommand{\reconstructedreward}{\hat{r}}
\newcommand{\environmentspace}{\mathcal{M}}
\newcommand{\reconstructedinfocost}{z}
\newcommand{\environmentalt}{l}
\newcommand{\dataset}{\mathbb{D}}
\newcommand{\utility}{r}
\newcommand{\numenvironments}{{M}}

\newcommand{\infoacquisitioncost}{Z}
\newcommand{\argmax}{\arg\max}
\newcommand{\stochastickernels}{\Delta}
\newcommand{\expectedreward}{U}
\newcommand{\observationdim}{Y}

\newcommand{\nias}{\text{NIAS}}
\newcommand{\niac}{\text{NIAC}}
\newcommand{\reconstructedic}{\hat{\infoacquisitioncost}}
\newcommand{\orderof}{\mathsf{O}}
\newcommand{\tokenspace}{\mathcal{D}}

\newcommand{\lengthcontext}{m_{\text{context}}}
\newcommand{\lengthsystem}{m_{\text{control}}}
\newcommand{\lengthuser}{m_{\text{user}}}
\newcommand{\lengthoutput}{m_{\text{output}}}
\newcommand{\systemprompt}{c}
\newcommand{\context}{\kappa}
\newcommand{\llmoutputfn}{g}

\newcommand{\llmprob}{\phi}

\newcommand{\observationmatrixllm}{O}
\newcommand{\ribum}{RIBUM}

\newcommand{\dpiter}{\environment}
\newcommand{\dpset}{\environmentspace}
\newcommand{\obsset}{\mathcal{Y}}
\newcommand{\runcostinst}{c}
\newcommand{\dpitertwo}{\environmentalt}
\newcommand{\attfunagent}[1]{\observationmatrix(#1)}
\newcommand{\attfunsymb}{\observationmatrix}
\newcommand{\normalizationfactor}{\sigma}
 \newcommand{\transitionmatrix}{P}
 \newcommand{\valuefn}{V}
 \newcommand{\change}[1]{{{#1}}}

%% file: plots/parts.tex
\tikzset{every picture/.style={line width=0.75pt}} 
\resizebox{0.8\textwidth}{!}{

\tikzset{every picture/.style={line width=0.75pt}} 

\begin{tikzpicture}[x=0.75pt,y=0.75pt,yscale=-1,xscale=1]

\draw  [fill={rgb, 255:red, 230; green, 24; blue, 24 }  ,fill opacity=0.54 ] (355,119) -- (596,119) -- (596,181) -- (355,181) -- cycle ;
\draw  [fill={rgb, 255:red, 189; green, 16; blue, 224 }  ,fill opacity=0.55 ] (95,118) -- (258.14,118) -- (258.14,178) -- (95,178) -- cycle ;
\draw   (261,152) -- (283.5,139) -- (283.5,145.5) -- (328.5,145.5) -- (328.5,139) -- (351,152) -- (328.5,165) -- (328.5,158.5) -- (283.5,158.5) -- (283.5,165) -- cycle ;
\draw  [fill={rgb, 255:red, 189; green, 16; blue, 224 }  ,fill opacity=0.55 ] (85,289) -- (200.59,289) -- (200.59,332) -- (85,332) -- cycle ;
\draw    (141,259.5) -- (141,283.5) ;
\draw [shift={(141,285.5)}, rotate = 270] [color={rgb, 255:red, 0; green, 0; blue, 0 }  ][line width=0.75]    (10.93,-3.29) .. controls (6.95,-1.4) and (3.31,-0.3) .. (0,0) .. controls (3.31,0.3) and (6.95,1.4) .. (10.93,3.29)   ;
\draw    (201.31,310) -- (222,310) ;
\draw [shift={(224,310)}, rotate = 180] [color={rgb, 255:red, 0; green, 0; blue, 0 }  ][line width=0.75]    (10.93,-3.29) .. controls (6.95,-1.4) and (3.31,-0.3) .. (0,0) .. controls (3.31,0.3) and (6.95,1.4) .. (10.93,3.29)   ;
\draw    (295.59,309.5) -- (318,309.5) ;
\draw [shift={(320,309.5)}, rotate = 180] [color={rgb, 255:red, 0; green, 0; blue, 0 }  ][line width=0.75]    (10.93,-3.29) .. controls (6.95,-1.4) and (3.31,-0.3) .. (0,0) .. controls (3.31,0.3) and (6.95,1.4) .. (10.93,3.29)   ;
\draw  [fill={rgb, 255:red, 189; green, 16; blue, 224 }  ,fill opacity=0.55 ] (387,290) -- (502.59,290) -- (502.59,333) -- (387,333) -- cycle ;
\draw    (443,260.5) -- (443,284.5) ;
\draw [shift={(443,286.5)}, rotate = 270] [color={rgb, 255:red, 0; green, 0; blue, 0 }  ][line width=0.75]    (10.93,-3.29) .. controls (6.95,-1.4) and (3.31,-0.3) .. (0,0) .. controls (3.31,0.3) and (6.95,1.4) .. (10.93,3.29)   ;
\draw    (503.31,311) -- (524,311) ;
\draw [shift={(526,311)}, rotate = 180] [color={rgb, 255:red, 0; green, 0; blue, 0 }  ][line width=0.75]    (10.93,-3.29) .. controls (6.95,-1.4) and (3.31,-0.3) .. (0,0) .. controls (3.31,0.3) and (6.95,1.4) .. (10.93,3.29)   ;
\draw    (597.59,310.5) -- (620,310.5) ;
\draw [shift={(622,310.5)}, rotate = 180] [color={rgb, 255:red, 0; green, 0; blue, 0 }  ][line width=0.75]    (10.93,-3.29) .. controls (6.95,-1.4) and (3.31,-0.3) .. (0,0) .. controls (3.31,0.3) and (6.95,1.4) .. (10.93,3.29)   ;
\draw    (74,198) -- (613,198) ;
\draw [shift={(613,198)}, rotate = 0] [color={rgb, 255:red, 0; green, 0; blue, 0 }  ][fill={rgb, 255:red, 0; green, 0; blue, 0 }  ][line width=0.75]      (0, 0) circle [x radius= 3.35, y radius= 3.35]   ;
\draw [shift={(74,198)}, rotate = 0] [color={rgb, 255:red, 0; green, 0; blue, 0 }  ][fill={rgb, 255:red, 0; green, 0; blue, 0 }  ][line width=0.75]      (0, 0) circle [x radius= 3.35, y radius= 3.35]   ;

\draw (475.5,150) node  [font=\large] [align=left] {\begin{minipage}[lt]{148.96pt}\setlength\topsep{0pt}
\begin{center}
Rationally Inattentive \\Bayesian Utility Maximizer 
\end{center}

\end{minipage}};
\draw (177.57,148) node  [font=\large] [align=left] {LLM Agent};
\draw (405.03,303) node [anchor=north west][inner sep=0.75pt]   [align=left] {LLM Agent $n$};
\draw (387.14,241) node [anchor=north west][inner sep=0.75pt]   [align=left] {Text observation};
\draw (526.74,293) node [anchor=north west][inner sep=0.75pt]   [align=left] {\begin{minipage}[lt]{49.19pt}\setlength\topsep{0pt}
\begin{center}
Action\\(Estimate)
\end{center}

\end{minipage}};
\draw (323,274.4) node [anchor=north west][inner sep=0.75pt]  [font=\LARGE]  {$\cdots $};
\draw (47,182) node [anchor=north west][inner sep=0.75pt]  [rotate=-270] [align=left] {Abstraction 1};
\draw (49,334) node [anchor=north west][inner sep=0.75pt]  [rotate=-270] [align=left] {Abstraction 2};
\draw (103.03,302) node [anchor=north west][inner sep=0.75pt]   [align=left] {LLM Agent 1};
\draw (85.14,240) node [anchor=north west][inner sep=0.75pt]   [align=left] {Text observation};
\draw (224.74,293) node [anchor=north west][inner sep=0.75pt]   [align=left] {\begin{minipage}[lt]{49.19pt}\setlength\topsep{0pt}
\begin{center}
Action\\(Estimate)
\end{center}

\end{minipage}};
\draw (173.14,211) node [anchor=north west][inner sep=0.75pt]   [align=left] {\textbf{Sequence of LLM Agents perform Bayesian social learning}};
\draw (148.14,92) node [anchor=north west][inner sep=0.75pt]   [align=left] {\textbf{LLM Agents are rationally inattentive Bayesian utility maximizers}};

\end{tikzpicture}
}

%% file: plots/llmengineer.tex
\tikzset{every picture/.style={line width=0.75pt}} 
\resizebox{0.7\textwidth}{!}{

\tikzset{every picture/.style={line width=0.75pt}} 

\begin{tikzpicture}[x=0.75pt,y=0.75pt,yscale=-1,xscale=1]

\draw    (242,105.5) -- (242,311.5) ;
\draw  [fill={rgb, 255:red, 218; green, 226; blue, 70 }  ,fill opacity=0.89 ] (342.81,42.13) .. controls (342.81,33.8) and (349.56,27.04) .. (357.89,27.04) .. controls (366.22,27.04) and (372.98,33.8) .. (372.98,42.13) .. controls (372.98,50.46) and (366.22,57.21) .. (357.89,57.21) .. controls (349.56,57.21) and (342.81,50.46) .. (342.81,42.13) -- cycle ;
\draw   (46,196) -- (116,196) -- (116,242) -- (46,242) -- cycle ;
\draw   (143,195) -- (213,195) -- (213,245) -- (143,245) -- cycle ;
\draw   (33,183) -- (228,183) -- (228,256) -- (33,256) -- cycle ;
\draw  [color={rgb, 255:red, 0; green, 0; blue, 0 }  ,draw opacity=1 ][fill={rgb, 255:red, 189; green, 16; blue, 224 }  ,fill opacity=0.55 ] (282,168.5) -- (451,168.5) -- (451,206.27) -- (282,206.27) -- cycle ;
\draw   (288,246) -- (454,246) -- (454,305) -- (288,305) -- cycle ;
\draw  [color={rgb, 255:red, 0; green, 0; blue, 0 }  ,draw opacity=1 ][fill={rgb, 255:red, 189; green, 16; blue, 224 }  ,fill opacity=0.55 ] (514,193) -- (582,193) -- (582,241) -- (514,241) -- cycle ;
\draw  [color={rgb, 255:red, 0; green, 0; blue, 0 }  ,draw opacity=1 ][fill={rgb, 255:red, 189; green, 16; blue, 224 }  ,fill opacity=0.55 ] (597,193) -- (665,193) -- (665,241) -- (597,241) -- cycle ;
\draw  [color={rgb, 255:red, 0; green, 0; blue, 0 }  ,draw opacity=1 ][fill={rgb, 255:red, 189; green, 16; blue, 224 }  ,fill opacity=0.55 ] (514,133) -- (582,133) -- (582,181) -- (514,181) -- cycle ;
\draw  [color={rgb, 255:red, 0; green, 0; blue, 0 }  ,draw opacity=1 ][fill={rgb, 255:red, 189; green, 16; blue, 224 }  ,fill opacity=0.55 ] (597,133) -- (665,133) -- (665,181) -- (597,181) -- cycle ;
\draw  [fill={rgb, 255:red, 218; green, 226; blue, 70 }  ,fill opacity=0.89 ] (341,56.21) -- (375.39,56.21) -- (375.39,90) -- (341,90) -- cycle ;
\draw   (369.37,210.94) -- (369.55,229.61) -- (374,229.57) -- (365.22,242.1) -- (356.21,229.74) -- (360.66,229.7) -- (360.48,211.03) -- cycle ;
\draw    (361,91) -- (361,103) ;
\draw [shift={(361,105)}, rotate = 270] [color={rgb, 255:red, 0; green, 0; blue, 0 }  ][line width=0.75]    (10.93,-3.29) .. controls (6.95,-1.4) and (3.31,-0.3) .. (0,0) .. controls (3.31,0.3) and (6.95,1.4) .. (10.93,3.29)   ;
\draw    (341,77) -- (140,77) -- (140,104) ;
\draw [shift={(140,106)}, rotate = 270] [color={rgb, 255:red, 0; green, 0; blue, 0 }  ][line width=0.75]    (10.93,-3.29) .. controls (6.95,-1.4) and (3.31,-0.3) .. (0,0) .. controls (3.31,0.3) and (6.95,1.4) .. (10.93,3.29)   ;
\draw    (374,76) -- (599,76) -- (599,96) ;
\draw [shift={(599,98)}, rotate = 270] [color={rgb, 255:red, 0; green, 0; blue, 0 }  ][line width=0.75]    (10.93,-3.29) .. controls (6.95,-1.4) and (3.31,-0.3) .. (0,0) .. controls (3.31,0.3) and (6.95,1.4) .. (10.93,3.29)   ;
\draw    (583,157) -- (595,157) ;
\draw [shift={(597,157)}, rotate = 180] [color={rgb, 255:red, 0; green, 0; blue, 0 }  ][line width=0.75]    (10.93,-3.29) .. controls (6.95,-1.4) and (3.31,-0.3) .. (0,0) .. controls (3.31,0.3) and (6.95,1.4) .. (10.93,3.29)   ;
\draw    (582,218) -- (594,218) ;
\draw [shift={(596,218)}, rotate = 180] [color={rgb, 255:red, 0; green, 0; blue, 0 }  ][line width=0.75]    (10.93,-3.29) .. controls (6.95,-1.4) and (3.31,-0.3) .. (0,0) .. controls (3.31,0.3) and (6.95,1.4) .. (10.93,3.29)   ;
\draw    (597,181) -- (583.56,191.75) ;
\draw [shift={(582,193)}, rotate = 321.34] [color={rgb, 255:red, 0; green, 0; blue, 0 }  ][line width=0.75]    (10.93,-3.29) .. controls (6.95,-1.4) and (3.31,-0.3) .. (0,0) .. controls (3.31,0.3) and (6.95,1.4) .. (10.93,3.29)   ;
\draw    (118,219.5) -- (141,219.5) ;
\draw [shift={(143,219.5)}, rotate = 180] [color={rgb, 255:red, 0; green, 0; blue, 0 }  ][line width=0.75]    (10.93,-3.29) .. controls (6.95,-1.4) and (3.31,-0.3) .. (0,0) .. controls (3.31,0.3) and (6.95,1.4) .. (10.93,3.29)   ;
\draw   (593,243) -- (592.91,257.6) -- (596.48,257.62) -- (589.28,267.31) -- (582.19,257.54) -- (585.76,257.56) -- (585.85,242.96) -- cycle ;
\draw    (488,104.5) -- (488,314.5) ;
\draw   (136.16,157.39) -- (136.07,171.99) -- (139.65,172.01) -- (132.44,181.7) -- (125.35,171.93) -- (128.93,171.95) -- (129.02,157.35) -- cycle ;
\draw   (138.16,260.39) -- (138.07,274.99) -- (141.65,275.01) -- (134.44,284.7) -- (127.35,274.93) -- (130.93,274.95) -- (131.02,260.35) -- cycle ;

\draw (58,205) node [anchor=north west][inner sep=0.75pt]  [font=\Large] [align=left] {LLM};
\draw (149,204) node [anchor=north west][inner sep=0.75pt]  [font=\small] [align=left] {\begin{minipage}[lt]{43.04pt}\setlength\topsep{0pt}
\begin{center}
Bayesian \\Engine
\end{center}

\end{minipage}};
\draw (35,109) node [anchor=north west][inner sep=0.75pt]  [font=\normalsize] [align=left] {{\small \textbf{LLMA as a Sensing Mechanism}}};
\draw (337,174) node [anchor=north west][inner sep=0.75pt]  [font=\Large] [align=left] {LLMA};
\draw (294,256) node [anchor=north west][inner sep=0.75pt]  [font=\normalsize] [align=left] {{\small Rational Inattention Cost \ \ \ \ }\\{\small Bayesian Utility}};
\draw (519,144) node [anchor=north west][inner sep=0.75pt]  [font=\Large] [align=left] {LLMA};
\draw (603,144) node [anchor=north west][inner sep=0.75pt]  [font=\Large] [align=left] {LLMA};
\draw (519,206) node [anchor=north west][inner sep=0.75pt]  [font=\Large] [align=left] {LLMA};
\draw (603,205) node [anchor=north west][inner sep=0.75pt]  [font=\Large] [align=left] {LLMA};
\draw (490,102) node [anchor=north west][inner sep=0.75pt]  [font=\normalsize] [align=left] {\begin{minipage}[lt]{150.94pt}\setlength\topsep{0pt}
\begin{center}
{\small \textbf{Bayesian Social Learning and Stochastic Control}}
\end{center}

\end{minipage}};
\draw (385.01,36.19) node [anchor=north west][inner sep=0.75pt]  [font=\large] [align=left] {\textbf{LLMA Engineer }};
\draw (528,271) node [anchor=north west][inner sep=0.75pt]  [font=\small] [align=left] {\begin{minipage}[lt]{94.58pt}\setlength\topsep{0pt}
\begin{center}
Interpretable \& Robust\\Bayesian Inference
\end{center}

\end{minipage}};
\draw (81,138) node [anchor=north west][inner sep=0.75pt]  [font=\small] [align=left] {\begin{minipage}[lt]{71.6pt}\setlength\topsep{0pt}
\begin{center}
Text Observation
\end{center}

\end{minipage}};
\draw (88,289) node [anchor=north west][inner sep=0.75pt]  [font=\small] [align=left] {\begin{minipage}[lt]{62.42pt}\setlength\topsep{0pt}
\begin{center}
State Estimate
\end{center}

\end{minipage}};
\draw (252,106.5) node [anchor=north west][inner sep=0.75pt]  [font=\normalsize] [align=left] {\begin{minipage}[lt]{154.22pt}\setlength\topsep{0pt}
\begin{center}
\textbf{{\small Reconstructed Interpretable Model}}\\{\small using }\\{\small Bayesian Revealed Preferences}
\end{center}

\end{minipage}};

\end{tikzpicture}
}

%% file: plots/organization.tex
\tikzset{every picture/.style={line width=0.75pt}} 
\resizebox{0.9\textwidth}{!}{
\begin{tikzpicture}[x=0.75pt,y=0.75pt,yscale=-1,xscale=1]

\draw    (100.67,67) -- (100.67,160.33) ;
\draw    (273.33,62.33) -- (274,212) ;
\draw    (100,110.33) -- (138,110.33) ;
\draw    (100.67,160.33) -- (138.67,160.33) ;
\draw    (273.33,91) -- (311.33,91) ;
\draw    (272.67,146.33) -- (310.67,146.33) ;
\draw    (272.67,195) -- (310.67,195) ;
\draw    (458.67,61) -- (458.67,219) ;
\draw    (458.67,89.67) -- (496.67,89.67) ;
\draw    (458,145) -- (496,145) ;
\draw    (458,186.33) -- (496,186.33) ;
\draw    (458.67,219) -- (496.67,219) ;

\draw (142.67,86.67) node [anchor=north west][inner sep=0.75pt]   [align=left] {{\footnotesize \textbf{Sec. 3}: LLM Agent }\\{\footnotesize as a Sensor}};
\draw (142,136.67) node [anchor=north west][inner sep=0.75pt]   [align=left] {{\footnotesize \textbf{Sec. 4}: LLM Agent as a }\\{\footnotesize Rationally Inattentive }\\{\footnotesize Bayesian Utility }\\{\footnotesize Maximizer (RIBUM)}};
\draw (89.33,26) node [anchor=north west][inner sep=0.75pt]   [align=left] {{\footnotesize \textbf{Part 1}}\\{\footnotesize \textbf{Individual LLM Agent }}};
\draw (253.33,24.67) node [anchor=north west][inner sep=0.75pt]   [align=left] {{\footnotesize \textbf{Part 2}}\\{\footnotesize \textbf{Network of LLM Agents }}};
\draw (313.33,72) node [anchor=north west][inner sep=0.75pt]   [align=left] {{\footnotesize \textbf{Sec. 5: } Sequential }\\{\footnotesize Bayesian Social Learning}};
\draw (314.67,128) node [anchor=north west][inner sep=0.75pt]   [align=left] {{\footnotesize \textbf{Sec. 6}: Word-of-Mouth}\\{\footnotesize Bayesian Social Learning}};
\draw (314,176.67) node [anchor=north west][inner sep=0.75pt]   [align=left] {{\footnotesize \textbf{Sec. 6: }Asynchronous}\\{\footnotesize Bayesian Social Learning}};
\draw (448,23.33) node [anchor=north west][inner sep=0.75pt]   [align=left] {{\footnotesize \textbf{Part 3}}\\{\footnotesize \textbf{Stochastic Control for delaying Herding in LLM Agents }}};
\draw (500.67,70) node [anchor=north west][inner sep=0.75pt]   [align=left] {{\footnotesize \textbf{Sec. 7:} Optimal Stopping for }\\{\footnotesize Centrally Controlled LLM Agents}};
\draw (500.67,126.67) node [anchor=north west][inner sep=0.75pt]   [align=left] {{\footnotesize \textbf{Sec. 8: }Optimal Stopping for }\\{\footnotesize Incentivized Autonomous LLMAs}};
\draw (499.33,176) node [anchor=north west][inner sep=0.75pt]   [align=left] {{\footnotesize \textbf{Sec. 9:} Stochastic Approximation for Optimal Policy }};
\draw (498,209.33) node [anchor=north west][inner sep=0.75pt]   [align=left] {{\footnotesize \textbf{Sec. 10:} Numerical Results on Bayesian inference }\\{\footnotesize Product Quality and Hate Speech Peddler Identification}};

\end{tikzpicture}}

%% file: plots/model.tex
\tikzset{every picture/.style={line width=0.75pt}} 
\resizebox{0.8\textwidth}{!}{

\begin{tikzpicture}[x=0.75pt,y=0.75pt,yscale=-1,xscale=1]

\draw   (276,250) -- (606,250) -- (606,350) -- (276,350) -- cycle ;
\draw   (296,270) -- (426,270) -- (426,330) -- (296,330) -- cycle ;
\draw   (456,270) -- (586,270) -- (586,330) -- (456,330) -- cycle ;
\draw    (606,300) -- (634,300) ;
\draw [shift={(636,300)}, rotate = 180] [color={rgb, 255:red, 0; green, 0; blue, 0 }  ][line width=0.75]    (10.93,-3.29) .. controls (6.95,-1.4) and (3.31,-0.3) .. (0,0) .. controls (3.31,0.3) and (6.95,1.4) .. (10.93,3.29)   ;
\draw    (226,300) -- (274,300) ;
\draw [shift={(276,300)}, rotate = 180] [color={rgb, 255:red, 0; green, 0; blue, 0 }  ][line width=0.75]    (10.93,-3.29) .. controls (6.95,-1.4) and (3.31,-0.3) .. (0,0) .. controls (3.31,0.3) and (6.95,1.4) .. (10.93,3.29)   ;
\draw    (426,300) -- (454,300) ;
\draw [shift={(456,300)}, rotate = 180] [color={rgb, 255:red, 0; green, 0; blue, 0 }  ][line width=0.75]    (10.93,-3.29) .. controls (6.95,-1.4) and (3.31,-0.3) .. (0,0) .. controls (3.31,0.3) and (6.95,1.4) .. (10.93,3.29)   ;
\draw    (226,340) -- (274,340) ;
\draw [shift={(276,340)}, rotate = 180] [color={rgb, 255:red, 0; green, 0; blue, 0 }  ][line width=0.75]    (10.93,-3.29) .. controls (6.95,-1.4) and (3.31,-0.3) .. (0,0) .. controls (3.31,0.3) and (6.95,1.4) .. (10.93,3.29)   ;
\draw    (226,261) -- (274,261) ;
\draw [shift={(276,261)}, rotate = 180] [color={rgb, 255:red, 0; green, 0; blue, 0 }  ][line width=0.75]    (10.93,-3.29) .. controls (6.95,-1.4) and (3.31,-0.3) .. (0,0) .. controls (3.31,0.3) and (6.95,1.4) .. (10.93,3.29)   ;
\draw  [dash pattern={on 0.84pt off 2.51pt}] (30,230) -- (250,230) -- (250,370) -- (30,370) -- cycle ;

\draw (208,291.4) node [anchor=north west][inner sep=0.75pt]    {$y^{'}_{k}$};
\draw (168,300.4) node [anchor=north west][inner sep=0.75pt]    {$x_{k}$};
\draw (185,300.4) node [anchor=north west][inner sep=0.75pt]    {$\sim $};
\draw (307,281) node [anchor=north west][inner sep=0.75pt]   [align=left] {\begin{minipage}[lt]{77.03pt}\setlength\topsep{0pt}
Large Language
\begin{center}
Model
\end{center}

\end{minipage}};
\draw (481,282) node [anchor=north west][inner sep=0.75pt]  [font=\scriptsize] [align=left] {Bayesian Engine\\(e.g. Probabilistic \\Neural Network)};
\draw (430,225) node   [align=left] {\begin{minipage}[lt]{190.23pt}\setlength\topsep{0pt}
\begin{center}
Large \textbf{L}anguage \textbf{M}odel \textbf{A}gent (Rationally Inattentive Utility Maximizer)
\end{center}

\end{minipage}};
\draw (639,291.4) node [anchor=north west][inner sep=0.75pt]    {$a_{k}$};
\draw (434,275.4) node [anchor=north west][inner sep=0.75pt]    {$y_{k}$};
\draw (187,332.4) node [anchor=north west][inner sep=0.75pt]    {$a_{k-1}$};
\draw (207,252.4) node [anchor=north west][inner sep=0.75pt]    {$c_{k}$};
\draw (41,249) node [anchor=north west][inner sep=0.75pt]  [font=\footnotesize] [align=left] {Control \\(System Prompt)};
\draw (41,289) node [anchor=north west][inner sep=0.75pt]  [font=\footnotesize] [align=left] {Private Observation\\(User Prompt)\\ (Text to be analyzed)};
\draw (41,329) node [anchor=north west][inner sep=0.75pt]  [font=\footnotesize] [align=left] {Context\\(Background)};
\draw (31,212) node [anchor=north west][inner sep=0.75pt]   [align=left] {LLM Input (\textcolor[rgb]{0.29,0.56,0.89}{Text Input})};
\draw (611,251) node [anchor=north west][inner sep=0.75pt]   [align=left] {LLM Output};
\draw (611,271) node [anchor=north west][inner sep=0.75pt]   [align=left] {(\textcolor[rgb]{0.29,0.56,0.89}{Sentiment})};

\end{tikzpicture}
}

%% file: plots/part2.tex
\tikzset{every picture/.style={line width=0.75pt}} 

\begin{tikzpicture}[x=0.75pt,y=0.75pt,yscale=-1,xscale=1]

\draw  [fill={rgb, 255:red, 235; green, 160; blue, 232 }  ,fill opacity=1 ] (80,60) -- (145,60) -- (145,96) -- (80,96) -- cycle ;
\draw    (300,102) -- (300,115) ;
\draw [shift={(300,117)}, rotate = 270] [color={rgb, 255:red, 0; green, 0; blue, 0 }  ][line width=0.75]    (10.93,-3.29) .. controls (6.95,-1.4) and (3.31,-0.3) .. (0,0) .. controls (3.31,0.3) and (6.95,1.4) .. (10.93,3.29)   ;
\draw    (300,204) -- (300.88,218) ;
\draw [shift={(301,220)}, rotate = 266.42] [color={rgb, 255:red, 0; green, 0; blue, 0 }  ][line width=0.75]    (10.93,-3.29) .. controls (6.95,-1.4) and (3.31,-0.3) .. (0,0) .. controls (3.31,0.3) and (6.95,1.4) .. (10.93,3.29)   ;
\draw    (300,154) -- (300.88,168) ;
\draw [shift={(301,170)}, rotate = 266.42] [color={rgb, 255:red, 0; green, 0; blue, 0 }  ][line width=0.75]    (10.93,-3.29) .. controls (6.95,-1.4) and (3.31,-0.3) .. (0,0) .. controls (3.31,0.3) and (6.95,1.4) .. (10.93,3.29)   ;
\draw    (335,240) -- (405,240) -- (404.02,161) ;
\draw [shift={(404,159)}, rotate = 89.29] [color={rgb, 255:red, 0; green, 0; blue, 0 }  ][line width=0.75]    (10.93,-3.29) .. controls (6.95,-1.4) and (3.31,-0.3) .. (0,0) .. controls (3.31,0.3) and (6.95,1.4) .. (10.93,3.29)   ;
\draw    (375,148) -- (335.4,188.57) ;
\draw [shift={(334,190)}, rotate = 314.31] [color={rgb, 255:red, 0; green, 0; blue, 0 }  ][line width=0.75]    (10.93,-3.29) .. controls (6.95,-1.4) and (3.31,-0.3) .. (0,0) .. controls (3.31,0.3) and (6.95,1.4) .. (10.93,3.29)   ;
\draw    (375,136) -- (337,136) ;
\draw [shift={(335,136)}, rotate = 360] [color={rgb, 255:red, 0; green, 0; blue, 0 }  ][line width=0.75]    (10.93,-3.29) .. controls (6.95,-1.4) and (3.31,-0.3) .. (0,0) .. controls (3.31,0.3) and (6.95,1.4) .. (10.93,3.29)   ;
\draw  [fill={rgb, 255:red, 141; green, 190; blue, 248 }  ,fill opacity=1 ] (375,121) -- (440,121) -- (440,157) -- (375,157) -- cycle ;
\draw    (375,148) -- (335.8,238.17) ;
\draw [shift={(335,240)}, rotate = 293.5] [color={rgb, 255:red, 0; green, 0; blue, 0 }  ][line width=0.75]    (10.93,-3.29) .. controls (6.95,-1.4) and (3.31,-0.3) .. (0,0) .. controls (3.31,0.3) and (6.95,1.4) .. (10.93,3.29)   ;
\draw  [fill={rgb, 255:red, 141; green, 190; blue, 248 }  ,fill opacity=1 ] (536,140) -- (601,140) -- (601,176) -- (536,176) -- cycle ;
\draw    (51,78) -- (75,78) ;
\draw [shift={(77,78)}, rotate = 180] [color={rgb, 255:red, 0; green, 0; blue, 0 }  ][line width=0.75]    (10.93,-3.29) .. controls (6.95,-1.4) and (3.31,-0.3) .. (0,0) .. controls (3.31,0.3) and (6.95,1.4) .. (10.93,3.29)   ;
\draw    (506,94) -- (534.91,138.32) ;
\draw [shift={(536,140)}, rotate = 236.89] [color={rgb, 255:red, 0; green, 0; blue, 0 }  ][line width=0.75]    (10.93,-3.29) .. controls (6.95,-1.4) and (3.31,-0.3) .. (0,0) .. controls (3.31,0.3) and (6.95,1.4) .. (10.93,3.29)   ;
\draw    (547,139) -- (520.07,96.69) ;
\draw [shift={(519,95)}, rotate = 57.53] [color={rgb, 255:red, 0; green, 0; blue, 0 }  ][line width=0.75]    (10.93,-3.29) .. controls (6.95,-1.4) and (3.31,-0.3) .. (0,0) .. controls (3.31,0.3) and (6.95,1.4) .. (10.93,3.29)   ;
\draw    (609,92) -- (587.85,137.19) ;
\draw [shift={(587,139)}, rotate = 295.08] [color={rgb, 255:red, 0; green, 0; blue, 0 }  ][line width=0.75]    (10.93,-3.29) .. controls (6.95,-1.4) and (3.31,-0.3) .. (0,0) .. controls (3.31,0.3) and (6.95,1.4) .. (10.93,3.29)   ;
\draw    (601,140) -- (621.2,93.83) ;
\draw [shift={(622,92)}, rotate = 113.63] [color={rgb, 255:red, 0; green, 0; blue, 0 }  ][line width=0.75]    (10.93,-3.29) .. controls (6.95,-1.4) and (3.31,-0.3) .. (0,0) .. controls (3.31,0.3) and (6.95,1.4) .. (10.93,3.29)   ;
\draw    (582,175) -- (610.91,219.32) ;
\draw [shift={(612,221)}, rotate = 236.89] [color={rgb, 255:red, 0; green, 0; blue, 0 }  ][line width=0.75]    (10.93,-3.29) .. controls (6.95,-1.4) and (3.31,-0.3) .. (0,0) .. controls (3.31,0.3) and (6.95,1.4) .. (10.93,3.29)   ;
\draw    (623,220) -- (596.07,177.69) ;
\draw [shift={(595,176)}, rotate = 57.53] [color={rgb, 255:red, 0; green, 0; blue, 0 }  ][line width=0.75]    (10.93,-3.29) .. controls (6.95,-1.4) and (3.31,-0.3) .. (0,0) .. controls (3.31,0.3) and (6.95,1.4) .. (10.93,3.29)   ;
\draw    (535,176) -- (513.85,221.19) ;
\draw [shift={(513,223)}, rotate = 295.08] [color={rgb, 255:red, 0; green, 0; blue, 0 }  ][line width=0.75]    (10.93,-3.29) .. controls (6.95,-1.4) and (3.31,-0.3) .. (0,0) .. controls (3.31,0.3) and (6.95,1.4) .. (10.93,3.29)   ;
\draw    (527,224) -- (547.2,177.83) ;
\draw [shift={(548,176)}, rotate = 113.63] [color={rgb, 255:red, 0; green, 0; blue, 0 }  ][line width=0.75]    (10.93,-3.29) .. controls (6.95,-1.4) and (3.31,-0.3) .. (0,0) .. controls (3.31,0.3) and (6.95,1.4) .. (10.93,3.29)   ;
\draw    (50,129) -- (74,129) ;
\draw [shift={(76,129)}, rotate = 180] [color={rgb, 255:red, 0; green, 0; blue, 0 }  ][line width=0.75]    (10.93,-3.29) .. controls (6.95,-1.4) and (3.31,-0.3) .. (0,0) .. controls (3.31,0.3) and (6.95,1.4) .. (10.93,3.29)   ;
\draw    (49,181) -- (73,181) ;
\draw [shift={(75,181)}, rotate = 180] [color={rgb, 255:red, 0; green, 0; blue, 0 }  ][line width=0.75]    (10.93,-3.29) .. controls (6.95,-1.4) and (3.31,-0.3) .. (0,0) .. controls (3.31,0.3) and (6.95,1.4) .. (10.93,3.29)   ;
\draw    (50,232) -- (74,232) ;
\draw [shift={(76,232)}, rotate = 180] [color={rgb, 255:red, 0; green, 0; blue, 0 }  ][line width=0.75]    (10.93,-3.29) .. controls (6.95,-1.4) and (3.31,-0.3) .. (0,0) .. controls (3.31,0.3) and (6.95,1.4) .. (10.93,3.29)   ;
\draw    (149,80) -- (173,80) ;
\draw [shift={(175,80)}, rotate = 180] [color={rgb, 255:red, 0; green, 0; blue, 0 }  ][line width=0.75]    (10.93,-3.29) .. controls (6.95,-1.4) and (3.31,-0.3) .. (0,0) .. controls (3.31,0.3) and (6.95,1.4) .. (10.93,3.29)   ;
\draw    (148,131) -- (172,131) ;
\draw [shift={(174,131)}, rotate = 180] [color={rgb, 255:red, 0; green, 0; blue, 0 }  ][line width=0.75]    (10.93,-3.29) .. controls (6.95,-1.4) and (3.31,-0.3) .. (0,0) .. controls (3.31,0.3) and (6.95,1.4) .. (10.93,3.29)   ;
\draw    (147,183) -- (171,183) ;
\draw [shift={(173,183)}, rotate = 180] [color={rgb, 255:red, 0; green, 0; blue, 0 }  ][line width=0.75]    (10.93,-3.29) .. controls (6.95,-1.4) and (3.31,-0.3) .. (0,0) .. controls (3.31,0.3) and (6.95,1.4) .. (10.93,3.29)   ;
\draw    (148,234) -- (172,234) ;
\draw [shift={(174,234)}, rotate = 180] [color={rgb, 255:red, 0; green, 0; blue, 0 }  ][line width=0.75]    (10.93,-3.29) .. controls (6.95,-1.4) and (3.31,-0.3) .. (0,0) .. controls (3.31,0.3) and (6.95,1.4) .. (10.93,3.29)   ;
\draw   (46,64) .. controls (41.33,64) and (39,66.33) .. (39,71) -- (39,143.25) .. controls (39,149.92) and (36.67,153.25) .. (32,153.25) .. controls (36.67,153.25) and (39,156.58) .. (39,163.25)(39,160.25) -- (39,235.5) .. controls (39,240.17) and (41.33,242.5) .. (46,242.5) ;
\draw   (175,243.5) .. controls (179.67,243.5) and (182,241.17) .. (182,236.5) -- (182,163.5) .. controls (182,156.83) and (184.33,153.5) .. (189,153.5) .. controls (184.33,153.5) and (182,150.17) .. (182,143.5)(182,146.5) -- (182,70.5) .. controls (182,65.83) and (179.67,63.5) .. (175,63.5) ;
\draw    (259,120.5) -- (259,253.5) ;
\draw [shift={(259,255.5)}, rotate = 270] [color={rgb, 255:red, 0; green, 0; blue, 0 }  ][line width=0.75]    (10.93,-3.29) .. controls (6.95,-1.4) and (3.31,-0.3) .. (0,0) .. controls (3.31,0.3) and (6.95,1.4) .. (10.93,3.29)   ;
\draw    (452,31.5) -- (452,329.5) ;
\draw    (625.5,274.5) -- (625.5,256) ;
\draw [shift={(625.5,254)}, rotate = 90] [color={rgb, 255:red, 0; green, 0; blue, 0 }  ][line width=0.75]    (10.93,-3.29) .. controls (6.95,-1.4) and (3.31,-0.3) .. (0,0) .. controls (3.31,0.3) and (6.95,1.4) .. (10.93,3.29)   ;
\draw    (217,34.5) -- (217,331.5) ;
\draw  [fill={rgb, 255:red, 235; green, 160; blue, 232 }  ,fill opacity=1 ] (79,111) -- (144,111) -- (144,147) -- (79,147) -- cycle ;
\draw  [fill={rgb, 255:red, 235; green, 160; blue, 232 }  ,fill opacity=1 ] (79,162) -- (144,162) -- (144,198) -- (79,198) -- cycle ;
\draw  [fill={rgb, 255:red, 235; green, 160; blue, 232 }  ,fill opacity=1 ] (78,215) -- (143,215) -- (143,251) -- (78,251) -- cycle ;
\draw  [fill={rgb, 255:red, 235; green, 160; blue, 232 }  ,fill opacity=1 ] (268,115) -- (333,115) -- (333,151) -- (268,151) -- cycle ;
\draw  [fill={rgb, 255:red, 235; green, 160; blue, 232 }  ,fill opacity=1 ] (269,170) -- (334,170) -- (334,206) -- (269,206) -- cycle ;
\draw  [fill={rgb, 255:red, 235; green, 160; blue, 232 }  ,fill opacity=1 ] (268,222) -- (333,222) -- (333,258) -- (268,258) -- cycle ;
\draw  [fill={rgb, 255:red, 235; green, 160; blue, 232 }  ,fill opacity=1 ] (480,56) -- (545,56) -- (545,92) -- (480,92) -- cycle ;
\draw  [fill={rgb, 255:red, 235; green, 160; blue, 232 }  ,fill opacity=1 ] (592,55) -- (657,55) -- (657,91) -- (592,91) -- cycle ;
\draw  [fill={rgb, 255:red, 235; green, 160; blue, 232 }  ,fill opacity=1 ] (484,223) -- (549,223) -- (549,259) -- (484,259) -- cycle ;
\draw  [fill={rgb, 255:red, 235; green, 160; blue, 232 }  ,fill opacity=1 ] (592,219) -- (657,219) -- (657,255) -- (592,255) -- cycle ;

\draw (91,71) node [anchor=north west][inner sep=0.75pt]   [align=left] {LLMA};
\draw (388,131) node [anchor=north west][inner sep=0.75pt]   [align=left] {\textbf{Prior}};
\draw (12.68,215.52) node [anchor=north west][inner sep=0.75pt]  [rotate=-269.85] [align=left] {Private Observations};
\draw (208.64,108.48) node [anchor=north west][inner sep=0.75pt]  [rotate=-89.86] [align=left] {Public Actions};
\draw (261.78,63.76) node [anchor=north west][inner sep=0.75pt]  [font=\small,rotate=-359.79] [align=left] {\begin{minipage}[lt]{52.22pt}\setlength\topsep{0pt}
\begin{center}
Private \\Observation
\end{center}

\end{minipage}};
\draw (233.08,247.97) node [anchor=north west][inner sep=0.75pt]  [rotate=-270.19] [align=left] {\begin{minipage}[lt]{96.26pt}\setlength\topsep{0pt}
\begin{center}
Increasing Hierarchy
\end{center}

\end{minipage}};
\draw (337.01,242.99) node [anchor=north west][inner sep=0.75pt]  [rotate=-359.79] [align=left] {\begin{minipage}[lt]{71.62pt}\setlength\topsep{0pt}
\begin{center}
{\small Final LLMA }\\{\small updates the prior}
\end{center}

\end{minipage}};
\draw (549,149) node [anchor=north west][inner sep=0.75pt]   [align=left] {\textbf{Prior}};
\draw (562.78,274.76) node [anchor=north west][inner sep=0.75pt]  [font=\small,rotate=-359.79] [align=left] {\begin{minipage}[lt]{83.34pt}\setlength\topsep{0pt}
\begin{center}
Private Observation
\end{center}

\end{minipage}};
\draw (73,31) node [anchor=north west][inner sep=0.75pt]   [align=left] {\textbf{Sequential}};
\draw (282,31) node [anchor=north west][inner sep=0.75pt]   [align=left] {\textbf{Word-of-Mouth}};
\draw (520,31) node [anchor=north west][inner sep=0.75pt]   [align=left] {\textbf{Asynchronous}};
\draw (21,7) node [anchor=north west][inner sep=0.75pt]   [align=left] {\textbf{{Different Types of Bayesian Social Learning}}};
\draw (25,288) node [anchor=north west][inner sep=0.75pt]   [align=left] {\textbf{{Undesirable phenomena}}};
\draw (42,310)  node [anchor=north west][inner sep=0.75pt]   [align=left] {\textcolor[rgb]{0,0,0}{\textbf{Information Cascade}}};

\draw (539,310) node [anchor=north west][inner sep=0.75pt]   [align=left] {\textcolor[rgb]{0,0,0}{\textbf{Data Incest}}};
\draw (288,310) node [anchor=north west][inner sep=0.75pt]   [align=left] {\textcolor[rgb]{0,0,0}{\textbf{Model Collapse}}};
\draw (90,122) node [anchor=north west][inner sep=0.75pt]   [align=left] {LLMA};
\draw (90,173) node [anchor=north west][inner sep=0.75pt]   [align=left] {LLMA};
\draw (89,226) node [anchor=north west][inner sep=0.75pt]   [align=left] {LLMA};
\draw (279,126) node [anchor=north west][inner sep=0.75pt]   [align=left] {LLMA};
\draw (280,181) node [anchor=north west][inner sep=0.75pt]   [align=left] {LLMA};
\draw (279,233) node [anchor=north west][inner sep=0.75pt]   [align=left] {LLMA};
\draw (491,67) node [anchor=north west][inner sep=0.75pt]   [align=left] {LLMA};
\draw (603,66) node [anchor=north west][inner sep=0.75pt]   [align=left] {LLMA};
\draw (495,234) node [anchor=north west][inner sep=0.75pt]   [align=left] {LLMA};
\draw (603,230) node [anchor=north west][inner sep=0.75pt]   [align=left] {LLMA};

\end{tikzpicture}

%% file: plots/herdingregion.tex
\tikzset{every picture/.style={line width=0.75pt}} 

\begin{tikzpicture}[x=0.75pt,y=0.75pt,yscale=-1,xscale=1]

\draw    (176,41) -- (355.33,41) -- (470,41) ;
\draw [shift={(470,41)}, rotate = 180] [color={rgb, 255:red, 0; green, 0; blue, 0 }  ][line width=0.75]    (0,5.59) -- (0,-5.59)   ;
\draw [shift={(265.67,41)}, rotate = 180] [color={rgb, 255:red, 0; green, 0; blue, 0 }  ][line width=0.75]    (0,5.59) -- (0,-5.59)   ;
\draw [shift={(412.67,41)}, rotate = 180] [color={rgb, 255:red, 0; green, 0; blue, 0 }  ][line width=0.75]    (0,5.59) -- (0,-5.59)   ;
\draw [shift={(176,41)}, rotate = 180] [color={rgb, 255:red, 0; green, 0; blue, 0 }  ][line width=0.75]    (0,5.59) -- (0,-5.59)   ;

\draw (196,1) node [anchor=north west][inner sep=0.75pt]   [align=left] {Herding\\($\displaystyle u=0$)};
\draw (414,0.67) node [anchor=north west][inner sep=0.75pt]   [align=left] {Herding\\($\displaystyle u=1$)};
\draw (304.67,2) node [anchor=north west][inner sep=0.75pt]   [align=left] {Learning\\};
\draw (152,31.73) node [anchor=north west][inner sep=0.75pt]    {$\pi_{k}$};
\draw (170,51.73) node [anchor=north west][inner sep=0.75pt]    {$0$};
\draw (465.33,48.4) node [anchor=north west][inner sep=0.75pt]    {$1$};
\draw (258,49.73) node [anchor=north west][inner sep=0.75pt]    {$\overline{\pi }_{1}$};
\draw (404.67,49.73) node [anchor=north west][inner sep=0.75pt]    {$\overline{\pi }_{2}$};

\end{tikzpicture}

%% file: plots/plot.pgf
\begingroup%
\makeatletter%
\begin{pgfpicture}%
\pgfpathrectangle{\pgfpointorigin}{\pgfqpoint{10.852401in}{5.015711in}}%
\pgfusepath{use as bounding box, clip}%
\begin{pgfscope}%
\pgfsetbuttcap%
\pgfsetmiterjoin%
\definecolor{currentfill}{rgb}{1.000000,1.000000,1.000000}%
\pgfsetfillcolor{currentfill}%
\pgfsetlinewidth{0.000000pt}%
\definecolor{currentstroke}{rgb}{1.000000,1.000000,1.000000}%
\pgfsetstrokecolor{currentstroke}%
\pgfsetdash{}{0pt}%
\pgfpathmoveto{\pgfqpoint{0.000000in}{0.000000in}}%
\pgfpathlineto{\pgfqpoint{10.852401in}{0.000000in}}%
\pgfpathlineto{\pgfqpoint{10.852401in}{5.015711in}}%
\pgfpathlineto{\pgfqpoint{0.000000in}{5.015711in}}%
\pgfpathlineto{\pgfqpoint{0.000000in}{0.000000in}}%
\pgfpathclose%
\pgfusepath{fill}%
\end{pgfscope}%
\begin{pgfscope}%
\pgfsetbuttcap%
\pgfsetmiterjoin%
\definecolor{currentfill}{rgb}{1.000000,1.000000,1.000000}%
\pgfsetfillcolor{currentfill}%
\pgfsetlinewidth{0.000000pt}%
\definecolor{currentstroke}{rgb}{0.000000,0.000000,0.000000}%
\pgfsetstrokecolor{currentstroke}%
\pgfsetstrokeopacity{0.000000}%
\pgfsetdash{}{0pt}%
\pgfpathmoveto{\pgfqpoint{0.906423in}{1.065711in}}%
\pgfpathlineto{\pgfqpoint{10.593923in}{1.065711in}}%
\pgfpathlineto{\pgfqpoint{10.593923in}{4.915711in}}%
\pgfpathlineto{\pgfqpoint{0.906423in}{4.915711in}}%
\pgfpathlineto{\pgfqpoint{0.906423in}{1.065711in}}%
\pgfpathclose%
\pgfusepath{fill}%
\end{pgfscope}%
\begin{pgfscope}%
\definecolor{textcolor}{rgb}{0.150000,0.150000,0.150000}%
\pgfsetstrokecolor{textcolor}%
\pgfsetfillcolor{textcolor}%
\pgftext[x=0.906423in,y=0.968489in,,top]{\color{textcolor}{\sffamily\fontsize{32.000000}{38.400000}\selectfont\catcode`\^=\active\def^{\ifmmode\sp\else\^{}\fi}\catcode`\%=\active\def
\end{pgfscope}%
\begin{pgfscope}%
\definecolor{textcolor}{rgb}{0.150000,0.150000,0.150000}%
\pgfsetstrokecolor{textcolor}%
\pgfsetfillcolor{textcolor}%
\pgftext[x=2.521007in,y=0.968489in,,top]{\color{textcolor}{\sffamily\fontsize{32.000000}{38.400000}\selectfont\catcode`\^=\active\def^{\ifmmode\sp\else\^{}\fi}\catcode`\%=\active\def
\end{pgfscope}%
\begin{pgfscope}%
\definecolor{textcolor}{rgb}{0.150000,0.150000,0.150000}%
\pgfsetstrokecolor{textcolor}%
\pgfsetfillcolor{textcolor}%
\pgftext[x=4.135590in,y=0.968489in,,top]{\color{textcolor}{\sffamily\fontsize{32.000000}{38.400000}\selectfont\catcode`\^=\active\def^{\ifmmode\sp\else\^{}\fi}\catcode`\%=\active\def
\end{pgfscope}%
\begin{pgfscope}%
\definecolor{textcolor}{rgb}{0.150000,0.150000,0.150000}%
\pgfsetstrokecolor{textcolor}%
\pgfsetfillcolor{textcolor}%
\pgftext[x=5.750173in,y=0.968489in,,top]{\color{textcolor}{\sffamily\fontsize{32.000000}{38.400000}\selectfont\catcode`\^=\active\def^{\ifmmode\sp\else\^{}\fi}\catcode`\%=\active\def
\end{pgfscope}%
\begin{pgfscope}%
\definecolor{textcolor}{rgb}{0.150000,0.150000,0.150000}%
\pgfsetstrokecolor{textcolor}%
\pgfsetfillcolor{textcolor}%
\pgftext[x=7.364757in,y=0.968489in,,top]{\color{textcolor}{\sffamily\fontsize{32.000000}{38.400000}\selectfont\catcode`\^=\active\def^{\ifmmode\sp\else\^{}\fi}\catcode`\%=\active\def
\end{pgfscope}%
\begin{pgfscope}%
\definecolor{textcolor}{rgb}{0.150000,0.150000,0.150000}%
\pgfsetstrokecolor{textcolor}%
\pgfsetfillcolor{textcolor}%
\pgftext[x=8.979340in,y=0.968489in,,top]{\color{textcolor}{\sffamily\fontsize{32.000000}{38.400000}\selectfont\catcode`\^=\active\def^{\ifmmode\sp\else\^{}\fi}\catcode`\%=\active\def
\end{pgfscope}%
\begin{pgfscope}%
\definecolor{textcolor}{rgb}{0.150000,0.150000,0.150000}%
\pgfsetstrokecolor{textcolor}%
\pgfsetfillcolor{textcolor}%
\pgftext[x=10.593923in,y=0.968489in,,top]{\color{textcolor}{\sffamily\fontsize{32.000000}{38.400000}\selectfont\catcode`\^=\active\def^{\ifmmode\sp\else\^{}\fi}\catcode`\%=\active\def
\end{pgfscope}%
\begin{pgfscope}%
\definecolor{textcolor}{rgb}{0.150000,0.150000,0.150000}%
\pgfsetstrokecolor{textcolor}%
\pgfsetfillcolor{textcolor}%
\pgftext[x=5.750173in,y=0.506467in,,top]{\color{textcolor}{\sffamily\fontsize{32.000000}{38.400000}\selectfont\catcode`\^=\active\def^{\ifmmode\sp\else\^{}\fi}\catcode`\%=\active\def
\end{pgfscope}%
\begin{pgfscope}%
\definecolor{textcolor}{rgb}{0.150000,0.150000,0.150000}%
\pgfsetstrokecolor{textcolor}%
\pgfsetfillcolor{textcolor}%
\pgftext[x=0.562022in, y=1.081640in, left, base]{\color{textcolor}{\sffamily\fontsize{32.000000}{38.400000}\selectfont\catcode`\^=\active\def^{\ifmmode\sp\else\^{}\fi}\catcode`\%=\active\def
\end{pgfscope}%
\begin{pgfscope}%
\definecolor{textcolor}{rgb}{0.150000,0.150000,0.150000}%
\pgfsetstrokecolor{textcolor}%
\pgfsetfillcolor{textcolor}%
\pgftext[x=0.562022in, y=4.581640in, left, base]{\color{textcolor}{\sffamily\fontsize{32.000000}{38.400000}\selectfont\catcode`\^=\active\def^{\ifmmode\sp\else\^{}\fi}\catcode`\%=\active\def
\end{pgfscope}%
\begin{pgfscope}%
\definecolor{textcolor}{rgb}{0.150000,0.150000,0.150000}%
\pgfsetstrokecolor{textcolor}%
\pgfsetfillcolor{textcolor}%
\pgftext[x=0.506467in,y=2.990711in,,bottom,rotate=90.000000]{\color{textcolor}{\sffamily\fontsize{32.000000}{38.400000}\selectfont\catcode`\^=\active\def^{\ifmmode\sp\else\^{}\fi}\catcode`\%=\active\def
\end{pgfscope}%
\begin{pgfscope}%
\pgfpathrectangle{\pgfqpoint{0.906423in}{1.065711in}}{\pgfqpoint{9.687500in}{3.850000in}}%
\pgfusepath{clip}%
\pgfsetbuttcap%
\pgfsetroundjoin%
\pgfsetlinewidth{1.505625pt}%
\definecolor{currentstroke}{rgb}{0.000000,0.000000,0.000000}%
\pgfsetstrokecolor{currentstroke}%
\pgfsetdash{{5.550000pt}{2.400000pt}}{0.000000pt}%
\pgfpathmoveto{\pgfqpoint{0.906423in}{4.740711in}}%
\pgfpathlineto{\pgfqpoint{1.229340in}{4.740711in}}%
\pgfpathlineto{\pgfqpoint{1.229340in}{1.240711in}}%
\pgfpathlineto{\pgfqpoint{1.552257in}{1.240711in}}%
\pgfpathlineto{\pgfqpoint{1.552257in}{4.740711in}}%
\pgfpathlineto{\pgfqpoint{1.875173in}{4.740711in}}%
\pgfpathlineto{\pgfqpoint{1.875173in}{1.240711in}}%
\pgfpathlineto{\pgfqpoint{2.198090in}{1.240711in}}%
\pgfpathlineto{\pgfqpoint{2.198090in}{4.740711in}}%
\pgfpathlineto{\pgfqpoint{4.458507in}{4.740711in}}%
\pgfpathlineto{\pgfqpoint{4.458507in}{1.240711in}}%
\pgfpathlineto{\pgfqpoint{4.781423in}{1.240711in}}%
\pgfpathlineto{\pgfqpoint{4.781423in}{4.740711in}}%
\pgfpathlineto{\pgfqpoint{10.603923in}{4.740711in}}%
\pgfpathlineto{\pgfqpoint{10.603923in}{4.740711in}}%
\pgfusepath{stroke}%
\end{pgfscope}%
\begin{pgfscope}%
\pgfpathrectangle{\pgfqpoint{0.906423in}{1.065711in}}{\pgfqpoint{9.687500in}{3.850000in}}%
\pgfusepath{clip}%
\pgfsetroundcap%
\pgfsetroundjoin%
\pgfsetlinewidth{1.505625pt}%
\definecolor{currentstroke}{rgb}{0.000000,0.000000,0.000000}%
\pgfsetstrokecolor{currentstroke}%
\pgfsetdash{}{0pt}%
\pgfpathmoveto{\pgfqpoint{0.906423in}{4.740711in}}%
\pgfpathlineto{\pgfqpoint{0.906423in}{1.240711in}}%
\pgfpathlineto{\pgfqpoint{1.229340in}{1.240711in}}%
\pgfpathlineto{\pgfqpoint{1.229340in}{4.740711in}}%
\pgfpathlineto{\pgfqpoint{1.552257in}{4.740711in}}%
\pgfpathlineto{\pgfqpoint{1.552257in}{1.240711in}}%
\pgfpathlineto{\pgfqpoint{1.875173in}{1.240711in}}%
\pgfpathlineto{\pgfqpoint{1.875173in}{4.740711in}}%
\pgfpathlineto{\pgfqpoint{2.198090in}{4.740711in}}%
\pgfpathlineto{\pgfqpoint{2.198090in}{1.240711in}}%
\pgfpathlineto{\pgfqpoint{10.603923in}{1.240711in}}%
\pgfpathlineto{\pgfqpoint{10.603923in}{1.240711in}}%
\pgfusepath{stroke}%
\end{pgfscope}%
\begin{pgfscope}%
\pgfsetrectcap%
\pgfsetmiterjoin%
\pgfsetlinewidth{0.803000pt}%
\definecolor{currentstroke}{rgb}{0.150000,0.150000,0.150000}%
\pgfsetstrokecolor{currentstroke}%
\pgfsetdash{}{0pt}%
\pgfpathmoveto{\pgfqpoint{0.906423in}{1.065711in}}%
\pgfpathlineto{\pgfqpoint{0.906423in}{4.915711in}}%
\pgfusepath{stroke}%
\end{pgfscope}%
\begin{pgfscope}%
\pgfsetrectcap%
\pgfsetmiterjoin%
\pgfsetlinewidth{0.803000pt}%
\definecolor{currentstroke}{rgb}{0.150000,0.150000,0.150000}%
\pgfsetstrokecolor{currentstroke}%
\pgfsetdash{}{0pt}%
\pgfpathmoveto{\pgfqpoint{10.593923in}{1.065711in}}%
\pgfpathlineto{\pgfqpoint{10.593923in}{4.915711in}}%
\pgfusepath{stroke}%
\end{pgfscope}%
\begin{pgfscope}%
\pgfsetrectcap%
\pgfsetmiterjoin%
\pgfsetlinewidth{0.803000pt}%
\definecolor{currentstroke}{rgb}{0.150000,0.150000,0.150000}%
\pgfsetstrokecolor{currentstroke}%
\pgfsetdash{}{0pt}%
\pgfpathmoveto{\pgfqpoint{0.906423in}{1.065711in}}%
\pgfpathlineto{\pgfqpoint{10.593923in}{1.065711in}}%
\pgfusepath{stroke}%
\end{pgfscope}%
\begin{pgfscope}%
\pgfsetrectcap%
\pgfsetmiterjoin%
\pgfsetlinewidth{0.803000pt}%
\definecolor{currentstroke}{rgb}{0.150000,0.150000,0.150000}%
\pgfsetstrokecolor{currentstroke}%
\pgfsetdash{}{0pt}%
\pgfpathmoveto{\pgfqpoint{0.906423in}{4.915711in}}%
\pgfpathlineto{\pgfqpoint{10.593923in}{4.915711in}}%
\pgfusepath{stroke}%
\end{pgfscope}%
\begin{pgfscope}%
\pgfsetbuttcap%
\pgfsetmiterjoin%
\definecolor{currentfill}{rgb}{1.000000,1.000000,1.000000}%
\pgfsetfillcolor{currentfill}%
\pgfsetfillopacity{0.800000}%
\pgfsetlinewidth{1.003750pt}%
\definecolor{currentstroke}{rgb}{0.800000,0.800000,0.800000}%
\pgfsetstrokecolor{currentstroke}%
\pgfsetstrokeopacity{0.800000}%
\pgfsetdash{}{0pt}%
\pgfpathmoveto{\pgfqpoint{6.121863in}{2.674089in}}%
\pgfpathlineto{\pgfqpoint{10.282812in}{2.674089in}}%
\pgfpathquadraticcurveto{\pgfqpoint{10.371701in}{2.674089in}}{\pgfqpoint{10.371701in}{2.762978in}}%
\pgfpathlineto{\pgfqpoint{10.371701in}{4.604600in}}%
\pgfpathquadraticcurveto{\pgfqpoint{10.371701in}{4.693489in}}{\pgfqpoint{10.282812in}{4.693489in}}%
\pgfpathlineto{\pgfqpoint{6.121863in}{4.693489in}}%
\pgfpathquadraticcurveto{\pgfqpoint{6.032975in}{4.693489in}}{\pgfqpoint{6.032975in}{4.604600in}}%
\pgfpathlineto{\pgfqpoint{6.032975in}{2.762978in}}%
\pgfpathquadraticcurveto{\pgfqpoint{6.032975in}{2.674089in}}{\pgfqpoint{6.121863in}{2.674089in}}%
\pgfpathlineto{\pgfqpoint{6.121863in}{2.674089in}}%
\pgfpathclose%
\pgfusepath{stroke,fill}%
\end{pgfscope}%
\begin{pgfscope}%
\definecolor{textcolor}{rgb}{0.150000,0.150000,0.150000}%
\pgfsetstrokecolor{textcolor}%
\pgfsetfillcolor{textcolor}%
\pgftext[x=6.485650in,y=4.197569in,left,base]{\color{textcolor}{\sffamily\fontsize{32.000000}{38.400000}\selectfont\catcode`\^=\active\def^{\ifmmode\sp\else\^{}\fi}\catcode`\%=\active\def
\end{pgfscope}%
\begin{pgfscope}%
\pgfsetbuttcap%
\pgfsetroundjoin%
\pgfsetlinewidth{1.505625pt}%
\definecolor{currentstroke}{rgb}{0.000000,0.000000,0.000000}%
\pgfsetstrokecolor{currentstroke}%
\pgfsetdash{{5.550000pt}{2.400000pt}}{0.000000pt}%
\pgfpathmoveto{\pgfqpoint{6.210752in}{3.724436in}}%
\pgfpathlineto{\pgfqpoint{6.210752in}{3.724436in}}%
\pgfpathlineto{\pgfqpoint{6.655197in}{3.724436in}}%
\pgfpathlineto{\pgfqpoint{6.655197in}{3.724436in}}%
\pgfpathlineto{\pgfqpoint{7.099641in}{3.724436in}}%
\pgfusepath{stroke}%
\end{pgfscope}%
\begin{pgfscope}%
\definecolor{textcolor}{rgb}{0.150000,0.150000,0.150000}%
\pgfsetstrokecolor{textcolor}%
\pgfsetfillcolor{textcolor}%
\pgftext[x=7.455197in,y=3.568880in,left,base]{\color{textcolor}{\sffamily\fontsize{32.000000}{38.400000}\selectfont\catcode`\^=\active\def^{\ifmmode\sp\else\^{}\fi}\catcode`\%=\active\def
\end{pgfscope}%
\begin{pgfscope}%
\pgfsetroundcap%
\pgfsetroundjoin%
\pgfsetlinewidth{1.505625pt}%
\definecolor{currentstroke}{rgb}{0.000000,0.000000,0.000000}%
\pgfsetstrokecolor{currentstroke}%
\pgfsetdash{}{0pt}%
\pgfpathmoveto{\pgfqpoint{6.210752in}{3.095747in}}%
\pgfpathlineto{\pgfqpoint{6.210752in}{3.095747in}}%
\pgfpathlineto{\pgfqpoint{6.655197in}{3.095747in}}%
\pgfpathlineto{\pgfqpoint{6.655197in}{3.095747in}}%
\pgfpathlineto{\pgfqpoint{7.099641in}{3.095747in}}%
\pgfusepath{stroke}%
\end{pgfscope}%
\begin{pgfscope}%
\definecolor{textcolor}{rgb}{0.150000,0.150000,0.150000}%
\pgfsetstrokecolor{textcolor}%
\pgfsetfillcolor{textcolor}%
\pgftext[x=7.455197in,y=2.940191in,left,base]{\color{textcolor}{\sffamily\fontsize{32.000000}{38.400000}\selectfont\catcode`\^=\active\def^{\ifmmode\sp\else\^{}\fi}\catcode`\%=\active\def
\end{pgfscope}%
\end{pgfpicture}%
\makeatother%
\endgroup%

%% file: plots/centrallcontrol.tex
\tikzset{every picture/.style={line width=0.75pt}} 
\resizebox{0.8\columnwidth}{!}{
\begin{tikzpicture}[x=0.75pt,y=0.75pt,yscale=-1,xscale=1]

\draw  [fill={rgb, 255:red, 161; green, 183; blue, 255 }  ,fill opacity=1 ] (125,23) -- (234,23) -- (234,92) -- (125,92) -- cycle ;
\draw  [fill={rgb, 255:red, 204; green, 255; blue, 144 }  ,fill opacity=1 ] (302,14.49) -- (411,14.49) -- (411,104.1) -- (302,104.1) -- cycle ;
\draw    (234,47) -- (297,47) ;
\draw [shift={(299,47)}, rotate = 180] [color={rgb, 255:red, 0; green, 0; blue, 0 }  ][line width=0.75]    (10.93,-3.29) .. controls (6.95,-1.4) and (3.31,-0.3) .. (0,0) .. controls (3.31,0.3) and (6.95,1.4) .. (10.93,3.29)   ;
\draw    (303,83) -- (237,83) ;
\draw [shift={(235,83)}, rotate = 360] [color={rgb, 255:red, 0; green, 0; blue, 0 }  ][line width=0.75]    (10.93,-3.29) .. controls (6.95,-1.4) and (3.31,-0.3) .. (0,0) .. controls (3.31,0.3) and (6.95,1.4) .. (10.93,3.29)   ;
\draw    (302,104.1) -- (222.51,183.59) ;
\draw [shift={(221.1,185)}, rotate = 315] [color={rgb, 255:red, 0; green, 0; blue, 0 }  ][line width=0.75]    (10.93,-3.29) .. controls (6.95,-1.4) and (3.31,-0.3) .. (0,0) .. controls (3.31,0.3) and (6.95,1.4) .. (10.93,3.29)   ;
\draw    (328,104) -- (328,184) ;
\draw [shift={(328,186)}, rotate = 270] [color={rgb, 255:red, 0; green, 0; blue, 0 }  ][line width=0.75]    (10.93,-3.29) .. controls (6.95,-1.4) and (3.31,-0.3) .. (0,0) .. controls (3.31,0.3) and (6.95,1.4) .. (10.93,3.29)   ;
\draw   (138.19,185) -- (223.19,185) -- (223.19,234) -- (138.19,234) -- cycle ;
\draw   (243,184) -- (328,184) -- (328,233) -- (243,233) -- cycle ;
\draw   (347.19,185) -- (432.19,185) -- (432.19,234) -- (347.19,234) -- cycle ;
\draw    (384,104) -- (432.15,183.29) ;
\draw [shift={(433.19,185)}, rotate = 238.73] [color={rgb, 255:red, 0; green, 0; blue, 0 }  ][line width=0.75]    (10.93,-3.29) .. controls (6.95,-1.4) and (3.31,-0.3) .. (0,0) .. controls (3.31,0.3) and (6.95,1.4) .. (10.93,3.29)   ;
\draw    (183.19,166) -- (183.19,182) ;
\draw [shift={(183.19,184)}, rotate = 270] [color={rgb, 255:red, 0; green, 0; blue, 0 }  ][line width=0.75]    (10.93,-3.29) .. controls (6.95,-1.4) and (3.31,-0.3) .. (0,0) .. controls (3.31,0.3) and (6.95,1.4) .. (10.93,3.29)   ;
\draw    (290.79,169) -- (290.79,181.55) ;
\draw [shift={(290.79,183.55)}, rotate = 270] [color={rgb, 255:red, 0; green, 0; blue, 0 }  ][line width=0.75]    (10.93,-3.29) .. controls (6.95,-1.4) and (3.31,-0.3) .. (0,0) .. controls (3.31,0.3) and (6.95,1.4) .. (10.93,3.29)   ;
\draw    (371.19,170) -- (371.19,183) ;
\draw [shift={(371.19,185)}, rotate = 270] [color={rgb, 255:red, 0; green, 0; blue, 0 }  ][line width=0.75]    (10.93,-3.29) .. controls (6.95,-1.4) and (3.31,-0.3) .. (0,0) .. controls (3.31,0.3) and (6.95,1.4) .. (10.93,3.29)   ;
\draw    (180,234) -- (180,249) ;
\draw [shift={(180,251)}, rotate = 270] [color={rgb, 255:red, 0; green, 0; blue, 0 }  ][line width=0.75]    (10.93,-3.29) .. controls (6.95,-1.4) and (3.31,-0.3) .. (0,0) .. controls (3.31,0.3) and (6.95,1.4) .. (10.93,3.29)   ;
\draw    (287,233) -- (287,250) ;
\draw [shift={(287,252)}, rotate = 270] [color={rgb, 255:red, 0; green, 0; blue, 0 }  ][line width=0.75]    (10.93,-3.29) .. controls (6.95,-1.4) and (3.31,-0.3) .. (0,0) .. controls (3.31,0.3) and (6.95,1.4) .. (10.93,3.29)   ;
\draw    (389,234) -- (389,251) ;
\draw [shift={(389,253)}, rotate = 270] [color={rgb, 255:red, 0; green, 0; blue, 0 }  ][line width=0.75]    (10.93,-3.29) .. controls (6.95,-1.4) and (3.31,-0.3) .. (0,0) .. controls (3.31,0.3) and (6.95,1.4) .. (10.93,3.29)   ;

\draw (323,29.8) node [anchor=north west][inner sep=0.75pt]   [align=left] {\begin{minipage}[lt]{49.2pt}\setlength\topsep{0pt}
\begin{center}
Central \\Controller\\for LLMAs
\end{center}

\end{minipage}};
\draw (247,23) node [anchor=north west][inner sep=0.75pt]   [align=left] {Policy};
\draw (246,62) node [anchor=north west][inner sep=0.75pt]   [align=left] {Costs};
\draw  [fill={rgb, 255:red, 255; green, 255; blue, 255 }  ,fill opacity=1 ]  (250,112) -- (434,112) -- (434,137) -- (250,137) -- cycle  ;
\draw (253,116) node [anchor=north west][inner sep=0.75pt]   [align=left] {{\footnotesize \textit{share} private observation or \textit{herd}}};
\draw  [fill={rgb, 255:red, 255; green, 255; blue, 255 }  ,fill opacity=1 ]  (127,145) -- (236,145) -- (236,170) -- (127,170) -- cycle  ;
\draw (130,149) node [anchor=north west][inner sep=0.75pt]   [align=left] {{\footnotesize private observation}};
\draw  [fill={rgb, 255:red, 255; green, 255; blue, 255 }  ,fill opacity=1 ]  (156,250) -- (204,250) -- (204,275) -- (156,275) -- cycle  ;
\draw (159,254) node [anchor=north west][inner sep=0.75pt]   [align=left] {action};
\draw (265.19,189) node [anchor=north west][inner sep=0.75pt]   [align=left] {\begin{minipage}[lt]{29.37pt}\setlength\topsep{0pt}
\begin{center}
LLMA\\$\displaystyle k+1$
\end{center}

\end{minipage}};
\draw (159.19,191) node [anchor=north west][inner sep=0.75pt]   [align=left] {\begin{minipage}[lt]{31.64pt}\setlength\topsep{0pt}
\begin{center}
LLMA \\$\displaystyle k$
\end{center}

\end{minipage}};
\draw (367.19,191) node [anchor=north west][inner sep=0.75pt]   [align=left] {\begin{minipage}[lt]{29.37pt}\setlength\topsep{0pt}
\begin{center}
LLMA\\$\displaystyle k+2$
\end{center}

\end{minipage}};
\draw (130,39) node [anchor=north west][inner sep=0.75pt]   [align=left] {\begin{minipage}[lt]{67.92pt}\setlength\topsep{0pt}
\begin{center}
Stochastic\\Approximation
\end{center}

\end{minipage}};

\end{tikzpicture}}

%% file: plots/autonomous.tex
\tikzset{every picture/.style={line width=0.75pt}} 
\resizebox{0.8\columnwidth}{!}{
\begin{tikzpicture}[x=0.75pt,y=0.75pt,yscale=-1,xscale=1]

\draw  [fill={rgb, 255:red, 161; green, 183; blue, 255 }  ,fill opacity=1 ] (125,23) -- (234,23) -- (234,92) -- (125,92) -- cycle ;
\draw  [fill={rgb, 255:red, 204; green, 255; blue, 144 }  ,fill opacity=1 ] (302,14.49) -- (411,14.49) -- (411,104.1) -- (302,104.1) -- cycle ;
\draw    (234,47) -- (297,47) ;
\draw [shift={(299,47)}, rotate = 180] [color={rgb, 255:red, 0; green, 0; blue, 0 }  ][line width=0.75]    (10.93,-3.29) .. controls (6.95,-1.4) and (3.31,-0.3) .. (0,0) .. controls (3.31,0.3) and (6.95,1.4) .. (10.93,3.29)   ;
\draw    (303,83) -- (237,83) ;
\draw [shift={(235,83)}, rotate = 360] [color={rgb, 255:red, 0; green, 0; blue, 0 }  ][line width=0.75]    (10.93,-3.29) .. controls (6.95,-1.4) and (3.31,-0.3) .. (0,0) .. controls (3.31,0.3) and (6.95,1.4) .. (10.93,3.29)   ;
\draw    (302,104.1) -- (222.51,183.59) ;
\draw [shift={(221.1,185)}, rotate = 315] [color={rgb, 255:red, 0; green, 0; blue, 0 }  ][line width=0.75]    (10.93,-3.29) .. controls (6.95,-1.4) and (3.31,-0.3) .. (0,0) .. controls (3.31,0.3) and (6.95,1.4) .. (10.93,3.29)   ;
\draw    (328,104) -- (328,184) ;
\draw [shift={(328,186)}, rotate = 270] [color={rgb, 255:red, 0; green, 0; blue, 0 }  ][line width=0.75]    (10.93,-3.29) .. controls (6.95,-1.4) and (3.31,-0.3) .. (0,0) .. controls (3.31,0.3) and (6.95,1.4) .. (10.93,3.29)   ;
\draw   (138.19,185) -- (223.19,185) -- (223.19,234) -- (138.19,234) -- cycle ;
\draw   (243,184) -- (328,184) -- (328,233) -- (243,233) -- cycle ;
\draw   (347.19,185) -- (432.19,185) -- (432.19,234) -- (347.19,234) -- cycle ;
\draw    (384,104) -- (432.15,183.29) ;
\draw [shift={(433.19,185)}, rotate = 238.73] [color={rgb, 255:red, 0; green, 0; blue, 0 }  ][line width=0.75]    (10.93,-3.29) .. controls (6.95,-1.4) and (3.31,-0.3) .. (0,0) .. controls (3.31,0.3) and (6.95,1.4) .. (10.93,3.29)   ;
\draw    (183.19,166) -- (183.19,182) ;
\draw [shift={(183.19,184)}, rotate = 270] [color={rgb, 255:red, 0; green, 0; blue, 0 }  ][line width=0.75]    (10.93,-3.29) .. controls (6.95,-1.4) and (3.31,-0.3) .. (0,0) .. controls (3.31,0.3) and (6.95,1.4) .. (10.93,3.29)   ;
\draw    (290.79,169) -- (290.79,181.55) ;
\draw [shift={(290.79,183.55)}, rotate = 270] [color={rgb, 255:red, 0; green, 0; blue, 0 }  ][line width=0.75]    (10.93,-3.29) .. controls (6.95,-1.4) and (3.31,-0.3) .. (0,0) .. controls (3.31,0.3) and (6.95,1.4) .. (10.93,3.29)   ;
\draw    (371.19,170) -- (371.19,183) ;
\draw [shift={(371.19,185)}, rotate = 270] [color={rgb, 255:red, 0; green, 0; blue, 0 }  ][line width=0.75]    (10.93,-3.29) .. controls (6.95,-1.4) and (3.31,-0.3) .. (0,0) .. controls (3.31,0.3) and (6.95,1.4) .. (10.93,3.29)   ;
\draw    (180,234) -- (180,249) ;
\draw [shift={(180,251)}, rotate = 270] [color={rgb, 255:red, 0; green, 0; blue, 0 }  ][line width=0.75]    (10.93,-3.29) .. controls (6.95,-1.4) and (3.31,-0.3) .. (0,0) .. controls (3.31,0.3) and (6.95,1.4) .. (10.93,3.29)   ;
\draw    (287,233) -- (287,250) ;
\draw [shift={(287,252)}, rotate = 270] [color={rgb, 255:red, 0; green, 0; blue, 0 }  ][line width=0.75]    (10.93,-3.29) .. controls (6.95,-1.4) and (3.31,-0.3) .. (0,0) .. controls (3.31,0.3) and (6.95,1.4) .. (10.93,3.29)   ;
\draw    (389,234) -- (389,251) ;
\draw [shift={(389,253)}, rotate = 270] [color={rgb, 255:red, 0; green, 0; blue, 0 }  ][line width=0.75]    (10.93,-3.29) .. controls (6.95,-1.4) and (3.31,-0.3) .. (0,0) .. controls (3.31,0.3) and (6.95,1.4) .. (10.93,3.29)   ;

\draw (312,29.8) node [anchor=north west][inner sep=0.75pt]   [align=left] {\begin{minipage}[lt]{61.69pt}\setlength\topsep{0pt}
\begin{center}
Central \\Entity\\using LLMAs
\end{center}

\end{minipage}};
\draw (237,6) node [anchor=north west][inner sep=0.75pt]   [align=left] {\begin{minipage}[lt]{43.55pt}\setlength\topsep{0pt}
\begin{center}
Incentive\\Policy
\end{center}

\end{minipage}};
\draw (246,62) node [anchor=north west][inner sep=0.75pt]   [align=left] {Costs};
\draw  [fill={rgb, 255:red, 255; green, 255; blue, 255 }  ,fill opacity=1 ]  (305,114) -- (373,114) -- (373,139) -- (305,139) -- cycle  ;
\draw (308,118) node [anchor=north west][inner sep=0.75pt]   [align=left] {Incentive};
\draw  [fill={rgb, 255:red, 255; green, 255; blue, 255 }  ,fill opacity=1 ]  (127,145) -- (236,145) -- (236,170) -- (127,170) -- cycle  ;
\draw (130,149) node [anchor=north west][inner sep=0.75pt]   [align=left] {{\footnotesize private observation}};
\draw  [fill={rgb, 255:red, 255; green, 255; blue, 255 }  ,fill opacity=1 ]  (156,250) -- (204,250) -- (204,275) -- (156,275) -- cycle  ;
\draw (159,254) node [anchor=north west][inner sep=0.75pt]   [align=left] {action};
\draw (265.19,189) node [anchor=north west][inner sep=0.75pt]   [align=left] {\begin{minipage}[lt]{29.37pt}\setlength\topsep{0pt}
\begin{center}
LLMA\\$\displaystyle k+1$
\end{center}

\end{minipage}};
\draw (159.19,191) node [anchor=north west][inner sep=0.75pt]   [align=left] {\begin{minipage}[lt]{31.64pt}\setlength\topsep{0pt}
\begin{center}
LLMA \\$\displaystyle k$
\end{center}

\end{minipage}};
\draw (367.19,191) node [anchor=north west][inner sep=0.75pt]   [align=left] {\begin{minipage}[lt]{29.37pt}\setlength\topsep{0pt}
\begin{center}
LLMA\\$\displaystyle k+2$
\end{center}

\end{minipage}};
\draw (130,39) node [anchor=north west][inner sep=0.75pt]   [align=left] {\begin{minipage}[lt]{67.92pt}\setlength\topsep{0pt}
\begin{center}
Stochastic\\Approximation
\end{center}

\end{minipage}};

\end{tikzpicture}}